\documentclass[12pt, a4paper, oneside]{article}
\usepackage[margin=1in]{geometry} 
\usepackage{authblk} 
\usepackage{amsmath, amsthm, amssymb, appendix, bm, graphicx, mathrsfs, subcaption}
\usepackage{enumitem}
\usepackage{algorithm}
\usepackage{multirow}
\usepackage{amsmath}
\usepackage{xcolor}
\usepackage{makecell} 
\usepackage{booktabs} 
\usepackage{algpseudocode}
\usepackage[round]{natbib} 
\usepackage{mathtools}
\usepackage{tocloft} 
\usepackage{etoc}    

\newlength{\springerfigwidth}
\newlength{\springerfigmaxheight}
\newcommand{\springersetfigwidth}{%
  \setlength{\springerfigwidth}{\linewidth}%
  \setlength{\springerfigmaxheight}{0.9\textheight}%
}
\providecommand{\backmatter}{}

\theoremstyle{plain} 
\newtheorem{theorem}{Theorem}[section]
\newtheorem{lemma}{Lemma}[section]
\newtheorem{corollary}{Corollary}[section]
\newtheorem{proposition}{Proposition}[section]

\theoremstyle{definition} 
\newtheorem{definition}{Definition}[section]
\newtheorem{assumption}{Assumption}[section]
\newtheorem{example}{Example}[section]

\theoremstyle{remark} 
\newtheorem{remark}{Remark}[section]

\newcommand{\A}{\mathcal{A}}
\newcommand{\R}{\mathbb{R}}

\newcommand{\G}{\mathbf{G}}
\newcommand{\U}{\mathbf{U}}
\newcommand{\V}{\mathbf{V}}

\newcommand{\sign}{\mathbf{sign}}
\newcommand{\msgn}{\mathbf{msgn}}
\newcommand{\Tmsgn}{\mathbf{Tmsgn}}

\newcommand{\cmark}{\checkmark}
\newcommand{\xmark}{$\times$}

\usepackage[colorlinks,
            linkcolor=blue,       
            anchorcolor=blue,  
            citecolor=blue,        
            ]{hyperref}

\title{\textbf{Convergence of Spectral Descent for Non-smooth Optimization}}
\date{}
\author[1]{Yixuan Yang}
\author[1]{Yuqing He}
\author[1]{Song Li\thanks{Corresponding author}}
\affil[1]{School of Mathematical Sciences, Zhejiang University, Hangzhou, China}
\affil[ ]{\texttt{\{yixuanyang, yuqinghe25, songli\}@zju.edu.cn}}

\begin{document}
\etocdepthtag.toc{mtmain} 
\maketitle
\begin{abstract}
    
The Muon optimizer has recently demonstrated remarkable empirical success in training large language models. 
However, the theoretical understanding of its mechanisms remains limited. 
Current convergence guarantees for Muon rely heavily on smoothness assumptions, leaving its non-smooth convergence behavior largely unexplored. 
In this work, we take a step toward bridging this gap by investigating Spectral Descent (SD), a simplified variant of Muon, together with its truncated counterpart, Truncated Spectral Descent (TSD). 
Under convexity, Lipschitz continuity, and sharpness conditions, we establish global linear convergence for both SD and TSD in non-smooth convex formulations. 
We also study regularized variants equipped with decoupled weight decay and derive sublinear convergence guarantees through their connection with Frank-Wolfe methods. 
Finally, we apply our theoretical framework to robust low-rank matrix recovery under mixed sparse and dense noise regimes and provide rigorous recovery guarantees. 
Numerical experiments support the theoretical findings and demonstrate the effectiveness of Muon-type methods for non-smooth optimization.\\

\end{abstract}
\thispagestyle{empty}

\section{Introduction}
Large Language Models (LLMs) have emerged as a central paradigm in modern artificial intelligence, demonstrating strong performance across a wide range of tasks and applications \citep{brown2020language, team2023gemini}.
However, the immense computational overhead of large-scale LLM training makes computational efficiency a primary practical concern.
In this context, the choice of optimizer is particularly important.
Over the past few years, stochastic gradient descent (SGD) and adaptive methods, particularly Adam \citep{kingma2014adam} and AdamW \citep{loshchilov2017decoupled}, have become the primary choices for large-scale model training.

Recently, Muon \citep{jordan6muon} has emerged as a highly promising alternative to these standard adaptive optimizers for large-scale training.
Its core mechanism lies in leveraging the matrix-wise geometry of the parameters to formulate an approximately orthogonalized descent direction from the gradients.
The practical advantages of Muon have been validated by recent empirical studies, which demonstrate its capability to improve the compute-time tradeoff in large-scale pretraining \citep{shah2025practical} and facilitate better generalization \citep{tveit2025muon}.
Figure~\ref{fig:nn_experiments} further illustrates this empirical advantage in a non-smooth setting, where Muon and Spectral Descent achieve faster empirical loss decay than GD and Adam in ReLU network training.
Moreover, with decoupled weight decay, Muon has been scaled to large language models with billions of parameters trained on trillions of tokens \citep{liu2025muon, zeng2025glm}.
DeepSeek-AI also reports that Muon is used in most modules of the DeepSeek-V4 series, including the 1.6-trillion-parameter DeepSeek-V4-Pro model, to improve convergence speed and training stability \citep{deepseekai2026deepseekv4}.

In parallel with these empirical successes, the theoretical investigation of Muon-type optimizers has also seen steady progress \citep{shen2025convergence,li2025note,fan2026implicit,lau2025polargrad,ma2026preconditioning,braun2026spectral,yang2026manifold}.
Recently, \citet{shen2025convergence} provided convergence guarantees for Muon and SD, and showed that these optimizers outperform gradient descent.
For non-convex objectives, \citet{li2025note} established Muon's convergence under a smoothness assumption.
Moreover, \citet{ma2026preconditioning} investigated Spectral Descent for specific tasks such as matrix factorization and in-context learning, revealing that spectral orthogonalization acts as a natural preconditioner to achieve a linear convergence rate independent of the condition number.
However, the convergence behavior of Muon-type optimizers in non-smooth optimization problems, which arise widely in data science and deep learning, remains poorly understood.

\paragraph{Our Goal and Contributions.}
In this work, we take a step towards bridging this gap by analyzing the convergence of Spectral Descent and its truncated variant.
To the best of our knowledge, this represents the first theoretical guarantee for muon-type optimizers under non-smooth conditions.
A theoretical comparison between our results and existing analyses is summarized in Table \ref{tab:problem_setting_model_assumption}.
    \begin{figure}[H]
    \centering
    \springersetfigwidth
    \begin{minipage}{\springerfigwidth}
    \centering

    \begin{subfigure}[b]{0.48\linewidth}
        \centering
        \includegraphics[width=\linewidth,height=\springerfigmaxheight,keepaspectratio]{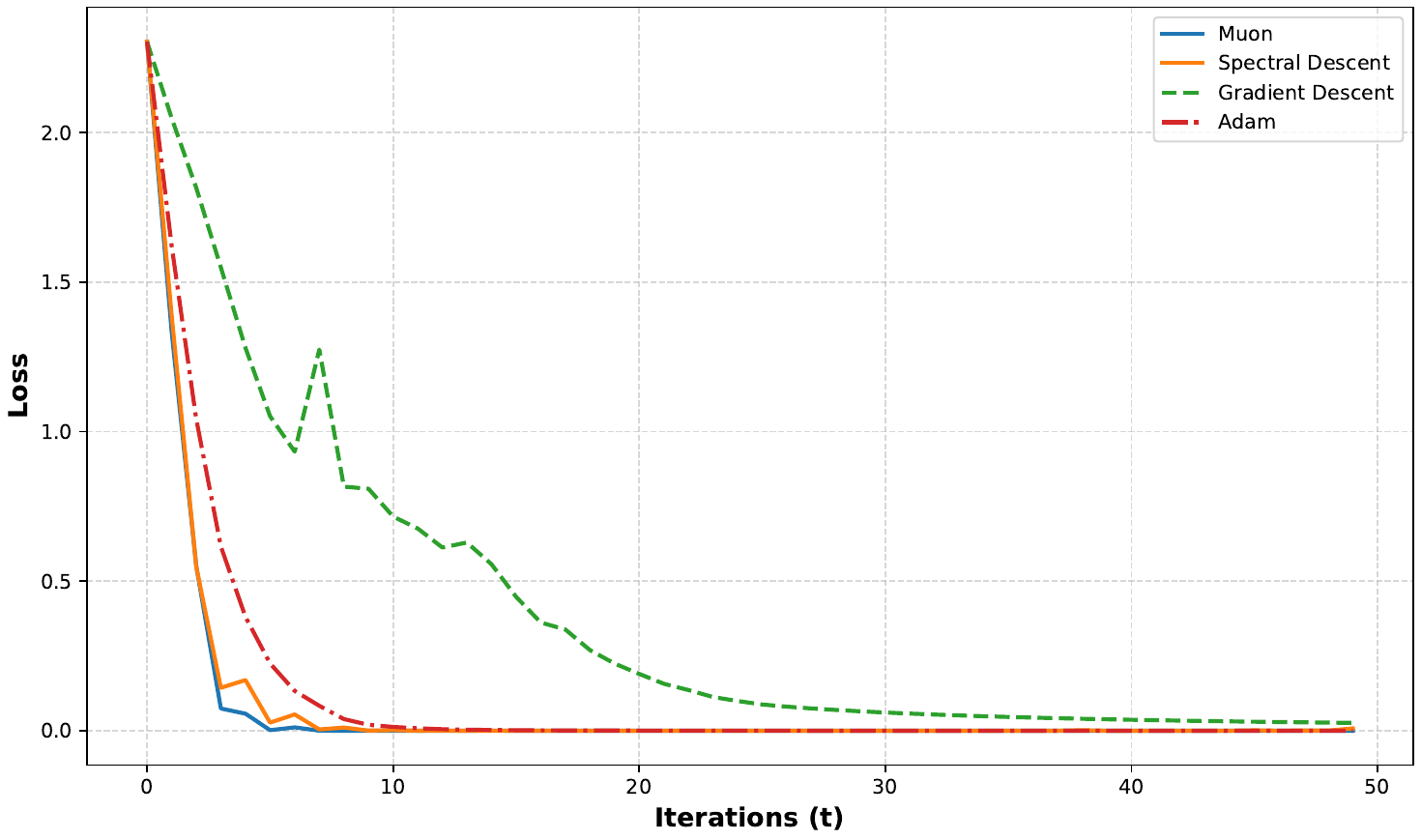}
        \caption{MNIST - Constant Step Size}
        \label{fig:mnist_constant}
    \end{subfigure}
    \hfill
    \begin{subfigure}[b]{0.48\linewidth}
        \centering
        \includegraphics[width=\linewidth,height=\springerfigmaxheight,keepaspectratio]{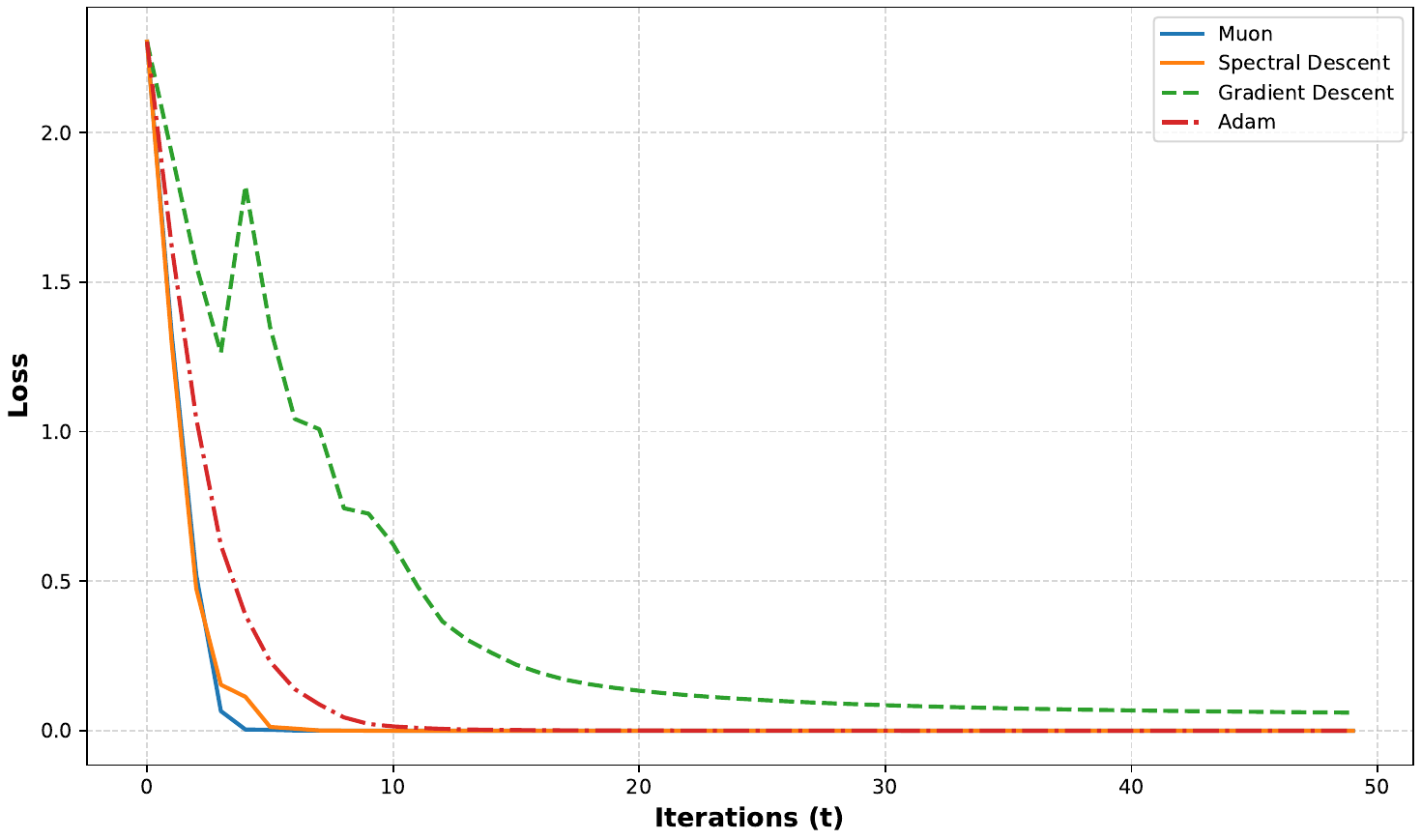}
        \caption{MNIST - Geometric Decay}
        \label{fig:mnist_decay}
    \end{subfigure}
    \vspace{0.4cm}

    \begin{subfigure}[b]{0.48\linewidth}
        \centering
        \includegraphics[width=\linewidth,height=\springerfigmaxheight,keepaspectratio]{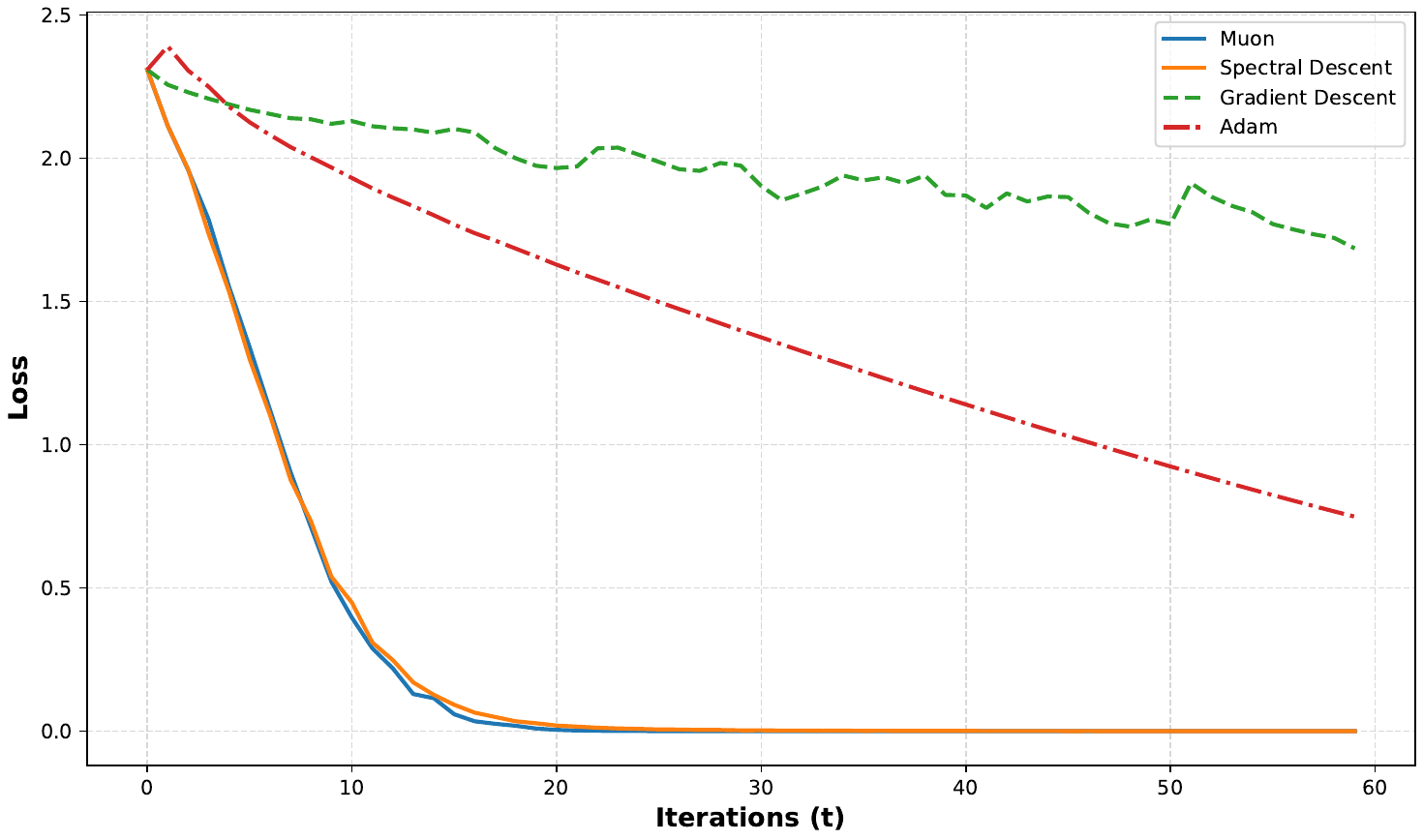}
        \caption{CIFAR-10 - Constant Step Size}
        \label{fig:cifar_constant}
    \end{subfigure}
    \hfill
    \begin{subfigure}[b]{0.48\linewidth}
        \centering
        \includegraphics[width=\linewidth,height=\springerfigmaxheight,keepaspectratio]{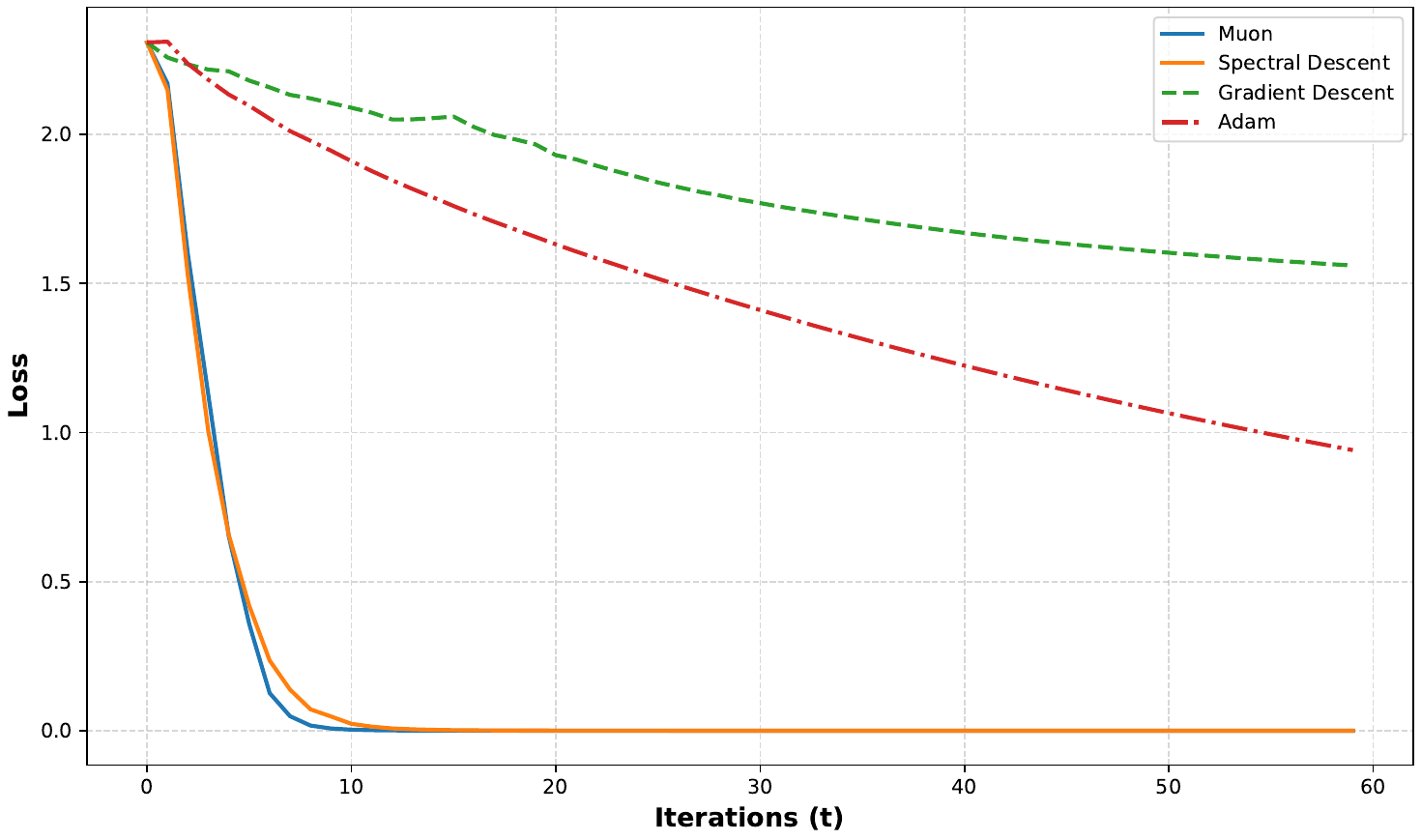}
        \caption{CIFAR-10 - Geometric Decay}
        \label{fig:cifar_decay}
    \end{subfigure}
    \end{minipage}
    \caption{Empirical comparison of GD, Adam, Spectral Descent, and Muon for training two-layer fully connected ReLU neural networks on MNIST and CIFAR-10.
    \textbf{Top row:} Performance on MNIST with (a) constant step size and (b) geometric decay step size. \textbf{Bottom row:} Performance on CIFAR-10 with (c) constant step size and (d) geometric decay step size.}
    \label{fig:nn_experiments}
\end{figure}
Our Contributions can be summarized as follows:
\begin{itemize}
    \item \textit{Global Linear Convergence of (Truncated) Spectral Descent:} In Theorem \ref{theorem:spectral_descent_convergence}, we rigorously establish the global linear convergence of standard Spectral Descent under convexity, Lipschitz continuity, and sharpness conditions.
    Our analysis reveals a rank-dependent descent mechanism, where the contraction rate is explicitly governed by the subgradient rank.
    Furthermore, Theorem~\ref{theorem:truncated_spectral_descent_convergence} extends this fast global linear convergence to the computationally efficient Truncated Spectral Descent (TSD) under the identical structural assumptions.
    Notably, we demonstrate that truncation not only reduces computational overhead but also theoretically relaxes the worst-case condition parameter requirement.
    Conceptually, SD and TSD are driven by distinct steepest descent geometries: while standard SD aligns with the spectral norm, TSD operates under the Ky Fan s-norm steepest descent geometry.

    \item \textit{Theoretical Guarantees with Decoupled Weight Decay:} We analyze the optimization dynamics when decoupled weight decay is incorporated, revealing a fundamental structural connection between regularized spectral updates and the Conditional Subgradient (Frank-Wolfe) method.
    To stabilize the optimization trajectory in non-smooth landscapes without the use of momentum, we introduce a neighborhood-based subgradient selection mechanism (acting as spatial smoothing).
    In Theorem~\ref{theorem:convergence_SD-WD} and Theorem~\ref{theorem:convergence_RTSD-WD}, we establish rigorous $\mathcal{O}(1/\sqrt{T})$ sublinear convergence guarantees for the regularized versions of Spectral Descent and its truncated variant under general structural assumptions, where $T$ denotes the number of iterations.
    We further instantiate this theoretical framework to the robust low-rank matrix recovery problem via a Least Absolute Deviation (LAD) formulation (cf. Theorem~\ref{theorem:convergence_low_rank_rtsd}).
    Our results provide theoretical guarantees for robust recovery under mixed-noise regimes and show that using the minimal truncation rank ($s=1$) can avoid the expensive $\mathcal{O}(n^3)$ bottleneck of full SVD computations.

    \item \textit{Experimental Validation:} We provide extensive numerical experiments to validate our theoretical findings and demonstrate the empirical performance of Muon-type optimizers.
    First, we show that Spectral Descent and Muon achieve significantly accelerated convergence and robustness to learning rate schedules when training non-smooth ReLU neural networks.
    Second, we empirically verify the convergence behaviors of SD and TSD on linear programming and matrix classification tasks, confirming their ability to efficiently minimize non-smooth objectives and accurately recover optimal solutions.
    Finally, we validate the Regularized Truncated Spectral Descent with Weight Decay (RTSD-WD) on robust low-rank matrix recovery, confirming its theoretical sublinear convergence and its robustness to mixed noise regimes.
\end{itemize}
\begin{table}[!h]
    \centering

    \renewcommand{\arraystretch}{1.2}
    \resizebox{\textwidth}{!}{
        \begin{tabular}{c|c|c|c|c|c}
            \Xhline{1.2pt}
            \multirow{3}{*}{} & \multirow{3}{*}{Optimizer} & \multicolumn{3}{c|}{Problem Setting} & \multirow{3}{*}{Convergence Rate} \\
            \cmidrule{3-5}
             & & \multirow{2}{*}{Model} & \multicolumn{2}{c|}{Assumptions} & \\
            \cmidrule{4-5}
             & &  & Smooth & L.C.$^*$ & \\
            \hline
            \multicolumn{6}{c}{\textbf{Without Weight Decay}} \\

            \Xhline{0.8pt}
            \multirow{2}{*}{\citep{shen2025convergence}} & SD & {General} & \cmark & \xmark & $\mathcal{O}(\log \frac{1}{\varepsilon})$ \\
            & SD \& Muon &  General(nonconvex) & \cmark & \xmark & $\mathcal{O}(\frac{1}{\varepsilon^2})$ \\

            \hline
            {\citep{li2025note}} & {SD \& Muon} & {General(nonconvex)}  & {\cmark} & {\xmark} & {$\mathcal{O}(\frac{1}{\varepsilon^2})$} \\

            \hline
            \citep{ma2026preconditioning} & SD & \multicolumn{3}{c|}{\begin{tabular}{@{}c@{}}Matrix factorization \\ \& in-context learning\end{tabular}} & $\mathcal{O}(\log \frac{1}{\varepsilon})$ \\

            \hline
            This work(Theorem \ref{theorem:spectral_descent_convergence}) & SD & General & \xmark & \cmark & $\mathcal{O}(\log \frac{1}{\varepsilon})$ \\
            \hline
            This work(Theorem \ref{theorem:truncated_spectral_descent_convergence}) & TSD & General  & \xmark & \cmark & $\mathcal{O}(\log \frac{1}{\varepsilon})$ \\

            \hline
            \multicolumn{6}{c}{\textbf{With Weight Decay}} \\
            \Xhline{0.8pt}
            \citep{sfyraki2025lions} & MuonW & General  & \cmark & \xmark & $\mathcal{O}(\frac{1}{\varepsilon^2})$ \\
            \hline
            This work(Theorem \ref{theorem:convergence_SD-WD}) & RSD-WD & General  & \xmark & \cmark & $\mathcal{O}(\frac{1}{\varepsilon^2})$ \\
            \hline
            This work(Theorem \ref{theorem:convergence_RTSD-WD}) & RTSD-WD & General  & \xmark & \cmark & $\mathcal{O}(\frac{1}{\varepsilon^2})$ \\
            \hline
            This work (Theorem \ref{theorem:convergence_low_rank_rtsd}) & RTSD-WD & low-rank matrix recovery via LAD & \xmark & \cmark & $\mathcal{O}(\frac{1}{\varepsilon^2})$ \\
            \Xhline{1.2pt}
        \end{tabular}
    }
    \vspace{0.15cm}
    \raggedright
    \footnotesize{$^*$ L.C. denotes the Lipschitz Continuous assumption.} \\
    \caption{Theoretical comparison of representative Muon-type optimizers.
    For general settings, we contrast the required assumptions (smoothness, and L.C.).
    We also report their corresponding convergence rates.
    }

    \label{tab:problem_setting_model_assumption}
\end{table}

\section{Related Work}
\label{sec:related_work}
\paragraph{Spectral Descent and Muon.}
The conceptual foundation of the Spectral Descent (SD) optimizer can be traced back to the works \citep{carlson2015stochasticdgm, carlson2015stochasticrbm, carlson2015preconditioned}, which initially introduced it for machine learning applications.
Fundamentally, SD operates as steepest descent with respect to the spectral norm. Recently, \citet{bernstein2024old} showed that the Shampoo \citep{gupta2018shampoo} optimizer can also be understood from this identical perspective.

Driven by these matrix-wise geometric insights, \citet{jordan6muon} proposed the Muon optimizer; it incorporates a momentum term into SD and demonstrates superior performance in large-scale model training.
Inspired by the empirical success of Muon, a proliferation of variants has recently emerged to improve its memory efficiency, stability, and overall performance \citep{an2025asgo, ahn2025dion, pang2026htmuon}.
These methods typically integrate Muon's core orthogonalization step with standard adaptive optimizers or other mechanisms \citep{liu2025cosmos, si2025adamuon, zhang2025adagrad, wang2025conda}.

A crucial line of research investigates the convergence of Muon-type optimizers \citep{ma2026preconditioning, braun2026spectral, chen2025muon, liu2025muon, sfyraki2025lions, pethick2025training}.
We note that recent works have explored extending the Muon optimizer to manifold optimization problems \citep{bernstein2025manifolds,buchanan2025mmuonadmm,kexuefm-11215,kexuefm-11221,yang2026manifold}.

Another line of research has recently explored the implicit bias of Spectral Descent and Muon.
\citet{fan2026implicit} proved that these algorithms converge to spectral norm max-margin solutions in multi-class linear classification.
This geometric bias was extended to smooth homogeneous neural networks by \citet{gronich2026implicit}, and further analyzed in practical stochastic settings by \citet{li2026implicit}.

\paragraph{Decoupled Weight Decay.}
Decoupled weight decay, initially popularized by AdamW \citep{loshchilov2017decoupled} to overcome the limitations of $L_2$ regularization in adaptive methods, has become a standard technique in large-scale training.
Recently, the integration of decoupled weight decay into Muon has garnered special attention \citep{chen2024lion,chen2025muon, liu2025muon, sfyraki2025lions, pethick2025training}.
Incorporating decoupled weight decay intrinsically alters the optimization dynamics, exposing a structural equivalence with the Frank-Wolfe (FW) method over spectral norm constraints.
Building upon this insight, we establish corresponding convergence guarantees for Spectral Descent and Truncated Spectral Descent under decoupled weight decay in non-smooth settings.

\paragraph{Sharpness.}
Sharpness and weak sharpness conditions \citep{burke1993weak, luo1993error, studniarski1999weak, burke2002weak, ma2025second} are prevalent across many mathematical problems, serving as central tools for establishing fast local convergence.
Recently, sharpness has become a key ingredient in the analysis of nonconvex low-rank matrix recovery \citep{li2020nonconvex, charisopoulos2021low, ding2023sharpness, xu2024convergence} and phase retrieval \citep{duchi2019solving, davis2020nonsmooth, zheng2024adaptive}.
Motivated by these developments, we adopt sharpness as a key structural assumption in our analysis of non-smooth problems.

\section{Preliminaries}
\label{sec:preliminaries}
\subsection{Notations and Muon Optimizer}
\textbf{Notations.}
For matrices $X, Y \in \mathbb{R}^{n_1 \times n_2}$, the trace inner product is defined as $\langle X, Y \rangle := \text{Tr}(X^\top Y)$.
Let $\sigma_1(X) \ge \sigma_2(X) \ge \dots \ge \sigma_n(X) \ge 0$ denote the singular values of $X$ in descending order, where $n = \min(n_1, n_2)$.
The operator norm (or spectral norm) is denoted by $\|X\|_2 = \sigma_1(X)$.
Furthermore, for any integer $1 \le k \le n$, the Ky Fan $k$-norm of $X$, denoted by $\|X\|_{(k)}$, is defined as the sum of its $k$ largest singular values, i.e., $\|X\|_{(k)} = \sum_{i=1}^k \sigma_i(X)$.
For a matrix $\G\in\R^{n_1\times n_2}$, its matrix sign function is defined as
$\msgn(\G):=(\G\G^T)^{-1/2}\G$,
where $(\cdot)^{-1/2}$ denotes the Moore-Penrose inverse square root.
If $\U_r\mathbf{\Sigma_r}\V_r^T$ is the compact SVD of $\G$ with $\U_r\in\R^{n_1\times r}$, $\V_r\in\R^{n_2\times r}$ and $r=\text{rank}(\G)$,
then the definition yields $\msgn(\G)=\U_r\V_r^T$.
For an integer $1\leq s\leq r$, we define the set-valued rank-$s$ truncated matrix sign set by
$\mathfrak T_s(\G):=\left\{U_sV_s^\top:\G=U\operatorname{Diag}(\sigma)V^\top\text{ is an SVD with }\sigma_1\ge \sigma_2\ge\cdots\ge \sigma_n\ge0 \right\}$, where \(U_s\) and \(V_s\) denote the first \(s\) columns of \(U\) and \(V\), respectively. 
Equivalently, $\mathfrak T_s(G)$ contains all admissible choices of the top-$s$ singular directions. 
If $\sigma_s(G)>\sigma_{s+1}(G)$, with the convention $\sigma_{n+1}(G)=0$, then \(\mathfrak T_s(G)\) is a singleton, and we write its unique element as $\Tmsgn_s(G)$. 
If $\sigma_s(G)=\sigma_{s+1}(G)$, then the top-\(s\) singular subspace is not unique, and $\Tmsgn_s(G)$ should be understood as an arbitrary element of $\mathfrak T_s(G)$. 
In the TSD/RTSD-WD update, any $D_s\in\mathfrak T_s(G)$ may be selected, and the convergence analysis below holds uniformly for every such admissible selection.
The tangent space to the manifold of rank-$r$ matrices at $G$ is given by $T_{G} = \{U_rA^\top+BV_r^\top:A,B\in\R^{n_1\times r}\}$ with its orthogonal complement denoted by $T_{G}^\perp$.
Let $P_{U_r} = U_r U_r^\top$ and $P_{V_r} = V_r V_r^\top$.
The orthogonal projections of any matrix $M$ onto these spaces are directly given by $M_{T_G} = P_{U_r} M + M P_{V_r} - P_{U_r} M P_{V_r}$ and $M_{T_G^\perp} = (I - P_{U_r}) M (I - P_{V_r})$.
Let $\mathbb{S}_+^{n-1} = \{ \mathbf{x} \in \mathbb{R}^n : \|\mathbf{x}\|_2 = 1, x_i \ge 0, \forall i \}$ denote the non-negative unit sphere in $\mathbb{R}^n$.
The set of $k \times n$ matrices with orthonormal columns is denoted by $\mathcal{V}_{k, n} = \{ \mathbf{U} \in \mathbb{R}^{k \times n} : \mathbf{U}^\top \mathbf{U} = \mathbf{I}_n \}$.
We write $a\lesssim b$ to indicate that $a\leq C b$ for some absolute constant $C>0$, independent of the dimension and other relevant parameters (unless otherwise specified).

Throughout this work, we consider the following optimization problem
\begin{equation}
    \label{eq:general_problem}
    \min_{X\in\mathcal{X}} f(X),
\end{equation}
where $\mathcal{X} \subseteq \mathbb{R}^{n_1\times n_2}$ denotes the feasible set, accounting for both unconstrained and constrained scenarios, and $f:\mathbb{R}^{n_1\times n_2}\to\mathbb{R}$ is a proper, lower semicontinuous, and possibly non-smooth function.

Applying the \textit{Muon} optimizer to problem (\ref{eq:general_problem}) yields the following update step:
\begin{equation}
    \label{eq:Muon_update}
    \begin{aligned}
    &G^{(t)} \in \partial f(X^{(t)}),\\
    &B^{(t)} = \mu B^{(t-1)} + G^{(t)},\\
    &{X}^{(t+1)} = X^{(t)} - \eta_t\msgn(B^{(t)}).
    \end{aligned}
\end{equation}
Here, $\mu\in[0,1)$ is the momentum parameter, and $\eta_t$ is the step size at iteration $t$.
Exactly computing the matrix sign function requires a full Singular Value Decomposition (SVD), which is computationally prohibitive for large-scale problems.
In practice, Newton-Schulz iterations \citep{bernstein2024old, jordan6muon} are employed as an efficient alternative.
Moreover, \citet{amsel2025polar} recently introduced the Polar Express method for computing the matrix polar decomposition, demonstrating faster convergence and improved validation loss when integrated into the Muon optimizer for training large language models.

Separately, we introduce \textit{Muon with decoupled weight decay (MuonW)} \citep{sfyraki2025lions,pethick2025training} for problem (\ref{eq:general_problem}), which incorporates a weight decay term into the update step:
\begin{equation}
    \label{eq:MuonW_update}
    \begin{aligned}
    &G^{(t)} \in \partial f(X^{(t)}),\\
    &B^{(t)} = \mu B^{(t-1)} + G^{(t)},\\
    &{X}^{(t+1)} = X^{(t)} - \eta_t(\msgn(B^{(t)}) + \lambda X^{(t)}),
    \end{aligned}
\end{equation}
where $\lambda>0$ is the weight decay parameter.
This variant has been shown to achieve significantly better generalization performance in LLM training compared to Muon \citep{liu2025muon}.

\subsection{This work: Spectral Descent}
In our work, we focus on the momentum-free case ($\mu=0$) for both the Muon and MuonW optimizers.
In this regime, the Muon update simplifies to the following \textit{Spectral Descent} \citep{bernstein2024old} step:
\begin{equation}
	\label{eq:spectral_descent}
	{X}^{(t+1)} = X^{(t)} - \eta_t\msgn(\G^{(t)}), \quad \G^{(t)}\in\partial f(X^{(t)}).
\end{equation}
The update removes the momentum term and directly uses the current subgradient to compute the descent direction, thus serving as a simplified variant of Muon that retains its core descent mechanism.

We also consider a truncated variant of Spectral Descent, \textit{Truncated Spectral Descent}, which replaces the matrix sign function with its truncated counterpart:
\begin{equation}
	\label{eq:truncated_spectral_descent}
	{X}^{(t+1)} = X^{(t)} - \eta_t\Tmsgn_{s}(\G^{(t)}), \quad \G^{(t)}\in\partial f(X^{(t)}).
\end{equation}
Here $\text{Tmsgn}_{s}(G^{(t)})$ extracts the top-$s$ singular subspace.
This update strategy is identical to the Fanion-$s$ algorithm \citep{kravatskiy2025ky}, operating as a Linear Minimization Oracle (LMO) with respect to the dual of the Ky Fan $s$-norm.
We will show the convergence of Spectral Descent and Truncated Spectral Descent in section \ref{ssec:spectral_descent}.

Similarly, in Section \ref{ssec:regularized_spectral_descent_with_weight_decay}, we consider the regularized versions of Spectral Descent and Truncated Spectral Descent, further incorporating decoupled weight decay.
These approaches correspond to the momentum-free variant of MuonW. 
For the regularized algorithms below, we use the following convention.
Let
$$
    \mathsf{Sgn}_*(G):=\partial\|G\|_*,
    \qquad
    \mathsf{Sgn}_{(s)}(G):=\partial\|G\|_{(s)}.
$$
When the corresponding norm is differentiable at \(G\), these set-valued maps reduce to the usual matrix sign and truncated matrix sign. 
In particular, if \(G\) has full rank, then \(\mathsf{Sgn}_*(G)=\{\msgn(G)\}\), and if
\(\sigma_s(G)>\sigma_{s+1}(G)\), then \(\mathsf{Sgn}_{(s)}(G)=\{\Tmsgn_s(G)\}\).
In degenerate cases, the regularized algorithms use a KKT-compatible element of these subdifferentials.
The corresponding updates are formulated as:

\textit{Regularized Spectral Descent with Weight Decay (RSD-WD)}
\begin{equation}
    \label{eq:regularized_spectral_descent_wd}
    \begin{aligned}
    &\tilde{\mathbf{G}}^{(t)}
    \in
    \arg\min_{\mathbf{G}\in T(X^{(t)},\epsilon_t)}
    \left\{
    \|\mathbf{G}\|_*+\lambda\langle X^{(t)},\mathbf{G}\rangle
    \right\},\\
    &D^{(t)}\in \mathsf{Sgn}_*(\tilde{\mathbf{G}}^{(t)})
    \quad\text{such that}\quad
    D^{(t)}\in
    -\lambda X^{(t)}
    -N_{T(X^{(t)},\epsilon_t)}(\tilde{\mathbf{G}}^{(t)}),\\
    &X^{(t+1)}
    =
    X^{(t)}
    -
    \eta_t\left(D^{(t)}+\lambda X^{(t)}\right).
    \end{aligned}
\end{equation}
where $ T(X, \epsilon) := \text{conv}(\bigcup_{u \in N(X, \epsilon)} \partial f(u)) $, $N(X, \epsilon)$ is the $\epsilon$-neighborhood of $X$ and $N_{T(X^{(t)},\epsilon_t)}(\tilde{\mathbf{G}}^{(t)})$ is the normal cone of the neighborhood $T(X^{(t)},\epsilon_t)$ at $\tilde{\mathbf{G}}^{(t)}$.

\textit{Regularized Truncated Spectral Descent with Weight Decay (RTSD-WD)}
\begin{equation}
    \label{eq:regularized_truncated_spectral_descent_wd}
    \begin{aligned}
    &\tilde{\mathbf{G}}^{(t)}
    \in
    \arg\min_{\mathbf{G}\in T(X^{(t)},\epsilon_t)}
    \left\{
    \|\mathbf{G}\|_{(s)}+\lambda\langle X^{(t)},\mathbf{G}\rangle
    \right\},\\
    &D_s^{(t)}\in \mathsf{Sgn}_{(s)}(\tilde{\mathbf{G}}^{(t)})
    \quad\text{such that}\quad
    D_s^{(t)}\in
    -\lambda X^{(t)}
    -N_{T(X^{(t)},\epsilon_t)}(\tilde{\mathbf{G}}^{(t)}),\\
    &X^{(t+1)}
    =
    X^{(t)}
    -
    \eta_t\left(D_s^{(t)}+\lambda X^{(t)}\right).
    \end{aligned}
\end{equation}

Note that as $\epsilon_t \to 0$, the neighborhood $T(X^{(t)}, \epsilon_t)$ shrinks, and $\tilde{\mathbf{G}}^{(t)}$ approaches a subgradient of the original objective function at $X^{(t)}$.
In this context, 'regularization' designates a specific descent direction selection mechanism explicitly designed to stabilize optimization within non-smooth landscapes.
For non-smooth objectives, applying the matrix sign operator to an arbitrary subgradient may produce oscillatory updates.
The $\mathop{\arg\min}$ problems in Eqs. (\ref{eq:regularized_spectral_descent_wd}) and (\ref{eq:regularized_truncated_spectral_descent_wd}) avoid this by selecting a subgradient $\tilde{\mathbf{G}}^{(t)}$ from the neighborhood $T(X^{(t)}, \epsilon_t)$.
This subproblem is carefully crafted to balance the pursuit of a minimal-norm direction for stability (governed by $\|\mathbf{G}\|_*$ or $\|\mathbf{G}\|_{(s)}$) with a structural alignment penalty ($\langle X^{(t)}, \mathbf{G} \rangle$).
In doing so, the neighborhood-based optimization performs a form of \textit{spatial smoothing}.
This mechanism provides a surrogate for the \textit{temporal smoothing} of momentum, stabilizing the update trajectory and providing rigorous convergence guarantees in non-smooth setting.

\paragraph{Organization.}
The remainder of this paper is organized as follows.
In Section~\ref{sec:main_results}, we present our core theoretical results: Subsection~\ref{ssec:spectral_descent} establishes the convergence guarantees for standard Spectral Descent and its truncated variant, while Subsection~\ref{ssec:regularized_spectral_descent_with_weight_decay} extends this analysis to their regularized counterparts equipped with decoupled weight decay.
Section~\ref{sec:application} instantiates our theoretical framework to the specific problem of robust low-rank matrix recovery.
Subsequently, Section~\ref{section:experiments} provides extensive numerical experiments to validate our theoretical claims and demonstrate the empirical efficacy of Muon-type optimizers.
Finally, we conclude the paper in Section~\ref{sec:discussion}.

\section{Main Results}
\label{sec:main_results}
In this section, we present our core theoretical results.
We begin in Subsection~\ref{ssec:spectral_descent} by establishing the convergence guarantees for standard Spectral Descent (SD) \eqref{eq:spectral_descent} and Truncated Spectral Descent (TSD) \eqref{eq:truncated_spectral_descent}.
By omitting the momentum term, we directly isolate and analyze the core geometry-aware descent mechanism.
Subsequently, Subsection~\ref{ssec:regularized_spectral_descent_with_weight_decay} extends our analysis to regularized optimizers that incorporate decoupled weight decay.
Specifically, we prove the convergence of RSD-WD \eqref{eq:regularized_spectral_descent_wd} and TSD-WD \eqref{eq:regularized_truncated_spectral_descent_wd}, theoretically demonstrating that our proposed neighborhood selection mechanism effectively stabilizes the optimization dynamics.

Let $\mathcal{X}^* = \{ X \in \mathbb{R}^{n_1 \times n_2} : f(X) \leq f(Z), \forall Z \in \mathbb{R}^{n_1 \times n_2} \}$ denote the set of global minimizers of $f$, and let $f^*$ be the optimal value of $f$.
We consider the following assumptions on the objective function $f$.
\assumption[Convexity]
	\label{assumption:Convexity}
	We say a function $f:\R^{n_1\times n_2}\to\R$ is convex,
	if for any $X,Y\in\R^{n_1\times n_2}$ and $\theta\in[0,1]$, we have
	$$f(\theta X + (1-\theta)Y) \leq \theta f(X) + (1-\theta)f(Y).$$
\endassumption
\assumption[Lipschitz Continuity]
	\label{assumption:Lipschitz Continuity}
A function $f:\R^{n_1\times n_2}\to\R$ is said to be $L$-Lipschitz continuous
if there exists a constant $L>0$ such that
$$|f(X)-f(Y)| \leq L\|X-Y\|_F, \quad \forall X,Y\in\R^{n_1\times n_2}.$$
\endassumption
This is a mild and standard assumption in non-smooth optimization.
In contrast to the smoothness assumption (i.e., Lipschitz continuous gradients) which strictly requires the function to be differentiable, Assumption \ref{assumption:Lipschitz Continuity} accommodates non-smooth landscapes.
Moreover, this condition ensures that the subgradients of $f$ are uniformly bounded, i.e., $\|G\|_F \leq L$ for any $G \in \partial f(X)$ and all $X\in\mathbb{R}^{n_1\times n_2}$.
\assumption[Sharpness]
    \label{assumption:sharpness}
    We say a function $f:\mathbb{R}^{n_1\times n_2}\to\mathbb{R}$ is $\mu$-sharp with respect to the optimal set $\mathcal{X}^*$ if there exists a constant $\mu>0$ such that
    $$f(X)-f^* \geq \mu \cdot \mathrm{dist}(X, \mathcal{X}^*), \quad \forall X\in\mathbb{R}^{n_1\times n_2},$$
    where $\mathrm{dist}(X, \mathcal{X}^*) := \inf_{Y \in \mathcal{X}^*} \|X-Y\|_F$ represents the distance from $X$ to the optimal set $\mathcal{X}^*$ w.r.t. the Frobenius norm.
\endassumption
The $\mu$-sharpness condition characterizes the local geometry around the optimal solution set, naturally capturing the non-differentiable "V-shaped" valleys in non-smooth optimization landscapes.
Utilizing this property is crucial for establishing rigorous convergence bounds across various data science problems \citep{davis2018subgradient,davis2020nonsmooth,li2020nonconvex,xu2024convergence}.
Mathematically, this condition is directly satisfied by polyhedral convex functions (e.g., $L_1$-regularization and SVM hinge loss), spectral penalties like the nuclear norm, as well as linear programming and linear complementarity problems \citep{burke1993weak}.
\subsection{Convergence Analysis of Spectral Descent}
\label{ssec:spectral_descent}
We begin by analyzing the convergence properties of Spectral Descent (\ref{eq:spectral_descent}).
For the convergence analysis, we use the following convention. 
If $0\in \partial f(X^{(t)})$, then by convexity $X^{(t)}\in\mathcal X^*$, and we set $G^{(t)}=0$ and $X^{(k)}=X^{(t)}$ for all $k\geq t$. 
Otherwise, we choose a nonzero subgradient $G^{(t)}\in\partial f(X^{(t)})\setminus\{0\}$.
All rank-dependent quantities below are defined only for active iterations, namely iterations with $G^{(t)}\neq0$.

To establish theoretical guarantees in non-smooth setting, our theoretical framework relies on three standard assumptions: convexity \ref{assumption:Convexity}, Lipschitz continuity \ref{assumption:Lipschitz Continuity}, and sharpness \ref{assumption:sharpness}.
As formalized in Theorem \ref{theorem:spectral_descent_convergence}, employing a geometrically decaying step-size schedule enables the algorithm to achieve global linear convergence.

\begin{theorem}[Convergence of spectral descent]
    \label{theorem:spectral_descent_convergence}
    Suppose that $f:\mathbb{R}^{n_1\times n_2}\to\mathbb{R}$ satisfies Assumptions~\ref{assumption:Convexity}, \ref{assumption:Lipschitz Continuity}, and~\ref{assumption:sharpness}.
    Let $\{X^{(t)}\}_{t=0}^{T}$ be the sequence generated by Spectral Descent (\ref{eq:spectral_descent}) with step sizes $\{\eta_t\}_{t=0}^{T-1}$.
    Define the condition parameter $\kappa:=\mu/L \in (0,1)$.
    Then the following statements hold:
    \begin{enumerate}[nosep, leftmargin=*]
        \item \textnormal{Rank-dependent descent:}
        For each $0 \leq t < T$, we have
        \begin{equation}
            \mathrm{dist}^2(X^{(t+1)}, \mathcal{X}^*) \leq \mathrm{dist}^2(X^{(t)}, \mathcal{X}^*)+ \eta_t^2 r_t - 2\eta_t C_t\mathrm{dist}(X^{(t)}, \mathcal{X}^*),
        \end{equation}
        where $r_t:=\mathrm{rank}(\mathbf{G}^{(t)})$ for the chosen subgradient $\mathbf{G}^{(t)}\in\partial f(X^{(t)})$, and $C_t:=\kappa-\sqrt{r_t-1}\sqrt{1-\kappa^2}$.

        \item \textnormal{Global linear convergence:}
        Let $\bar{r}\geq 1$ be any upper bound on the subgradient ranks along the
        trajectory, i.e., $r_t\leq \bar{r}$ for all $0\leq t<T$. In particular, one may
        take the a priori bound $\bar{r}=\min\{n_1,n_2\}$, or the sharper a posteriori
        choice $\bar{r}=\max_{0\leq t<T}r_t$.
        Assume that the condition parameter satisfies $\kappa > \sqrt{1-\frac{1}{\bar{r}}}$.
        Define the constant $C := \kappa - \sqrt{\bar{r}-1}\sqrt{1-\kappa^2}$.
        If the step size is chosen as $\eta_t = \frac{C}{\bar{r}}\gamma^t\mathrm{dist}(X^{(0)}, \mathcal{X}^*)$, where the decay rate $\gamma$ satisfies $\gamma := \max\left\{\frac{C}{\sqrt{\bar{r}}}, \sqrt{1 - \frac{C^2}{\bar{r}}}\right\} \in (0,1)$, then the sequence satisfies
        \begin{equation}
            \mathrm{dist}(X^{(t)}, \mathcal{X}^*) \leq \gamma^t\mathrm{dist}(X^{(0)}, \mathcal{X}^*),\quad \forall\, t \in \{0, 1, \dots, T\}.
        \end{equation}
    \end{enumerate}
\end{theorem}

\begin{proof}
    Let $X_*^{(t)} \in \mathcal{X}^*$ denote the projection of $X^{(t)}$ onto the optimal solution set $\mathcal{X}^*$, such that $\mathrm{dist}(X^{(t)}, \mathcal{X}^*) = \|X^{(t)} - X_*^{(t)}\|_F$.
    We define the residual vector at iteration $t$ as $\mathbf{R}^{(t)} := X^{(t)} - X_*^{(t)}$.
    Let $\mathbf{G}^{(t)} \in \partial f(X^{(t)})$ be the subgradient used in the update.
    Consider the compact SVD $\mathbf{G}^{(t)}=\sum_{i=1}^{r_t}\sigma^{(t)}_i u^{(t)}_i (v^{(t)}_i)^\top$, where $r_t = \mathrm{rank}(\mathbf{G}^{(t)}) \leq \bar{r}$.

    The update rule implies $X^{(t+1)} = X^{(t)} - \eta_t \msgn(\mathbf{G}^{(t)})$.
	Consequently, the residual dynamics (in terms of the squared Frobenius norm) follow
    \begin{align}
        \mathrm{dist}^2(X^{(t+1)}, \mathcal{X}^*) & \leq \|X^{(t+1)} - X_*^{(t)}\|^2_F \nonumber\\
        & = \|X^{(t)} - X_*^{(t)} - \eta_t \msgn(\mathbf{G}^{(t)}) \|^2_F \nonumber\\
        & = \|\mathbf{R}^{(t)}\|^2_F + \eta_t^2\|\msgn(\mathbf{G}^{(t)})\|^2_F - 2\eta_t\langle \mathbf{R}^{(t)}, \msgn(\mathbf{G}^{(t)})\rangle.
        \label{eq:residual_update}
    \end{align}
    We proceed to bound the second and third terms on the right-hand side of \eqref{eq:residual_update}.

    For the second term, by the definition of the matrix sign function, we have
    \begin{equation}
        \label{ineq:second_term}
        \|\msgn(\mathbf{G}^{(t)})\|^2_F = \|U^{(t)} {V^{(t)}}^\top\|^2_F = \sum_{i=1}^{r_t} 1^2 = r_t.
    \end{equation}

    Then, we aim to lower bound the inner product $\langle \mathbf{R}^{(t)}, \msgn(\mathbf{G}^{(t)})\rangle$.
    First, using the convexity and sharpness assumption yields
    \begin{equation}
        \label{ineq:third_term_1}
        \langle \mathbf{R}^{(t)}, \mathbf{G}^{(t)}\rangle \geq f(X^{(t)}) - f(X^{(t)}_*) \geq \mu\|\mathbf{R}^{(t)}\|_F.
    \end{equation}
    Second, extending the singular vectors to orthonormal bases of $\mathbb{R}^{n_1}$ and $\mathbb{R}^{n_2}$, Parseval's identity implies
    \begin{equation}
        \label{ineq:third_term_2}
        \|\mathbf{R}^{(t)}\|^2_F \geq \sum_{i=1}^{r_t} \langle \mathbf{R}^{(t)}, u^{(t)}_i (v^{(t)}_i)^\top \rangle^2.
    \end{equation}
    Third, by the $L$-Lipschitz continuity, we have
    \begin{equation}
        \label{ineq:third_term_3}
        \|\mathbf{G}^{(t)}\|^2_F = \sum_{i=1}^{r_t}(\sigma^{(t)}_i)^2 \leq L^2.
    \end{equation}
    Combining \eqref{ineq:third_term_1}, \eqref{ineq:third_term_2}, and \eqref{ineq:third_term_3}, we apply Lemma \ref{lemma:lower_bound_descent_term_SD} to obtain the lower bound for the descent term
    \begin{equation}
		\label{ineq:third_term}
        \langle \mathbf{R}^{(t)}, \msgn(\mathbf{G}^{(t)})\rangle \geq \left(\kappa - \sqrt{r_t-1}\sqrt{1-\kappa^2}\right)\|\mathbf{R}^{(t)}\|_F.
    \end{equation}

    Substituting \eqref{ineq:second_term} and \eqref{ineq:third_term} into \eqref{eq:residual_update}, we obtain
    \begin{equation} \label{eq:recurrence_relation}
        \mathrm{dist}^2(X^{(t+1)}, \mathcal{X}^*) \leq \mathrm{dist}^2(X^{(t)}, \mathcal{X}^*) + \eta_t^2 r_t - 2\eta_t C_t \mathrm{dist}(X^{(t)}, \mathcal{X}^*).
    \end{equation}
    This proves the first statement.

Since $r_t\leq \bar{r}$ and $C_t\geq C$, the rank-dependent recursion \eqref{eq:recurrence_relation} implies the uniform upper bound
    \begin{equation}
        \label{eq:uniform_recursion}
        \mathrm{dist}^2(X^{(t+1)}, \mathcal{X}^*) \leq \mathrm{dist}^2(X^{(t)}, \mathcal{X}^*) + \eta_t^2 \bar{r} - 2\eta_t C \mathrm{dist}(X^{(t)}, \mathcal{X}^*).
    \end{equation}
    Let $\Delta_t:=\mathrm{dist}(X^{(t)}, \mathcal{X}^*)$ and set $B_t:=\gamma^t\Delta_0$ with $\Delta_0=\mathrm{dist}(X^{(0)}, \mathcal{X}^*)$.
    Choose the step size $\eta_t=\frac{C}{\bar{r}}B_t$.
    We show by induction that $\Delta_t\leq B_t$ for all $t \in \{0, \dots, T\}$.
    The base case $t=0$ is immediate.
    Assume $\Delta_t\leq B_t$ holds for some $t \geq 0$. We now prove it for $t+1$.

Define the quadratic function $h(x):=x^2-2\eta_t Cx+\eta_t^2 \bar{r}$.
Then \eqref{eq:uniform_recursion} implies $\Delta_{t+1}^2\le h(\Delta_t)$.
Since $h$ is convex and $\Delta_t\in[0,B_t]$, we have
$$
h(\Delta_t)\le \max\{h(0),h(B_t)\}.
$$
A direct computation gives
$$
h(0)=\eta_t^2 \bar{r}=\frac{C^2}{\bar{r}}B_t^2,
\qquad
h(B_t)=\Bigl(1-\frac{C^2}{\bar{r}}\Bigr)B_t^2.
$$
Recall decay rate $\gamma = \max\left\{\frac{C}{\sqrt{\bar{r}}}, \sqrt{1 - \frac{C^2}{\bar{r}}}\right\}$.
We obtain the upper bound
$$
\Delta_{t+1}^2\le h(\Delta_t)\le \max\{h(0), h(B_t)\}=\gamma^2B_t^2,
$$
which implies $\Delta_{t+1}\le B_{t+1}$ and closes the induction.
This proves $\mathrm{dist}(X^{(t)}, \mathcal{X}^*)\le \gamma^t\mathrm{dist}(X^{(0)}, \mathcal{X}^*)$ for all $t \in \{0, \dots, T\}$.
\end{proof}
\begin{remark}[Rank-dependent descent and linear convergence]
Theorem \ref{theorem:spectral_descent_convergence} characterizes the precise optimization dynamics of Spectral Descent from both local and global perspectives.
First, the rank-dependent descent reveals that the per-step contraction is explicitly governed by the subgradient rank $r_t$: a smaller rank yields a larger constant $C_t$, resulting in a sharper one-step reduction of the residual.
Second, by leveraging the maximum trajectory rank $\bar{r}$, we establish global linear convergence.
This ensures that the distance to the optimal solution set decays steadily at a strict geometric rate $\gamma^t$.
\end{remark}
\begin{remark}[condition parameter $\kappa$ and worst-case geometry]
Fundamentally, the condition parameter $\kappa > \sqrt{1 - 1/\bar{r}}$ (which degrades to $\sqrt{1 - 1/n}$ in the extreme case where $\bar{r} = n$) is required to ensure that the lower bound of the single-step descent term remains strictly positive.
While Lemma \ref{lemma:lower_bound_descent_term_SD} proves that this bound is tight, the worst-case scenario achieving it exhibits a highly adversarial geometry.
Specifically, the construction in Eq.~(\ref{eq:tightness_construction}) corresponds to a degenerate limiting configuration in which the subgradient is nearly rank one, while the residual matrix has uniformly negative projections along the remaining orthogonal singular directions.
In practice, such extreme conditions rarely occur.
This explains the empirical observation that Spectral Descent converges rapidly even when the theoretical condition on $\kappa$ is violated, see Figure~\ref{fig:12_grid_comparison}.

However, it is possible to relax this pessimistic deterministic constraint of $\kappa$.
By leveraging specific problem structures and probabilistic tools, one could establish a tighter, structure-aware alignment condition $\langle \mathbf{R}^{(t)}, \msgn(\mathbf{G}^{(t)})\rangle \geq \beta \|\mathbf{R}^{(t)}\|_F$, where $\beta > 0$ is a refined and even dimension-independent constant.
We leave the exploration of such fine-grained convergence guarantees to future work.
Moreover, it is worth noting that we avoid the constraint of the condition parameter $\kappa$ by introducing regularization and weight decay in Section \ref{ssec:regularized_spectral_descent_with_weight_decay}.
\end{remark}

\begin{remark}[Choice of step size $\eta_t$]
   Note that the linear convergence in Theorem \ref{theorem:spectral_descent_convergence} requires a step size $\eta_t$ proportional to the initial distance to the optimal set, i.e., $\mathrm{dist}(X^{(0)}, \mathcal{X}^*)$.
   While utilizing this distance is a standard convention to establish convergence rates in theory \citep{li2020nonconvex,xu2024convergence}, it is typically unknown in practice.
   Nevertheless, our theoretical result provides clear guidance for practical implementations.
   In practice, we can simply absorb this unknown distance into the proportionality constant, adopting the practical geometric decay schedule $\eta_t \propto \gamma^t$.
   As shown in Figure~\ref{fig:12_grid_comparison}, this practical step-size scheme is empirically consistent with the linear convergence behavior predicted by the theory.
\end{remark}
We now construct an explicit example that strictly satisfies the (worst-case) conditions required in our theorem.
\begin{example}[Worst-case Robust Matrix Estimation]
Consider the following robust matrix estimation problem under an adversarial continuous measurement setting:
\begin{equation}
\label{eq:robust_matrix_estimation}
\min_{X\in\R^{n_1\times n_2}} f(X):= \max_{A\in\mathcal{X}} |\langle X, A \rangle - y_{A} |,
\end{equation}
where $\mathcal{X} = \{A\in\R^{n_1\times n_2}:A=U\text{diag}(s)V^\top, U\in\mathcal{N}_{U}, V\in\mathcal{N}_{V}, s\in\mathcal{N}_{\sigma}\}$.
Here, $\mathcal{N}_{\sigma}$ is an $\epsilon$-net of the non-negative unit sphere $\mathbb{S}_+^{n-1}$, while $\mathcal{N}_{U}$ and $\mathcal{N}_{V}$ denote the $\epsilon$-nets of the Stiefel manifolds $\mathcal{V}_{n_1, n}$ and $\mathcal{V}_{n_2, n}$, respectively.
It can be shown that $f(X)$ is convex, globally Lipschitz continuous with $L=1$, and satisfies the sharpness condition with $\mu=1-(\frac{3}{2}\epsilon^2+2\epsilon)$.
Provided $\epsilon$ is chosen sufficiently small such that the condition parameter satisfies $\kappa = \mu/L > \sqrt{1 - 1/n}$, this setup strictly meets the landscape requirement of Theorem \ref{theorem:spectral_descent_convergence}.
While exact subgradient evaluation over the dense $\epsilon$-nets $\mathcal{X}$ makes practical algorithmic optimization intractable, this formulation acts as a purely theoretical vehicle.
It mathematically guarantees that the worst-case bound is strictly attainable.
\end{example}

Next, we turn to the convergence analysis of Truncated Spectral Descent (\ref{eq:truncated_spectral_descent}), which incorporates a rank truncation mechanism.
\begin{theorem}[Convergence of Truncated Spectral Descent]
    \label{theorem:truncated_spectral_descent_convergence}
    Suppose that $f:\mathbb{R}^{n_1\times n_2}\to\mathbb{R}$ satisfies Assumptions~\ref{assumption:Convexity}, \ref{assumption:Lipschitz Continuity}, and~\ref{assumption:sharpness}.
    Let $\{X^{(t)}\}_{t=0}^{T}$ be the sequence generated by Truncated Spectral Descent \eqref{eq:truncated_spectral_descent} with step sizes $\{\eta_t\}_{t=0}^{T-1}$ and a fixed truncation level $s$.
    Assume that $1\le s\le r_t:=\operatorname{rank}(\G^{(t)})$ for all $0\le t<T$.
    Let $\kappa := \mu/L \in (0,1)$ be the condition parameter and define the truncation-dependent constants $\alpha_t:=\min\left\{1,\frac{s}{\sqrt{r_t}}\right\}$ and $\widetilde C_t:=\kappa\alpha_t-\sqrt{s-\alpha_t^2}\sqrt{1-\kappa^2}$.
    Then the following statements hold:
     \begin{enumerate}[nosep, leftmargin=*]
        \item \textnormal{Truncation-dependent descent:}
        For each $0 \leq t < T$, the iterates satisfy
        \begin{equation}
            \mathrm{dist}^2(X^{(t+1)}, \mathcal{X}^*) \leq \mathrm{dist}^2(X^{(t)}, \mathcal{X}^*) + \eta_t^2 s - 2\eta_t \widetilde C_t\mathrm{dist}(X^{(t)}, \mathcal{X}^*).
        \end{equation}

        \item \textnormal{Global linear convergence:}
        Let $\bar r:=\max_{0\le t<T}r_t$ and $\alpha_s:=\min\left\{1,\frac{s}{\sqrt{\bar r}}\right\}$.
        Define $\widetilde C:=\kappa\alpha_s-\sqrt{s-\alpha_s^2}\sqrt{1-\kappa^2}$.
        Provided that $\kappa>\sqrt{1-\frac{\alpha_s^2}{s}}$, if the step size is chosen as $\eta_t = \frac{\widetilde C}{s}\gamma^t\mathrm{dist}(X^{(0)}, \mathcal{X}^*)$, where the decay rate satisfies $\gamma :=\max\left\{\frac{\widetilde C}{\sqrt{s}},\sqrt{1-\frac{\widetilde C^2}{s}}\right\}\in(0,1)$, then the sequence converges linearly as
        \begin{equation}
            \mathrm{dist}(X^{(t)}, \mathcal{X}^*) \leq \gamma^t\mathrm{dist}(X^{(0)}, \mathcal{X}^*),\quad \forall\, t \in \{0, 1, \dots, T\}.
        \end{equation}
    \end{enumerate}
\end{theorem}
\begin{proof}
	The proof strategy is identical to Theorem \ref{theorem:spectral_descent_convergence}.
	The key difference lies in the lower bound of the descent term $\langle \mathbf{R}^{(t)}, \Tmsgn_s(G^{(t)})\rangle$.
	Detailed derivations are presented in Appendix \ref{app:lower_bound_descent_term_TSD}.
\end{proof}
It is worth noting that for a fixed upper rank bound $\bar r$, Theorem~\ref{theorem:truncated_spectral_descent_convergence} gives the conditioning threshold
$$
    \kappa>\sqrt{1-\frac{\alpha_s^2}{s}},
    \qquad
    \alpha_s=\min\left\{1,\frac{s}{\sqrt{\bar r}}\right\}.
$$
This threshold is minimized when $s$ is chosen on the order of $\sqrt{\bar r}$.
In particular, the choice $s=\sqrt{\bar r}$ formally yields the relaxed threshold
$\kappa>\sqrt{1-1/\sqrt{\bar r}}$, whereas $s=\bar r$ recovers the standard SD-type threshold
$\kappa>\sqrt{1-1/\bar r}$.
Furthermore, the update directions of SD and TSD are guided by different steepest descent geometries: SD corresponds to the spectral norm, whereas TSD is governed by a mixed-norm constraint involving both the spectral and nuclear norms (Ky Fan norms).
Consequently, TSD and SD offer distinct structural advantages, which may allow each to achieve superior performance in different problem settings.

\subsection{Convergence Analysis of Regularized Spectral Descent with Weight Decay}
\label{ssec:regularized_spectral_descent_with_weight_decay}
In this subsection, we give a convergence analysis for RSD-WD \eqref{eq:regularized_spectral_descent_wd} and RTSD-WD \eqref{eq:regularized_truncated_spectral_descent_wd}.
The key insight is to establish an equivalence between these regularized spectral descent algorithms and a non-smooth version of the Frank-Wolfe method \citep{frank1956algorithm, jaggi2013revisiting}.
We show that both RSD-WD and RTSD-WD can be viewed as specific instances of the Conditional Subgradient method \citep{ravi2019deterministic} (Algorithm \ref{alg:frank_wolfe}) with a modified definition of the set $T(X^{(t)},\epsilon_t)$.

\begin{algorithm}[htbp]
\caption{Conditional Subgradient Method}
\label{alg:frank_wolfe}
\begin{algorithmic}
        \State \textbf{Input:} convex set $\mathcal{X}$, $X^{(0)} \in \mathcal{X}$, total steps $T$, step sizes $\{\gamma_t\}_{t=0}^{T-1}$ and tolerance parameters $\{\epsilon_t\}_{t=0}^{T-1}$
		\For{$t = 0, 1, \cdots, T-1$}
        \State Compute $g^{(t)} \in \mathcal{X}$ such that
         $$\max_{G\in T(X^{(t)},\epsilon_t)} \langle g^{(t)}-X^{(t)}, \G\rangle\leq\min_{S\in\mathcal{X}} \max_{G\in T(X^{(t)},\epsilon_t)} \langle S-X^{(t)}, \G\rangle$$
		\State ${X}^{(t+1)} = (1-\gamma_t)X^{(t)} + \gamma_t g^{(t)}$.
        \EndFor
\end{algorithmic}
\end{algorithm}
\begin{lemma}[Equivalence between RSD-WD/RTSD-WD and Conditional Subgradient Method]
\label{lemma:equivalence_spectral_descent_wd_frank_wolfe}
Assume that \(f\) satisfies Assumption~\ref{assumption:Convexity} and $\mathcal T_t:=T(X^{(t)},\epsilon_t)$ is nonempty, compact, and convex.

\textnormal{(i) RSD-WD.}
Let $\mathcal X_\lambda:=\left\{X\in\mathbb R^{n_1\times n_2}:\|X\|_2\le \frac1\lambda\right\}$.
Suppose $X^{(t)}\in\mathcal X_\lambda$. 
Let $\tilde{\mathbf G}^{(t)}\in\arg\min_{\mathbf G\in\mathcal T_t}\left\{\|\mathbf G\|_*+\lambda\langle X^{(t)},\mathbf G\rangle\right\}$, and choose \(D^{(t)}\in\partial\|\tilde{\mathbf G}^{(t)}\|_*\) such that $0\in D^{(t)}+\lambda X^{(t)}+N_{\mathcal T_t}(\tilde{\mathbf G}^{(t)})$.
Define $g^{(t)}:=-\frac1\lambda D^{(t)}$.
Then $g^{(t)}\in\arg\min_{S\in\mathcal X_\lambda}\max_{\mathbf G\in\mathcal T_t}\left\langle S-X^{(t)},\mathbf G\right\rangle$.
Consequently, if \(\gamma_t=\lambda\eta_t\in[0,1]\), then the RSD-WD update
$$X^{(t+1)}=X^{(t)}-\eta_t\left(D^{(t)}+\lambda X^{(t)}\right)$$
coincides with the Conditional Subgradient update
$$X^{(t+1)} =(1-\gamma_t)X^{(t)}+\gamma_t g^{(t)}.$$

\textnormal{(ii) RTSD-WD.}
Let $\mathcal X_{\lambda,s}:=\left\{X\in\mathbb R^{n_1\times n_2}:\|X\|_2\le \frac1\lambda,\ \|X\|_*\le \frac{s}{\lambda}\right\}$.
Suppose \(X^{(t)}\in\mathcal X_{\lambda,s}\). 
Let $\tilde{\mathbf G}^{(t)}\in\arg\min_{\mathbf G\in\mathcal T_t}\left\{\|\mathbf G\|_{(s)}+\lambda\langle X^{(t)},\mathbf G\rangle\right\}$, and choose \(D_s^{(t)}\in\partial\|\tilde{\mathbf G}^{(t)}\|_{(s)}\) such that
$0\in D_s^{(t)}+\lambda X^{(t)}+N_{\mathcal T_t}(\tilde{\mathbf G}^{(t)})$.
Define $g^{(t)}:=-\frac1\lambda D_s^{(t)}$.
Then $g^{(t)}\in\arg\min_{S\in\mathcal X_{\lambda,s}}\max_{\mathbf G\in\mathcal T_t}\left\langle S-X^{(t)},\mathbf G\right\rangle$.
Consequently, if \(\gamma_t=\lambda\eta_t\in[0,1]\), then the RTSD-WD update
$$X^{(t+1)}=X^{(t)}-\eta_t\left(D_s^{(t)}+\lambda X^{(t)}\right)$$
coincides with the Conditional Subgradient update
$$X^{(t+1)}=(1-\gamma_t)X^{(t)}+\gamma_t g^{(t)}.$$
\end{lemma}

The proof of Lemma \ref{lemma:equivalence_spectral_descent_wd_frank_wolfe} is deferred to Appendix \ref{proof:equivalence_spectral_descent_wd_frank_wolfe}.
A direct consequence of Lemma~\ref{lemma:equivalence_spectral_descent_wd_frank_wolfe} is the following feasibility-invariance property. 

\begin{corollary}[Feasibility invariance]
\label{cor:feasibility-invariance}
Consider the RSD-WD update in Lemma~\ref{lemma:equivalence_spectral_descent_wd_frank_wolfe}. Suppose that
\(X^{(0)}\in\mathcal X_\lambda\) and \(\gamma_t=\lambda\eta_t\in[0,1]\) for
all \(t\). Then \(X^{(t)}\in\mathcal X_\lambda\) for all \(t\ge 0\).

Similarly, for the RTSD-WD update, if
\(X^{(0)}\in\mathcal X_{\lambda,s}\) and
\(\gamma_t=\lambda\eta_t\in[0,1]\) for all \(t\), then
\(X^{(t)}\in\mathcal X_{\lambda,s}\) for all \(t\ge 0\).
\end{corollary}

\begin{proof}
We prove the RSD-WD case. The RTSD-WD case is identical. By
Lemma~\ref{lemma:equivalence_spectral_descent_wd_frank_wolfe}, the update can be written as
\[
X^{(t+1)}=(1-\gamma_t)X^{(t)}+\gamma_t g^{(t)},
\]
where \(g^{(t)}\in\mathcal X_\lambda\). Since \(\mathcal X_\lambda\) is convex
and \(\gamma_t\in[0,1]\), \(X^{(t+1)}\in\mathcal X_\lambda\) whenever
\(X^{(t)}\in\mathcal X_\lambda\). The claim follows by induction from
\(X^{(0)}\in\mathcal X_\lambda\).
\end{proof}

The convergence analysis relies on the Generalized curvature constant \citep{ravi2019deterministic} defined as
\begin{equation}
	\mathcal{C}_f(\epsilon) = \sup_{\substack{X,S\in\mathcal{X}, \gamma\in(0,1] \\ Y=(1-\gamma)X+\gamma S}}\min_{\substack{\G\in T(X,\epsilon)}} \frac{2}{\gamma^2} (f(Y) - f(X) - \langle Y-X, \G\rangle).
\end{equation}
To ensure theoretical consistency, we define the $\epsilon$-neighborhood under the Frobenius norm as $N(X, \epsilon) = \{Y \in \mathbb{R}^{n_1 \times n_2} : \|Y - X\|_F \leq \epsilon\}$.
In specific applications, this neighborhood geometry can be tailored to facilitate efficient computation.
Now, we provide an upper bound for the curvature constant $\mathcal{C}_f(\epsilon)$.

\begin{proposition}
\label{prop:boundedness_curvature_matrix}
Let $f: \mathbb{R}^{n_1 \times n_2} \to \mathbb{R}$ be a function defined on set $\mathcal{X}_\lambda := \{X\in\R^{n_1\times n_2}:\|X\|_2\leq \frac{1}{\lambda}\}$.
Assume that $f$ satisfies Assumption \ref{assumption:Convexity} and \ref{assumption:Lipschitz Continuity} with Lipschitz constant $L>0$.
For any $\epsilon > 0$, we have

\begin{equation}
	\mathcal{C}_f(\epsilon) \leq \frac{16L n}{\lambda^2\epsilon}.
\end{equation}
\end{proposition}
The proof is postponed to Appendix \ref{proof:boundedness_curvature_matrix}.

We are now ready to present the convergence guarantee for RSD-WD \eqref{eq:regularized_spectral_descent_wd}.
\begin{theorem}
    \label{theorem:convergence_SD-WD}
    Let $\mathcal{X}_\lambda = \{X\in\R^{n_1\times n_2}:\|X\|_2\leq \frac{1}{\lambda}\}$ be the feasible set for some $\lambda > 0$.
    Suppose that the function $f:\mathbb{R}^{n_1\times n_2}\to\mathbb{R}$ satisfies Assumptions \ref{assumption:Convexity}, \ref{assumption:Lipschitz Continuity} with Lipschitz constant $L>0$ and satisfies the Sharpness condition (Assumption \ref{assumption:sharpness}) restricted over the feasible set $\mathcal{X}_\lambda$ with parameter $\mu>0$ relative to the optimal solution set $\mathcal{X}^* := \arg\min_{X\in\mathcal{X}_\lambda} f(X)$.
    Assume that the initial point satisfies $X^{(0)}\in\mathcal X_\lambda$.
    Let $\{X^{(t)}\}_{t=0}^{T}$ be the sequence generated by the Regularized Spectral Descent with Weight Decay \eqref{eq:regularized_spectral_descent_wd} with the step size $\eta_t = \frac{2}{\lambda(t+3)}$ and the tolerance parameter $\epsilon_t = \frac{2\sqrt{n}}{\sqrt{\lambda}} \sqrt{\eta_t}$.
    Then, the sequence $\{X^{(t)}\}_{t=0}^{T}$ satisfies
    \begin{equation}
        \mathrm{dist}(X^{(T)}, \mathcal{X}^*) \leq \frac{2}{\mu(T+2)(T+1)}\Delta_f^{0} + \frac{32L \sqrt{2n}}{3\mu\lambda} \frac{(T+3)^{3/2}}{(T+2)(T+1)}.
    \end{equation}
    where $\Delta_f^{0} = f(X^{(0)}) - \min_{X\in\mathcal{X}_\lambda} f(X)$ is the initial optimality gap.
\end{theorem}
\begin{remark}[Impact of decoupled weight decay on convergence]
    Under identical Lipschitz and sharpness conditions, the regularized variant with decoupled weight decay (RSD-WD) achieves only a sublinear $\mathcal{O}(1/\sqrt{T})$ rate in terms of parameter distance, whereas standard Spectral Descent enjoys a fast global linear convergence rate (Theorem \ref{theorem:spectral_descent_convergence}).
    This deceleration is a direct consequence of the decoupled weight decay.
    As shown in Lemma \ref{lemma:equivalence_spectral_descent_wd_frank_wolfe}, weight decay implicitly transforms the unconstrained descent step into a Frank-Wolfe convex combination update, i.e., $X^{(t+1)} = (1-\gamma_t)X^{(t)} + \gamma_t g^{(t)}$.
    Rather than taking a free gradient step, this update strictly interpolates between the current iterate and an extreme point of the feasible set.
    Because the step size $\gamma_t$ must decay over iterations, the algorithm is restricted to progressively smaller updates.
    This structural restriction inherently bottlenecks the optimization process, canceling out the acceleration benefits of the sharpness condition.

    However, this structural restriction is precisely what endows the algorithm with critical theoretical and practical advantages.
By enforcing this strict convex combination, the regularized update intrinsically bounds the trajectory within the feasible set.
This effectively acts as a spatial stabilization mechanism that entirely circumvents the stringent condition parameter assumption required for standard Spectral Descent.
Consequently, decoupled weight decay fundamentally trades raw convergence speed for unconditional algorithmic stability.
\end{remark}
The proof of Theorem \ref{theorem:convergence_SD-WD} adapts the analysis strategy from \citet{jaggi2013revisiting,ravi2019deterministic} to our matrix setting, leveraging the bounded curvature constant established in Proposition \ref{prop:boundedness_curvature_matrix}.
\begin{proof}[Proof of Theorem \ref{theorem:convergence_SD-WD}]
    Let $X^*$ be the optimal solution to problem \eqref{eq:general_problem} and define $\Delta_f^0 = f(X^{(0)}) - f(X^*)$.
	Since $X^{(0)}\in\mathcal{X}_\lambda$, by Lemma~\ref{lemma:equivalence_spectral_descent_wd_frank_wolfe} and Corollary~\ref{cor:feasibility-invariance}, the
RSD-WD update can be written as
$$X^{(t+1)} =(1-\gamma_t)X^{(t)}+\gamma_t g^{(t)},\qquad \gamma_t=\lambda\eta_t,$$
where $g^{(t)}\in\arg\min_{S\in\mathcal X_\lambda}\max_{\mathbf G\in T(X^{(t)},\epsilon_t)}\langle S-X^{(t)},\mathbf G\rangle$ and $X^{(t)}\in\mathcal X_\lambda$ for all $t$.

    From the definition of the curvature constant, we have
	\begin{equation}
		\begin{aligned}
		f(X^{(t+1)}) = &f((1-\gamma_t)X^{(t)} + \gamma_t g^{(t)})\\
		\leq &f(X^{(t)}) + \gamma_t \max_{G\in T(X^{(t)},\epsilon_t)}\langle g^{(t)} - X^{(t)}, \G \rangle + \frac{\gamma_t^2}{2} \mathcal{C}_f(\epsilon_t).
		\end{aligned}
	\end{equation}
	Since $g^{(t)}$ minimizes the linearized subproblem, we obtain
	\begin{equation}
		f(X^{(t+1)}) \leq f(X^{(t)}) + \gamma_t \min_{S\in\mathcal{X}_\lambda}\max_{G\in T(X^{(t)},\epsilon_t)}\langle S - X^{(t)}, \G \rangle + \frac{\gamma_t^2}{2} \mathcal{C}_f(\epsilon_t).
	\end{equation}
	By Lemma~\ref{lemma:epsilon-subgradient-inequality} and the definition of $X^*$, we have
	\begin{equation}
		f(X^{t+1}) \leq f(X^{t}) + \gamma_t (f(X^*) - f(X^{t}) + 2L\epsilon_t) + \frac{\gamma_t^2}{2} \mathcal{C}_f(\epsilon_t).
	\end{equation}
	Rearranging terms, we obtain
	\begin{equation}
		\begin{aligned}
		\Delta_f^{t+1} \leq &(1-\gamma_t)\Delta_f^{t} + 2L\gamma_t\epsilon_t + \frac{\gamma_t^2}{2} \mathcal{C}_f(\epsilon_t)\\
		\leq &(1-\gamma_t)\Delta_f^{t} + 2L\gamma_t\epsilon_t + \frac{8L n}{\lambda^2\epsilon_t}\gamma_t^2\\
		= &(1-\gamma_t)\Delta_f^{t} + 8L\frac{\sqrt{n}}{\lambda}\gamma_t^{3/2},
		\end{aligned}
	\end{equation}
	where the second inequality follows from Proposition \ref{prop:boundedness_curvature_matrix} and the last equality holds by substituting $\epsilon_t = \frac{2\sqrt{n}}{\sqrt{\lambda}} \sqrt{\eta_t}= \frac{2\sqrt{n}}{\lambda} \sqrt{\gamma_t}$.
	Unrolling the recursion, we obtain
	\begin{equation}
		\Delta_f^{T} \leq \prod_{t=0}^{T-1}(1-\gamma_t)\Delta_f^{0} + 8L\frac{\sqrt{n}}{\lambda}\sum_{t=0}^{T-1}\left(\gamma_t^{3/2}\prod_{k=t+1}^{T-1}(1-\gamma_k)\right).
	\end{equation}
    It is straightforward to verify that $\prod_{t=0}^{T-1}(1-\gamma_t) = \frac{2}{(T+2)(T+1)}$.
    Using the step size $\gamma_t = \frac{2}{t+3}$ and the equality $\prod_{k=t+1}^{T-1}(1-\gamma_k) = \frac{(t+2)(t+3)}{(T+2)(T+1)}$, the second term can be bounded via an integral approximation:
\begin{equation}
    \sum_{t=0}^{T-1}\left(\gamma_t^{3/2}\prod_{k=t+1}^{T-1}(1-\gamma_k)\right) \leq \frac{2\sqrt{2}}{(T+2)(T+1)} \int_{0}^{T} \sqrt{x+3} \, dx \leq \frac{4\sqrt{2}}{3} \frac{(T+3)^{3/2}}{(T+2)(T+1)}.
\end{equation}
Substituting this back into the unrolled recursion, the optimality gap is bounded by
\begin{equation}
    \Delta_f^{T} \leq \frac{2\Delta_f^{0}}{(T+2)(T+1)} + \frac{32L \sqrt{2n}}{3\lambda} \frac{(T+3)^{3/2}}{(T+2)(T+1)}.
\end{equation}
Finally, invoking the sharpness condition (Assumption \ref{assumption:sharpness}), we have $\mathrm{dist}(X^{(T)}, \mathcal{X}^*) \leq \frac{1}{\mu} \Delta_f^{T}$, which completes the proof.
\end{proof}
Similarly, we can establish the convergence guarantee for RTSD-WD \eqref{eq:regularized_truncated_spectral_descent_wd}.

\begin{theorem}
    \label{theorem:convergence_RTSD-WD}
     Let $\mathcal{X}_{\lambda,s}= \{X\in\R^{n_1\times n_2}:\|X\|_2\leq \frac{1}{\lambda}, \|X\|_* \leq \frac{s}{\lambda}\}$ be the feasible set for some $\lambda > 0$.
    Suppose that the function $f:\mathbb{R}^{n_1\times n_2}\to\mathbb{R}$ satisfies Assumptions \ref{assumption:Convexity}, \ref{assumption:Lipschitz Continuity} with Lipschitz constant $L>0$ and satisfies the Sharpness condition (Assumption \ref{assumption:sharpness}) over $\mathcal X_{\lambda,s}$ with parameter $\mu>0$ relative to the optimal solution set $\mathcal{X}^* := \arg\min_{X\in\mathcal{X}_{\lambda,s}} f(X)$.
    Assume that the initial point satisfies $X^{(0)}\in\mathcal X_{\lambda,s}$.
    Let $\{X^{(t)}\}_{t=0}^{T}$ be the sequence generated by the Regularized Truncated Spectral Descent with Weight Decay \eqref{eq:regularized_truncated_spectral_descent_wd} with the step size $\eta_t = \frac{2}{\lambda(t+3)}$ and the tolerance parameter $\epsilon_t = \frac{2\sqrt{s}}{\sqrt{\lambda}} \sqrt{\eta_t}$.
    Then, we have
    \begin{equation}
        \mathrm{dist}(X^{(T)}, \mathcal{X}^*) \leq \frac{2}{\mu(T+2)(T+1)}\Delta_f^{0} + \frac{32L \sqrt{2s}}{3\mu\lambda} \frac{(T+3)^{3/2}}{(T+2)(T+1)},
    \end{equation}
    where $\Delta_f^{0} = f(X^{(0)}) - f(X^*)$ is the initial optimality gap.
    \begin{proof}

    It is straightforward to verify that the corresponding diameter $D_{\mathcal{X}_{\lambda,s}}$ of $\mathcal{X}_{\lambda,s}$ is $\frac{2\sqrt{s}}{\lambda}$.
    By following the same proof strategy as Proposition \ref{prop:boundedness_curvature_matrix}, we can obtain the following bound for the curvature constant:
    \begin{equation}
        \mathcal{C}_f(\epsilon) \leq \frac{16L s}{\lambda^2\epsilon}.
    \end{equation}
    Substituting this bound into the proof of Theorem \ref{theorem:convergence_SD-WD}, we can derive the desired convergence guarantee for the Regularized Truncated Spectral Descent with Weight Decay.

    \end{proof}
\end{theorem}

\section{Robust Low-Rank Matrix Recovery}
\label{sec:application}
In this section, we apply the proposed RTSD-WD \eqref{eq:regularized_truncated_spectral_descent_wd} to the problem of robust low-rank matrix recovery under mixed noise.

Low-rank matrix recovery is a fundamental problem in modern signal processing \citep{davenport2016overview} and machine learning \citep{srebro2004maximum, li2018algorithmic}.
As a standard problem in this domain, matrix sensing aims to reconstruct an unknown low-rank matrix $X^* \in \mathbb{R}^{n_1\times n_2}$ from a limited number of linear measurements.
In this subsection, we consider the mixed-noise setting, where the observation is given by $b = \mathcal{A}(X^*) + e_1 + e_2$.
Here, $\mathcal{A}: \mathbb{R}^{n_1 \times n_2} \to \mathbb{R}^m$ denotes a linear measurement operator, $X^*$ is the ground-truth matrix with $\text{rank}(X^*) \leq r^*$, $e_1$ is a sparse noise vector with a small number of non-zero entries, and $e_2$ is a dense noise vector with small magnitude.
Throughout this section, we assume that the sparse support of $e_1$ is independent of the measurement operator $\mathcal A$.
Assuming the nuclear norm of the target matrix $X^*$ is known a priori, i.e., $\|X^*\|_* = R$, the matrix recovery task can be formulated as the following feasibility problem:
\begin{equation}
    \begin{aligned}
    \text{find} \quad & X \\
    \text{subject to} \quad & \mathcal{A}(X) +e_1 + e_2 = b, \\
    & \text{rank}(X) \leq r^*, \\
    & \|X\|_* = R.
    \end{aligned}
\end{equation}
However, directly solving the above rank-constrained problem is generally NP-hard.
Therefore, we consider the least absolute deviation (LAD) formulation for low-rank matrix recovery:
		\begin{equation}
			\label{eq:low_rank_matrix_recovery}
			\min_{X\in\R^{n_1\times n_2}:\text{rank}(X) \leq r^*, \|X\|_* = R} f(X):=\frac{1}{m}\|\mathcal{A}(X) - b\|_1.
		\end{equation}
		The LAD model is robust to outliers in the measurements, which has been shown to improve recovery performance in the presence of sparse noise \citep{dielman2005least,li2020nonconvex, xu2024convergence, huang2025adversarial}.
        To establish the recovery guarantees, we introduce the $\ell_1/\ell_2$-RIP condition.

        \begin{definition}[$\ell_1/\ell_2$-RIP]
	    \label{definition:ell_1/ell_2-RIP}
	    We say an operator $\A:\R^{n_1\times n_2}\to\R^m$ satisfies the $\ell_1/\ell_2$-RIP with constants $\delta_{r}\in(0,\sqrt{\frac{2}{\pi}})$,
	    if for any matrix $X\in\R^{n_1\times n_2}$ with rank at most $r$, we have
	    $$(\sqrt{\frac{2}{\pi}}-\delta_{r})\|X\|_F \leq \frac{1}{m}\|\A(X)\|_1 \leq (\sqrt{\frac{2}{\pi}}+\delta_{r})\|X\|_F.$$
        \end{definition}

        This is a standard assumption in the robust low-rank matrix recovery literature \citep{li2020nonconvex,xu2024convergence}
        If the linear operator $\mathcal{A}$ consists of random matrices $\{A_i\}_{i=1}^m$ with i.i.d. $\mathcal{N}(0,1)$ entries, it satisfies the $\ell_1/\ell_2$-RIP condition with high probability, provided that the number of measurements scales as $m = \mathcal{O}(rn)$ \citep{li2020nonconvex}.

         We will show that the objective function $f(X) = \frac{1}{m}\|\mathcal{A}(X) - b\|_1$ is Lipschitz continuous under the $\ell_1/\ell_2$-RIP condition.

\begin{proposition}[Lipschitz Constant via $\ell_1/\ell_2$-RIP]
	\label{proposition:Lipschitz_rip}
	If the linear operator $\mathcal{A}$ satisfies the $\ell_1/\ell_2$-RIP with constant $\delta_{1}$, then $f(X)$ satisfies Lipschitz continuity relative to the nuclear norm with a Lipschitz constant $L_{\mathcal{A}} := (\sqrt{\frac{2}{\pi}}+\delta_{1})$.
\end{proposition}
\begin{proof}
	For any matrix $H \in \mathbb{R}^{n_1 \times n_2}$, it can be partitioned and decomposed into a sum of mutually orthogonal matrices
	\begin{equation}
		H = \sum_{i=1}^n H_i,
	\end{equation}
	where each sub-matrix $H_i$ is the best rank-$1$ approximation of $H - \sum_{j=1}^{i-1} H_j$.
	Then we obtain
\begin{equation}
    \label{eq:rip_triangle_bound}
     \frac{1}{m}\|\mathcal{A}(H)\|_1 =  \frac{1}{m}\left\|\mathcal{A}\left(\sum_{i=1}^n H_i\right)\right\|_1 \leq \sum_{i=1}^n  \frac{1}{m}\|\mathcal{A}(H_i)\|_1 \leq (\sqrt{\frac{2}{\pi}}+\delta_{1})\sum_{i=1}^n \|H_i\|_F,
\end{equation}
where the second inequality follows from the $\ell_1/\ell_2$-RIP condition.

By the definition of the nuclear norm, we have
\begin{equation}
    \label{eq:sum_F_bound}
    \|H\|_* = \sum_{i=1}^n \|H_i\|_F.
\end{equation}

Substituting equation \eqref{eq:sum_F_bound} into \eqref{eq:rip_triangle_bound} yields the desired global upper bound:
\begin{equation}
     \frac{1}{m}\|\mathcal{A}(H)\|_1 \leq (\sqrt{\frac{2}{\pi}}+\delta_{1})\|H\|_*.
\end{equation}
This completes the proof.
\end{proof}

To prove $f(X)$ satisfies the restricted sharpness property, we first show that the error matrix $H^{(t)} = X^{(t)} - X^*$ belongs to a restricted error set $\mathcal{C}$.

    \begin{lemma}[Restricted Error Set]
	\label{lemma:restricted_error_set_low_rank_matrix_recovery}
	Let $X\in\mathcal{G}:= \{X\in\R^{n_1\times n_2}:\|X\|_* \leq R\}$.
	Define the error matrix $H = X - X^*$.
	Then $H$ belongs to the restricted error set
    \begin{equation}
        \mathcal{C}=\{H\in\R^{n_1\times n_2}:\|H_{T_{X^*}^\perp}\|_* \leq \|H_{T_{X^*}}\|_*\}.
    \end{equation}
	\end{lemma}
	\begin{proof}
	Let $X^* = U\Sigma V^\top$ denote the SVD of $X^*$.
	Since $\|X\|_* \leq R$, it implies $\|X^*+H\|_* \leq \|X^*\|_*$.
	Combining this with Lemma \ref{lemma:nuclear_norm_property}, we have
	\begin{equation}
		\label{eq:restricted_error_set}
		\|H_{T_{X^*}^\perp}\|_* \leq \|H_{T_{X^*}}\|_*.
	\end{equation}
	\end{proof}

    Let $\mathcal{S}$ be the support set of the sparse noise $e_1$, i.e., $\mathcal{S} = \{i\in[m] : e_{1,i} \neq 0\}$.
    Furthermore, let $p$ denote the fraction of the sparse noise in the measurements $b$, i.e., $p = \frac{|\mathcal{S}|}{m}$.
    Given prior knowledge of $\mathcal{S}$, we introduce the restricted outlier bound condition \citep{charisopoulos2021low, xu2024convergence, huang2025adversarial} to deal with the influence of the sparse noise on the recovery performance.

    \begin{definition}[Restricted outlier bound condition]
        \label{definition:restricted_outlier_bound}
         We say an operator $\A:\R^{n_1\times n_2}\to\R^m$ satisfies the restricted $(\mathcal{S}, \tau_r)$-outlier bound condition over $\mathcal{K}$, if for any matrix $X\in\mathcal{K}$, we have
        $$\tau\|X\|_F \leq \frac{1}{m}\|\A_{\mathcal{S}^c}(X)\|_1-\frac{1}{m}\|\A_{\mathcal{S}}(X)\|_1$$
        where $\mathcal{A}_{\mathcal{S}}(X) \in \mathbb{R}^{|\mathcal{S}|}$ and $\mathcal{A}_{\mathcal{S}^c}(X) \in \mathbb{R}^{|\mathcal{S}^c|}$ represent the sub-vectors formed by extracting the entries of $\mathcal{A}(X)$ indexed by $\mathcal{S}$ and its complement $\mathcal{S}^c$, respectively.
    \end{definition}

    If $\mathcal{K}$ is the manifold of rank-r matrices $\mathcal{M}_r = \{ X \in \mathbb{R}^{n_1 \times n_2} : \text{rank}(X) = r \}$, $\mathcal{A}_{\mathcal{S}}$ and $\mathcal{A}_{\mathcal{S}^c}$ satisfy the $\ell_1/\ell_2$-RIP condition with constants $\delta_{r}$ and $\delta_{r}'$, respectively, it is straightforward to verify that the restricted outlier bound condition holds with $\tau_r = (1-p)(\sqrt{2/\pi}-\delta_{r}')-(\sqrt{2/\pi}+\delta_{r})p$.
    We will show that if the $\ell_1/\ell_2$-RIP constant and the fraction of sparse noise $p$ are sufficiently small, then the restricted outlier bound condition over the restricted error set $\mathcal{C}$ holds with a positive constant $\tau > 0$.

    \begin{lemma}
    \label{lemma:restricted_outlier_bound}
    Suppose that $\mathcal{A}_{\mathcal{S}}$ and $\mathcal{A}_{\mathcal{S}^c}$ satisfy the $\ell_1/\ell_2$-RIP condition.
    If the fraction of sparse noise $p$ satisfies $p < \frac{1}{2} \left[ 1 - \sqrt{\frac{2}{3}} - \frac{\delta_{5r^*} + \sqrt{\frac{2}{3}}\delta_{5r^*}}{\sqrt{\frac{2}{\pi}}} \right]$ and $\delta_{5r^*} < 0.08$, then the restricted outlier bound condition over $\mathcal{C}=\{H\in\R^{n_1\times n_2}:\|H_{T_{X^*}^\perp}\|_* \leq \|H_{T_{X^*}}\|_*\}$ holds with $\tau= \left[ (1-2p)\sqrt{\frac{2}{\pi}} - \delta_{5r^*} - \left( \sqrt{\frac{2}{\pi}} + \delta_{3r^*} \right)\sqrt{\frac{2}{3}} \right]\sqrt{\frac{3}{5}} > 0$.
    \end{lemma}
    \begin{proof}
        Let $H\in\mathcal C$ be arbitrary.
        We partition the error matrix $H \in \mathcal{C}$ into orthogonal components $H = B_0 + B_1 + \sum^J_{i \geq 2} B_i$.
        Here $B_0 = H_{T_{X^*}}$ and $\{B_i\}_{i=1}^J$ is the best rank-$3r^*$ approximation of $H_{T_{X^*}^\perp} - \sum_{j=1}^{i-1} B_j$.

Applying the $\ell_1/\ell_2$-RIP condition and the triangle inequality, we deduce
\begin{equation}
	\label{eq:lower_bound_AScH}
    \begin{aligned}
        \frac{1}{m}\|\mathcal{A}_{S^c}(H)\|_1 &\geq \frac{1}{m}\|\mathcal{A}_{S^c}(B_0+B_1)\|_1 - \sum_{i=2}^{J}\frac{1}{m}\|\mathcal{A}_{S^c}(B_i)\|_1 \\
        &\geq (1-p)\left[(\sqrt{\frac{2}{\pi}}-\delta_{5r^*})\|B_0+B_1\|_F - (\sqrt{\frac{2}{\pi}}+\delta_{3r^*})\sum_{i=2}^{J}\|B_i\|_F \right] \\
        &\geq (1-p)\left[(\sqrt{\frac{2}{\pi}}-\delta_{5r^*})\|B_0+B_1\|_F - (\sqrt{\frac{2}{\pi}}+\delta_{3r^*})\sqrt{\frac{2}{3}}\|B_0\|_F \right] \\
        &\geq (1-p)\left(\sqrt{\frac{2}{\pi}}-\delta_{5r^*} - (\sqrt{\frac{2}{\pi}}+\delta_{3r^*})\sqrt{\frac{2}{3}}\right)\|B_0+B_1\|_F,
    \end{aligned}
\end{equation}
where the third inequality follows from $\sum_{i=2}^{J}\|B_i\|_F \leq \frac{1}{\sqrt{3r^*}}\|B_0\|_* \leq \sqrt{\frac{2}{3}}\|B_0\|_F$ and the last inequality holds by substituting $\|B_0\|_F \leq \|B_0 + B_1\|_F$.
Similarly, we can derive an upper bound for $\frac{1}{m}\|\mathcal{A}_{S}(H)\|_1$ as follows:
\begin{equation}
	\label{eq:upper_bound_ASH}
        \frac{1}{m}\|\mathcal{A}_{S}(H)\|_1 \leq p\left(\sqrt{\frac{2}{\pi}}+\delta_{5r^*} + (\sqrt{\frac{2}{\pi}}+\delta_{3r^*})\sqrt{\frac{2}{3}}\right)\|B_0+B_1\|_F.
\end{equation}
Therefore, we have
\begin{equation}
    \begin{aligned}
        &\quad\frac{1}{m}\|\mathcal{A}_{S^c}(H)\|_1 - \frac{1}{m}\|\mathcal{A}_{S}(H)\|_1 \\
        &\geq \left[(1-2p)\sqrt{\frac{2}{\pi}}-\delta_{5r^*} - (\sqrt{\frac{2}{\pi}}+\delta_{3r^*})\sqrt{\frac{2}{3}}\right]\|B_0+B_1\|_F\\
        &\geq \left[(1-2p)\sqrt{\frac{2}{\pi}}-\delta_{5r^*} - (\sqrt{\frac{2}{\pi}}+\delta_{3r^*})\sqrt{\frac{2}{3}}\right]\sqrt{\frac{3}{5}}\|H\|_F,
\end{aligned}
\end{equation}
where the last inequality follows from $\|H\|_F \leq \sqrt{\frac{5}{3}}\|B_0+ B_1\|_F$.

Since $H\in\mathcal C$ was arbitrary, this proves that $\mathcal A$ satisfies
the restricted $(S,\tau)$-outlier bound condition over $\mathcal C$.
The assumed upper bound on $p$ guarantees $\tau>0$, and the proof is complete.
\end{proof}
    It is worth noting that if the outlier fraction $p$ is an absolute constant, the linear operators $\mathcal{A}_{\mathcal{S}^c}$ and $\mathcal{A}_{\mathcal{S}}$ generated from the standard Gaussian ensemble satisfy the $\ell_1/\ell_2$-RIP condition with constant $\delta_r$ with high probability, provided that the number of measurements scales as $m = \mathcal{O}(rn)$.
    Now, we can establish the restricted sharpness property of $f(X)$ under the $\ell_1/\ell_2$-RIP condition and the restricted outlier bound condition.
	\begin{proposition}[Restricted Sharpness via $\ell_1/\ell_2$-RIP]
    \label{proposition:restricted_sharpness_rip}
	Assume the restricted outlier bound condition in Lemma \ref{lemma:restricted_outlier_bound} holds.
    Then, $f(X)$ satisfies the $\mu$-restricted sharpness property over $\mathcal{G}= \{X\in\R^{n_1\times n_2}:\|X\|_* \leq R\}$ with respect to $X^*$, i.e.,
    \begin{equation}
        f(X) - f(X^*) \geq \mu\|X-X^*\|_F - \frac{2}{m}\|e_2\|_1, \quad \forall X \in \mathcal{G},
    \end{equation}
    where $\mu= \left[ (1-2p)\sqrt{\frac{2}{\pi}} - \delta_{5r^*} - \left( \sqrt{\frac{2}{\pi}} + \delta_{3r^*} \right)\sqrt{\frac{2}{3}} \right]\sqrt{\frac{3}{5}} > 0$.

	\end{proposition}
    We omit the proof here and refer readers to \citet[Lemma~4]{xu2024convergence} for a detailed derivation.

For computational efficiency, we choose the structured subdifferential set $T(X, \epsilon) := \text{conv}\left(\bigcup_{Y \in N(X, \epsilon)} \partial f(Y)\right)$, where $N(X, \epsilon) = \{Y \in \R^{n_1 \times n_2} : \|\mathcal{A}(Y) - \mathcal{A}(X)\|_\infty \leq \epsilon\}$ is the unconstrained $L_\infty$-neighborhood around $X$.
By employing Regularized Truncated Spectral Descent with Weight Decay \eqref{eq:regularized_truncated_spectral_descent_wd}, we theoretically guarantee convergence to the ground-truth low-rank matrix $X^*$ at a sublinear rate of $\mathcal{O}\left(1/\sqrt{T}\right)$, up to the bounded noise level.
Note that the RTSD-WD \eqref{eq:regularized_truncated_spectral_descent_wd} update is analyzed over the convex nuclear-norm ball
$$
\mathcal{G}:=\{X\in\mathbb R^{n_1\times n_2}:\|X\|_*\le R\}.
$$
Since the ground-truth matrix satisfies $\|X^*\|_*=R$, we have
$X^*\in\mathcal{G}$. 
Therefore, an optimization guarantee over $\mathcal{G}$, combined with the restricted sharpness property, directly yields
a recovery guarantee for $X^*$.
\begin{theorem}[Convergence of RTSD-WD for Robust Low-Rank Matrix Recovery]
	\label{theorem:convergence_low_rank_rtsd}
    Consider the nuclear-norm-ball relaxation of the LAD formulation \eqref{eq:low_rank_matrix_recovery}:
\begin{equation}
    \min_{X\in\mathcal G} f(X):=\frac1m\|\mathcal A(X)-b\|_1.
\end{equation}
    Suppose the linear operator $\mathcal{A}: \mathbb{R}^{n_1 \times n_2} \to \mathbb{R}^m$ consists of random matrices $\{A_i\}_{i=1}^m$ with i.i.d. $\mathcal{N}(0,1)$ entries.
    Define the neighborhood of $X$ as $N(X, \epsilon) = \{Y \in \R^{n_1 \times n_2} : \|\mathcal{A}(Y) - \mathcal{A}(X)\|_\infty \leq \epsilon\}$.
    Assume that the initial point satisfies $X^{(0)}\in\mathcal G$.
    Let $\{X^{t}\}_{t=0}^{T}$ be the sequence generated by RTSD-WD \eqref{eq:regularized_truncated_spectral_descent_wd} with $\lambda=\frac{1}{R}$, $\eta_t = \frac{2}{\lambda(t+3)}$, $\epsilon_t = \frac{2L_{\mathcal{A}}}{\lambda}\sqrt{\frac{2m}{t+3}}$ and $s=1$.

    Suppose $m \gtrsim nr^*$.
    If $\delta_{5r^*} < 0.08$ and the fraction of outliers $p$ satisfies $p < \frac{1}{2} [ 1 - \sqrt{2/3} - (\delta_{5r^*} + \sqrt{2/3}\delta_{3r^*})/\sqrt{2/\pi} ]$, then there exist absolute constants $C_1, C_2, C_3 > 0$ such that, with probability at least $1 - C_1\exp(-C_2\delta^2_{5r^*} m)$, $X^{t}$ converges to the target matrix $X^*$ at a sublinear rate:
	\begin{equation}
		\|X^{T} - X^*\|_F \leq C_3 \frac{L_{\mathcal{A}}\sqrt{m}R}{\mu \sqrt{T}} + \frac{\xi}{\mu}.
	\end{equation}
	where $\mu = \left[ (1-2p)\sqrt{\frac{2}{\pi}} - \delta_{5r^*} - \left( \sqrt{\frac{2}{\pi}} + \delta_{3r^*} \right)\sqrt{\frac{2}{3}} \right]\sqrt{\frac{3}{5}}>0$, $L_{\mathcal{A}} = (\sqrt{\frac{2}{\pi}}+\delta_{1})$ and $\xi = \frac{2}{m}\|e_2\|_1$ is the dense noise level.
\end{theorem}

\begin{figure}[H]
    \centering
    \springersetfigwidth
    \includegraphics[width=\springerfigwidth,height=\springerfigmaxheight,keepaspectratio]{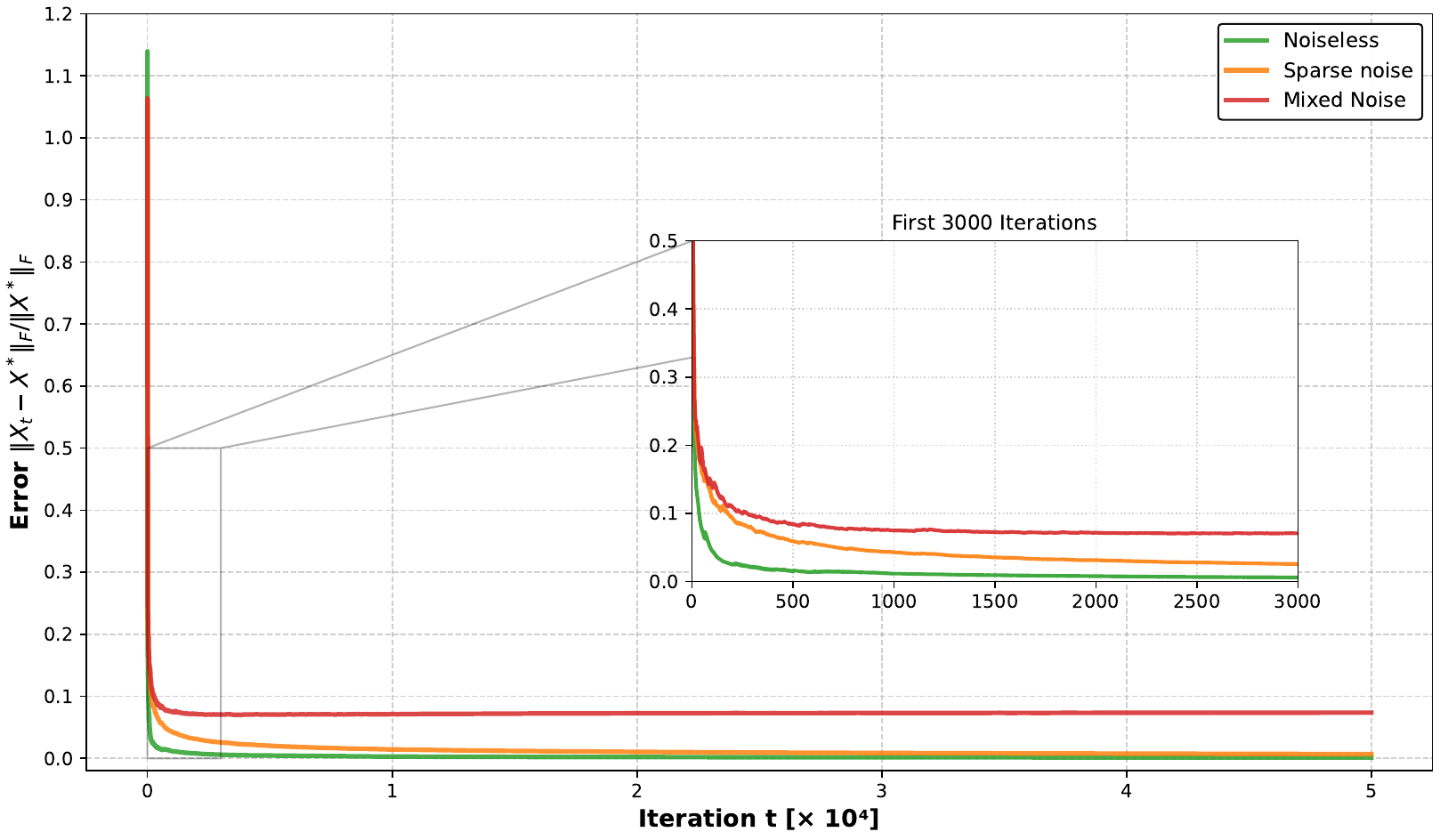}
    \caption{Convergence of RTSD-WD for low-rank matrix recovery under different noise conditions.
    The inset highlights the rapid initial descent during the first 3000 iterations.}
    \label{fig:Convergence of RTSD-WD for LRMR under different noise conditions}
\end{figure}

\begin{proof}

    First, by \citet[Proposition~1]{li2020nonconvex}, the linear operators $\mathcal{A}$, $\mathcal{A}_{\mathcal{S}}$ and $\mathcal{A}_{\mathcal{S}^c}$ generated from the standard Gaussian ensemble satisfy the $\ell_1/\ell_2$-RIP condition with constant $\delta_{5r^*}$, with high probability, provided that the number of measurements scales as $m = \mathcal{O}(r^*n)$.
    Second, by Proposition \ref{proposition:Lipschitz_rip}, the objective $f(X)$ is Lipschitz continuous with constant $L_{\mathcal{A}} = \sqrt{\frac{2}{\pi}} + \delta_1$.
    Furthermore, by Proposition \ref{prop:boundedness_curvature_matrix_LAD}, the generalized curvature constant associated with $T(X,\epsilon)$ is bounded by
    $$ \mathcal{C}^T_f(\epsilon_t) \le \frac{16 m L_{\mathcal{A}}^2 R^2}{\epsilon_t} $$

    Substituting this bound into the general recurrence relation established in the proof of Theorem \ref{theorem:convergence_RTSD-WD}, and applying Lemma \ref{lemma:surrogate-subgradient-inequality} and Proposition \ref{proposition:restricted_sharpness_rip}, we obtain
    $$ ||X^T - X^*||_F \le \frac{1}{\mu}(f(X^T) - f(X^*))+  \frac{\xi}{\mu} \le \mathcal{O}\left(\frac{L_{\mathcal{A}} \sqrt{m} R}{\mu \sqrt{T}}\right) + \frac{\xi}{\mu} $$
    This completes the proof.
\end{proof}

\begin{figure}[H]
    \centering
    \springersetfigwidth
    \includegraphics[width=\springerfigwidth,height=\springerfigmaxheight,keepaspectratio]{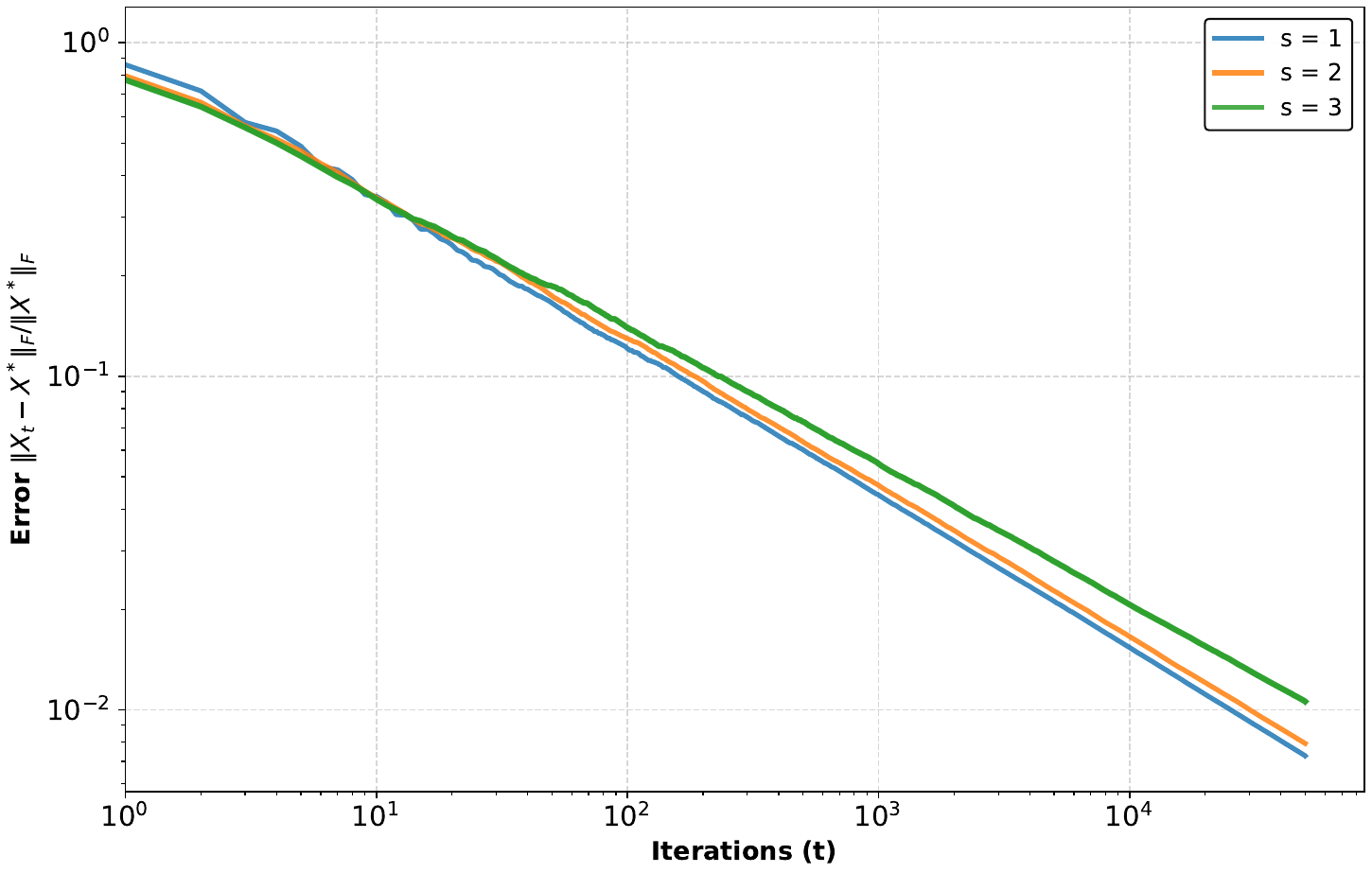}
    \caption{Convergence of RTSD-WD for low-rank matrix recovery under different truncation parameters $s$.}
    \label{fig:Convergence of RTSD-WD for LRMR under different truncation parameters s}
\end{figure}

\begin{remark}[Avoiding dimensional penalty via nuclear norm]
    We emphasize that the Lipschitz constant $L_{\mathcal{A}}$ in Theorem \ref{theorem:convergence_low_rank_rtsd} is established with respect to the \textit{nuclear norm}.
    This is dictated by the $\ell_1/\ell_2$-RIP condition to avoid the dimension-dependent penalty $\sqrt{n}$ that inevitably arises under the Frobenius norm for full-rank matrices.
    By contrast, \citet{xu2024convergence} establishes a dimension-free convergence guarantee with a Lipschitz constant of $L_{\mathcal{A}}=\sqrt{\frac{2}{\pi}}+\delta_{3r^*}$ relative to the Frobenius norm.
    This is possible because their algorithm employs iterative hard thresholding (IHT), which explicitly constrains the error residuals to be low-rank matrices (rank at most $3r^*$).
    The cost, however, is the requirement of a computationally expensive SVD at every iteration.
\end{remark}

\begin{remark}[Computational efficiency of $s=1$]
In Theorem \ref{theorem:convergence_low_rank_rtsd}, we choose the truncation parameter to $s=1$.
This is not merely a theoretical simplification but a crucial computational advantage.
For $s=1$, the update only requires the top singular vectors.
This can be quickly computed via the power method or other numerical methods, avoiding the $\mathcal{O}(n^3)$ of a full SVD.
\end{remark}

\begin{remark}[Accelerating iteration complexity via restarts]
    Although RTSD-WD converges at a sublinear rate of $\mathcal{O}(1/\sqrt{T})$, its overall iteration complexity can be significantly improved by exploiting a restart scheme \citep{roulet2017sharpness}.
    By executing the algorithm for a fixed number of inner iterations and utilizing the final iterate to initialize the subsequent stage over a strictly smaller, localized search space, the restarted scheme drastically reduces the total iteration complexity to $\mathcal{O}(\log(1/\epsilon))$ to achieve an $\epsilon$-accuracy.
    However, explicitly operating over these rigidly constrained search spaces at each stage is computationally intractable.
    Consequently, designing a practically efficient restart scheme that circumvents this computational bottleneck remains an interesting open challenge.
\end{remark}

\section{Numerical Experiments}
\label{section:experiments}
In this section, we will provide some numerical experiments to demonstrate the empirical performance of Muon-type optimizers and validate our theoretical findings.
All experiments were conducted on a single NVIDIA A100 GPU (40GB).

\paragraph{ReLU Neural Networks.}
We conduct experiments to demonstrate the empirical superiority of Muon and Spectral Descent in training non-smooth ReLU neural networks, with the results illustrated in Figure~\ref{fig:nn_experiments}.
Throughout the experiments, we initialize the weights of the layers with a uniform distribution $\mathcal{U}(-\frac{1}{\sqrt{k}}, \frac{1}{\sqrt{k}})$, where $k$ is the number of input features.

For the evaluations in Figure~\ref{fig:mnist_constant} and \ref{fig:mnist_decay}, we randomly select 150 training samples from the MNIST dataset and train a two-layer fully connected ReLU network parameterized by weight matrices $W_1\in\mathbb{R}^{512\times 784}$ and $W_2\in\mathbb{R}^{10\times 512}$ with a cross-entropy loss.
In the constant step size setting (Figure~\ref{fig:mnist_constant}), the learning rates for Muon, Spectral Descent, and Gradient Descent are tuned from the set $\{0.4, 0.5, 0.6, 0.7, 0.8, 0.9, 1.0\}$, while Adam's step size is tuned from $\{5, 6, 7, 8, 9\} \times 10^{-3}$.
In the geometric decay setting (Figure~\ref{fig:mnist_decay}), the initial step sizes for Muon, Spectral Descent, and Gradient Descent are tuned from $\{0.5, 0.6, 0.7, 0.8, 0.9, 1.0\}$, and for Adam from $\{3, 5, 7, 9, 10\} \times 10^{-3}$.
The decay factor for all optimizers is selected from $\{0.9, 0.95, 0.97, 0.99\}$.
To ensure reliability, all reported results are averaged over five independent runs with different random seeds.

For the evaluations in Figure~\ref{fig:cifar_constant} and \ref{fig:cifar_decay}, we randomly select 200 training samples from the CIFAR-10 dataset and adjust the first weight matrix to $W_1\in\mathbb{R}^{512\times 3072}$ to accommodate the input dimensionality.
In the constant step size setting (Figure~\ref{fig:cifar_constant}), we conduct a grid search for the learning rates of Muon, Spectral Descent, and Gradient Descent over $\{10^{-1}, 5 \times 10^{-2}, 10^{-2}, 5 \times 10^{-3}, 10^{-3}, 5 \times 10^{-4}, 10^{-4}\}$, whereas Adam's step size is tuned across $\{10^{-1}, 5 \times 10^{-2}, 10^{-2}, 5 \times 10^{-3}, 5 \times 10^{-4}, 10^{-3}, 10^{-4}\}$.
For the geometric decay schedule (Figure~\ref{fig:cifar_decay}), the initial step sizes for Muon, Spectral Descent, and Gradient Descent are selected from $\{0.15, 0.2, 0.3, 0.4, 0.5\}$, while the initial step size for Adam is tuned from $\{2, 4, 5, 6, 8\} \times 10^{-4}$.
The decay factor remains identical to those used in the MNIST tasks.
To ensure reliability, all reported results are averaged over five independent runs with different random seeds.

Our experimental results consistently demonstrate that Muon and Spectral Descent achieve significantly accelerated convergence compared to Adam and Gradient Descent, while maintaining remarkable robustness to learning rate schedules across both the MNIST and CIFAR-10 datasets.
Establishing a rigorous theoretical framework for these spectral descent methods within the non-convex and non-smooth landscapes of deep neural networks represents a pivotal direction for future research.

\paragraph{Low-rank Matrix Recovery.}
To empirically validate the theoretical guarantees of the Regularized Truncated Spectral Descent with Weight Decay (RTSD-WD) algorithm \eqref{eq:regularized_truncated_spectral_descent_wd}, we evaluate its performance on the robust low-rank matrix recovery task under various noise regimes and truncation settings.

We consider a matrix sensing problem with dimensions $n_1 = n_2 = 50$ and a target rank $r^* = 3$.
The number of linear measurements is set to $m = 10 n r^*$, where the sensing operator $\mathcal{A}$ is constructed from a standard Gaussian ensemble.
The ground-truth low-rank matrix $X^* \in \mathbb{R}^{50 \times 50}$ is generated as $X^* = UV^\top$, with the factor matrices $U, V \in \mathbb{R}^{50 \times 3}$ sampled with i.i.d. standard Gaussian entries.
The observations are formulated as $b = \mathcal{A}(X^*) + e_1 + e_2$, where $e_1$ and $e_2$ represent sparse outliers and dense magnitude noise, respectively.
Following the framework of RTSD-WD \eqref{eq:regularized_truncated_spectral_descent_wd}, we optimize the objective function $f(X):=\frac{1}{m}\|\mathcal{A}(X) - b\|_1$ using the truncation parameter $s=1$, the regularization parameter $\lambda = \frac{1}{\|X^*\|_*}$, the step size $\eta_t = \frac{2}{\lambda(t+3)}$ and the tolerance $\epsilon_t = \frac{0.08}{\lambda}\sqrt{\frac{m}{\pi(t+3)}}$, respectively.
To facilitate efficient computation, we replace the subdifferential set $T$ with a surrogate set $H$.
Both sets have the same generalized curvature constant $\mathcal{C}_f(\epsilon)$ and satisfy the same approximate subgradient inequality required by the convergence analysis of RTSD-WD, see Appendix \ref{app:auxiliary_lemmas_robust_low_rank_matrix_recovery}.
Numerically, for a residual vector $z = \mathcal{A}(X) - b$, the dual variable $v$ is determined component-wise:
$v_i = \text{sgn}(z_i)/m$ for $|z_i| > \epsilon_t$, and $v_i = z_i / (\epsilon_t m)$ otherwise.
This ensures $v_i \in [-1/m, 1/m]$ as required by the theory.
The gradient direction is then efficiently computed via the adjoint operator $G = \mathcal{A}^*(v)$.
The following results should be interpreted as empirical evidence for a practical surrogate implementation inspired by the theory.
For all trials, the algorithm is initialized with a standard Gaussian matrix scaled by $10^{-4}$ to ensure the initial iterate $X^{(0)}$ resides within the feasible nuclear norm ball $\mathcal{X}$.

Figure~\ref{fig:Convergence of RTSD-WD for LRMR under different noise conditions} illustrates the convergence trajectories of RTSD-WD across three distinct noise scenarios: (i) a noiseless baseline ($e_1 = \mathbf{0}, e_2 = \mathbf{0}$); (ii) a sparse noise setting where fraction $p=6\%$ of the entries in $e_1$ are corrupted by noise drawn from $\mathcal{N}(0, 100)$; and (iii) a mixed noise regime combining the aforementioned sparse outliers with dense noise $e_2 \sim \mathcal{N}(0, 1)$.
As depicted in the results, RTSD-WD demonstrates a remarkably rapid initial descent during the first 3,000 iterations across all conditions.
In the noiseless and sparse noise cases, the algorithm achieves steady convergence toward the optimal solution.
In the mixed noise setting, the iterates successfully converge to a stable neighborhood of $X^*$, with the residual error floor governed by the dense noise level $\xi$.
These findings align with Theorem \ref{theorem:convergence_low_rank_rtsd}, confirming the algorithm's robustness to outlier measurements.

We further investigate the sensitivity of the algorithm to the truncation parameter $s \in \{1, 2, 3\}$ under the sparse noise setting (where $p=6\%$).
As illustrated in Figure~\ref{fig:Convergence of RTSD-WD for LRMR under different truncation parameters s}, the relative error trajectories for all tested values of $s$ exhibit a clear sublinear convergence rate, which is consistent with the $O(1/\sqrt{T})$ guarantee in Theorem~\ref{theorem:convergence_low_rank_rtsd}.
This confirms that the convergence behavior remains remarkably robust to the choice of the truncation parameter, even in the presence of outliers.
Notably, the configuration with $s=1$ not only achieves a marginally lower relative error but also offers a significant computational advantage, which only requires the computation of the top-1 singular vector at each iteration.
To ensure the statistical reliability of these observations, all reported results are averaged over five independent runs with different random seeds.

\paragraph{Linear Programming.}
To empirically validate the optimization dynamics of Spectral Descent (SD) and Truncated Spectral Descent (TSD), we evaluate them on a synthetic linear programming (LP) problem.
Consider a general LP formulated as:
\begin{equation}
    \min_{X\in \mathcal{X} \subseteq \mathbb{R}^{n_1 \times n_2}} \quad f(X)
\end{equation}
where $f(X) = \langle C, X \rangle$ is a linear function and $\mathcal{X}$ is a polyhedral set. It is well-established that linear programming satisfies the sharpness property with respect to its optimal solution set $\mathcal{X}^*$, as shown by \citet[Theorem~11]{burke1993weak}.
In our experiments, we consider the following specific LP instance:
\begin{equation}
    \label{eq:linear_programming}
\begin{aligned}
&\min_{W, \xi} \quad  \sum_{i=1}^{N} \xi_i \\
\text{s.t.} \quad & \xi_i \ge y_i - \langle X_i, W \rangle, \quad \forall i\in [N], \\
& \xi_i \ge -(y_i - \langle X_i, W \rangle), \quad \forall i\in [N].
\end{aligned}
\end{equation}

By eliminating the slack variables $\xi_i$, the constrained LP problem \eqref{eq:linear_programming} seamlessly reduces to the following unconstrained empirical risk minimization (ERM) problem with a Least Absolute Deviation (LAD) objective:
\begin{equation}
    \min_{W\in\mathbb{R}^{n_1\times n_2} } \quad f(W) := \sum_{i=1}^{N} | y_i - \langle X_i, W \rangle |.
\end{equation}
It is straightforward to verify that this objective function is a proper polyhedral convex function. If the optimal solution set $\mathcal{W}^*$ is non-empty, there exists a constant $\mu > 0$ such that $f$ is $\mu$-sharp with respect to $\mathcal{W}^*$, i.e.,
\begin{equation}
    \nonumber
    f(W) - f(W^*) \geq \mu \cdot \text{dist}(W, \mathcal{W}^*), \quad \forall W \in \mathbb{R}^{n_1 \times n_2}.
\end{equation}
Consequently, optimizing $f(W)$ serves as an ideal testbed to investigate the convergence behaviors of SD and TSD under the combined theoretical conditions of convexity, Lipschitz continuity, and sharpness.

For the empirical setup, we set the sample size to $N=2000$ and the matrix dimensions to $n_1 = n_2 = 10$.
Both the input data matrices $X_i$ and the ground-truth weight matrix $W^*$ are drawn with i.i.d. standard Gaussian entries $\mathcal{N}(0, 1)$.
The target labels are generated via a noiseless linear mapping $y_i = \langle X_i, W^* \rangle$.
We initialize the weight matrix $W^{(0)}$ using a standard Gaussian distribution.

We now analyze the specific convergence behaviors depicted in Figure~\ref{fig:12_grid_comparison}. In this evaluation, we consider three distinct algorithmic configurations: standard Spectral Descent (SD) alongside Truncated Spectral Descent (TSD) with truncation parameters $s=3$ and $s=1$, which correspond to the top, middle, and bottom rows, respectively.
Regarding the hyperparameter settings, all algorithms employ a geometric decay step size schedule. The initial step sizes $\eta_0$ for SD and TSD ($s=1$) are tuned from the set $\{0.8, 0.85, 0.9, 1.0, 1.1\}$, while $\eta_0$ of TSD ($s=3$) is selected from $\{1, 1.2, 1.4, 1.5, 2\}$
Their decay factors are selected from the range $\{0.89, 0.91, 0.93, 0.96, 0.97, 0.99\}$.
The first column of Figure~\ref{fig:12_grid_comparison} illustrates the convergence trajectories under various step size configurations.
For the analysis in the subsequent three columns, we isolate and report the results using the optimal parameter settings that yielded the fastest convergence in the first column.
\begin{itemize}
    \item Observing the first column, it is evident that the objective loss decreases rapidly across all three configurations.
This rapid descent empirically validates that satisfying the conditions of convexity, Lipschitz continuity, and sharpness provides a highly favorable landscape for spectral updates.
    \item In the second column, we monitor the evolution of the subgradient norm $\|G_t\|_F$.
Because the objective function $f$ is convex and satisfies the sharpness condition, we can leverage this subgradient norm to empirically estimate the condition parameter $\kappa=\frac{\mu}{L}$.
The $\kappa$ calculated through this empirical observation effectively serves as an upper bound for the true condition parameter of the underlying problem.
    \item It is crucial to highlight that this specific linear programming instance does not strictly satisfy the stringent theoretical assumptions regarding the condition parameter posited in our foundational framework.
Nevertheless, the alignment metric in the third column reveals a crucial insight.
Throughout the entire optimization trajectory, the alignment remains strictly positive and is significantly larger than the empirically estimated $\kappa$. This robust behavior not only demonstrates that the algorithms can succeed beyond their strict theoretical regimes, but also indicates that our current theoretical bounds are relatively loose.
It strongly suggests that the theoretical guarantees can be further tightened by exploiting specific problem structures and leveraging advanced probabilistic tools in future work.
    \item Finally, as shown in the fourth column, this persistent high alignment naturally translates into a steady decline in the relative parameter error $\|W_t - W^*\|_F / \|W^*\|_F$, confirming the reliable recovery of the optimal solution for both standard SD and its truncated variants.
\end{itemize}

\begin{figure}[H]
    \centering
    \springersetfigwidth
    \begin{minipage}{\springerfigwidth}
    \centering

    \begin{subfigure}[b]{0.23\linewidth}
        \centering
        \includegraphics[width=\linewidth,height=\springerfigmaxheight,keepaspectratio]{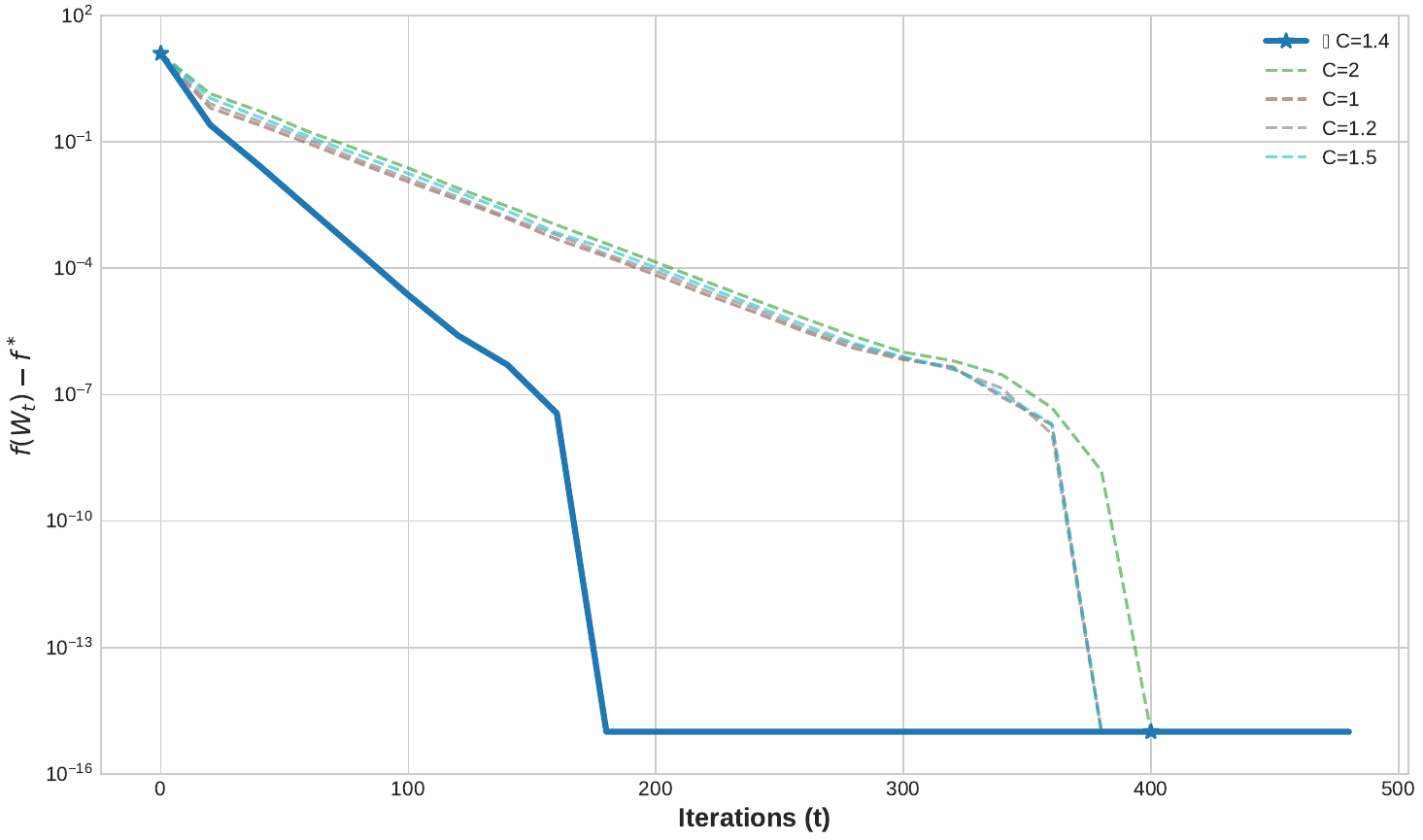}
        \caption{Loss}
        \label{fig:a1}
    \end{subfigure}
    \hfill
    \begin{subfigure}[b]{0.23\linewidth}
        \centering
        \includegraphics[width=\linewidth,height=\springerfigmaxheight,keepaspectratio]{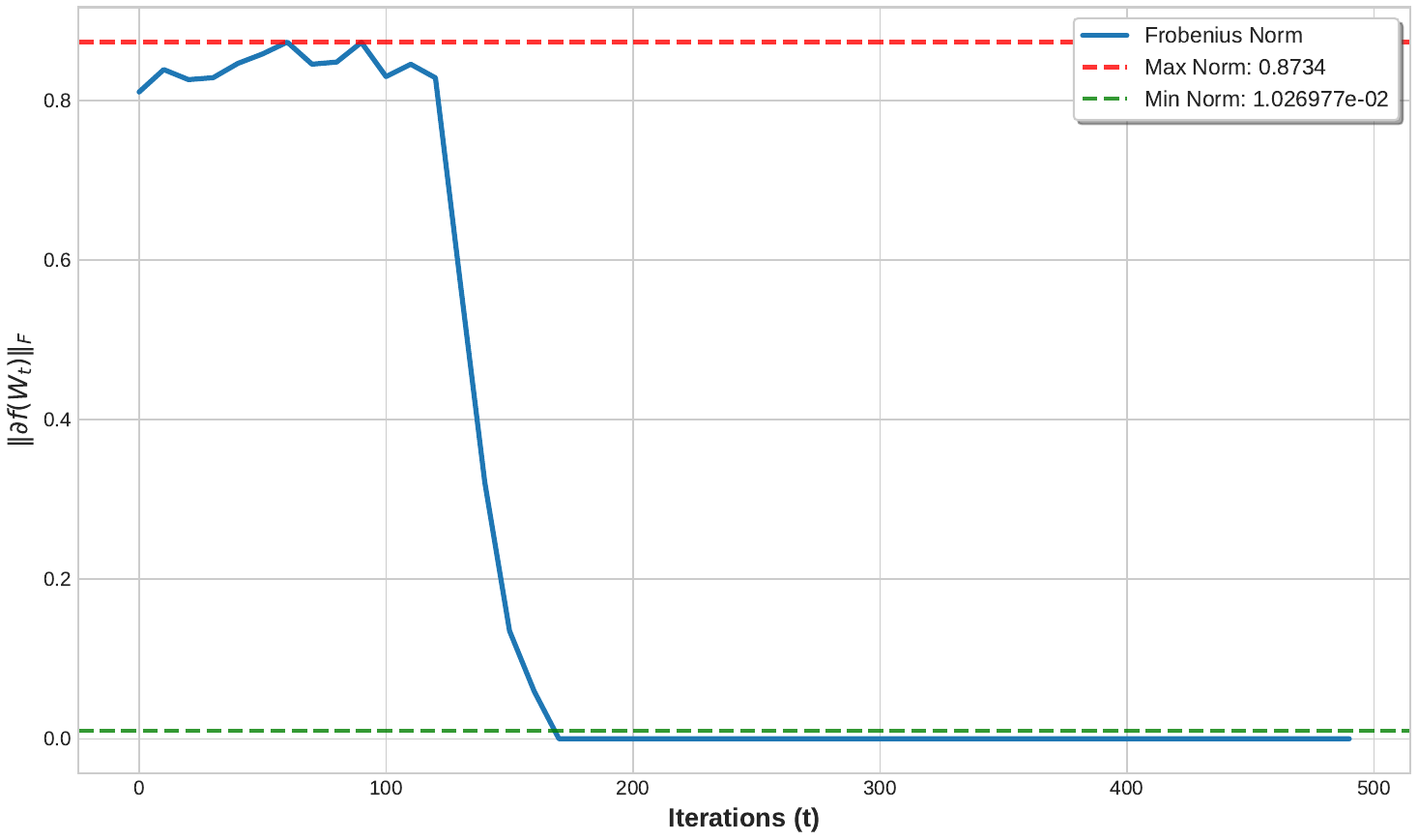}
        \caption{Gradient Norm}
        \label{fig:a2}
    \end{subfigure}
    \hfill
    \begin{subfigure}[b]{0.23\linewidth}
        \centering
        \includegraphics[width=\linewidth,height=\springerfigmaxheight,keepaspectratio]{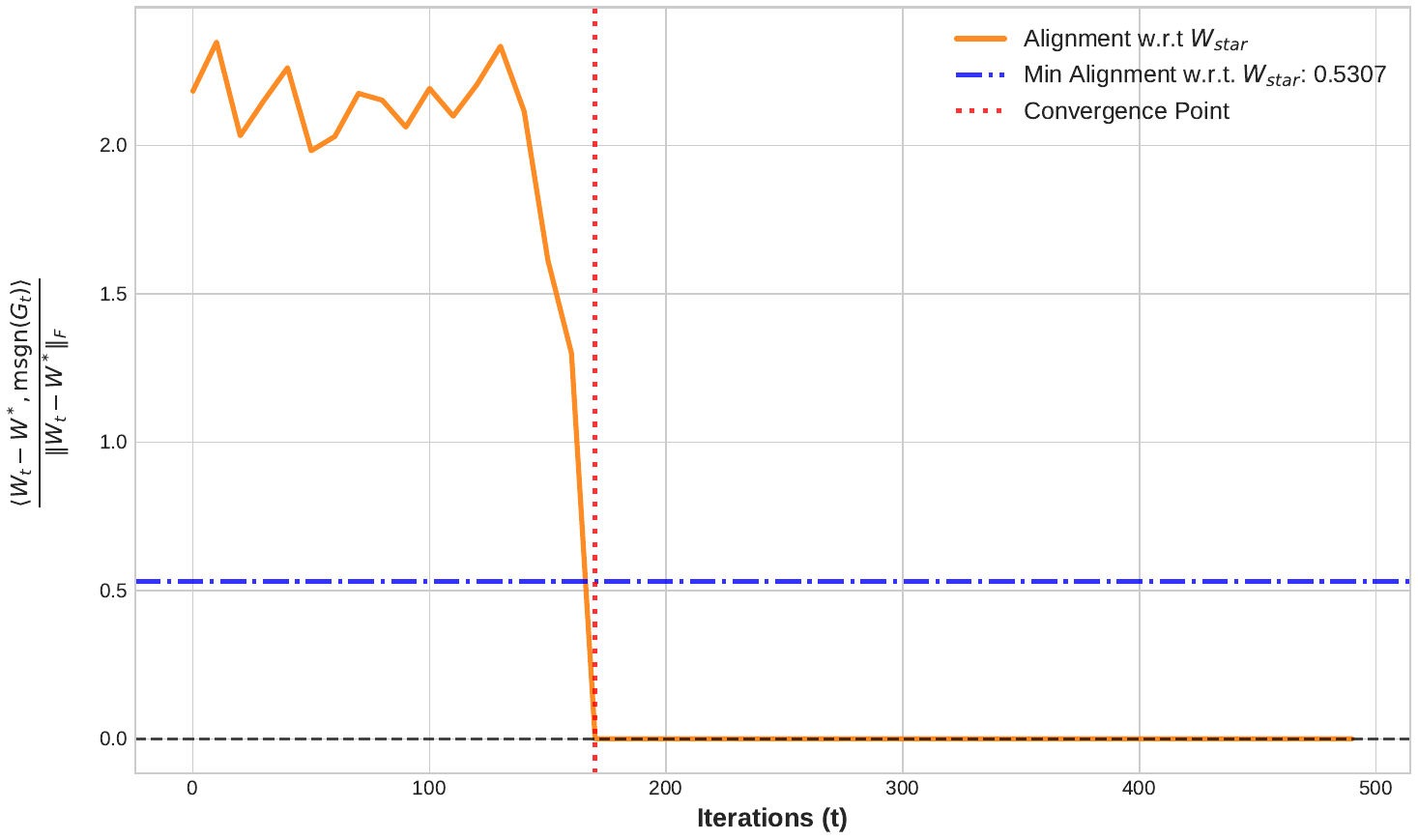}
        \caption{Alignment}
        \label{fig:a3}
    \end{subfigure}
    \hfill
    \begin{subfigure}[b]{0.23\linewidth}
        \centering
        \includegraphics[width=\linewidth,height=\springerfigmaxheight,keepaspectratio]{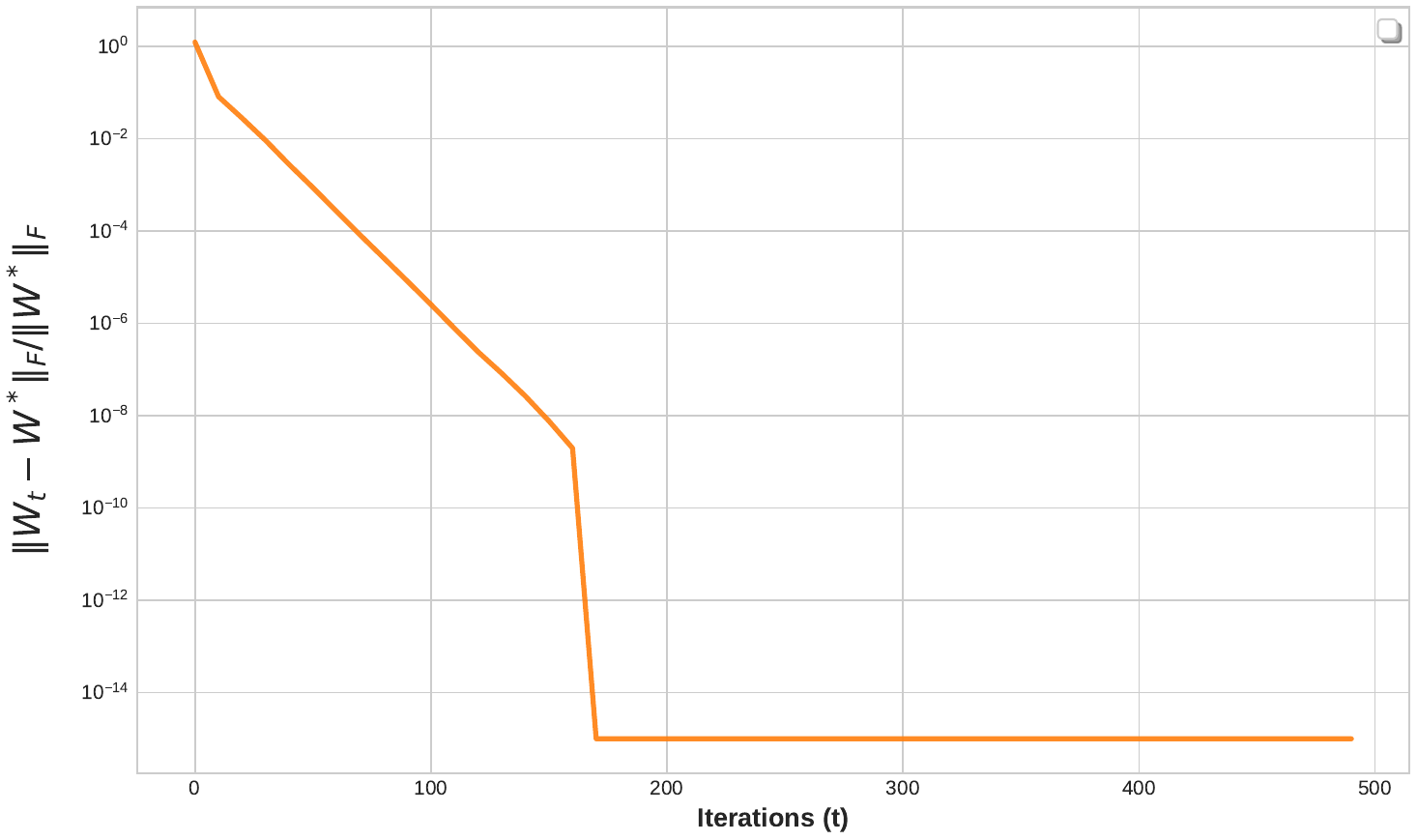}
        \caption{Relative error}
        \label{fig:a4}
    \end{subfigure}

    \vspace{0.05cm}

    \begin{subfigure}[b]{0.23\linewidth}
        \centering
        \includegraphics[width=\linewidth,height=\springerfigmaxheight,keepaspectratio]{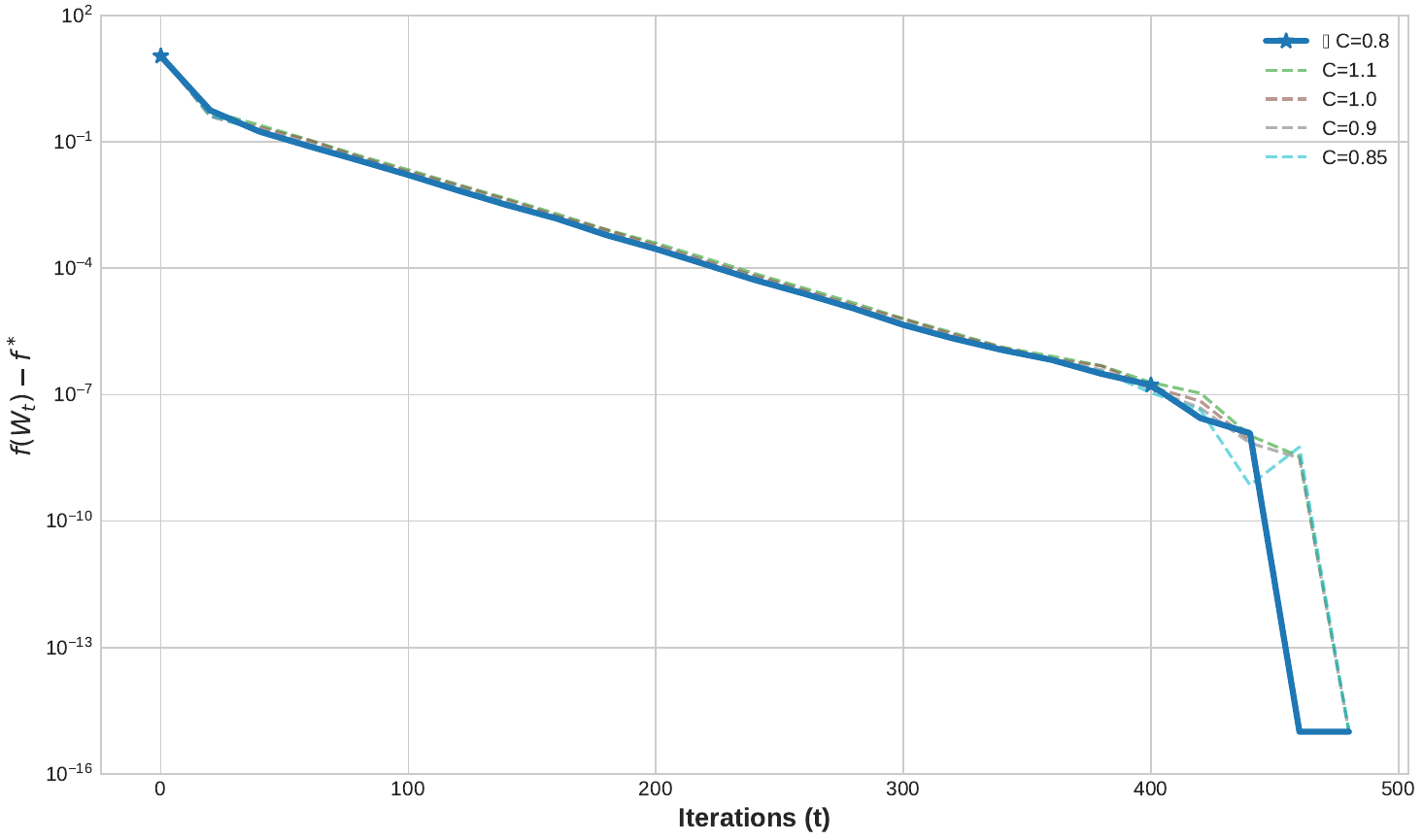}
        \caption{Loss}
        \label{fig:b1}
    \end{subfigure}
    \hfill
    \begin{subfigure}[b]{0.23\linewidth}
        \centering
        \includegraphics[width=\linewidth,height=\springerfigmaxheight,keepaspectratio]{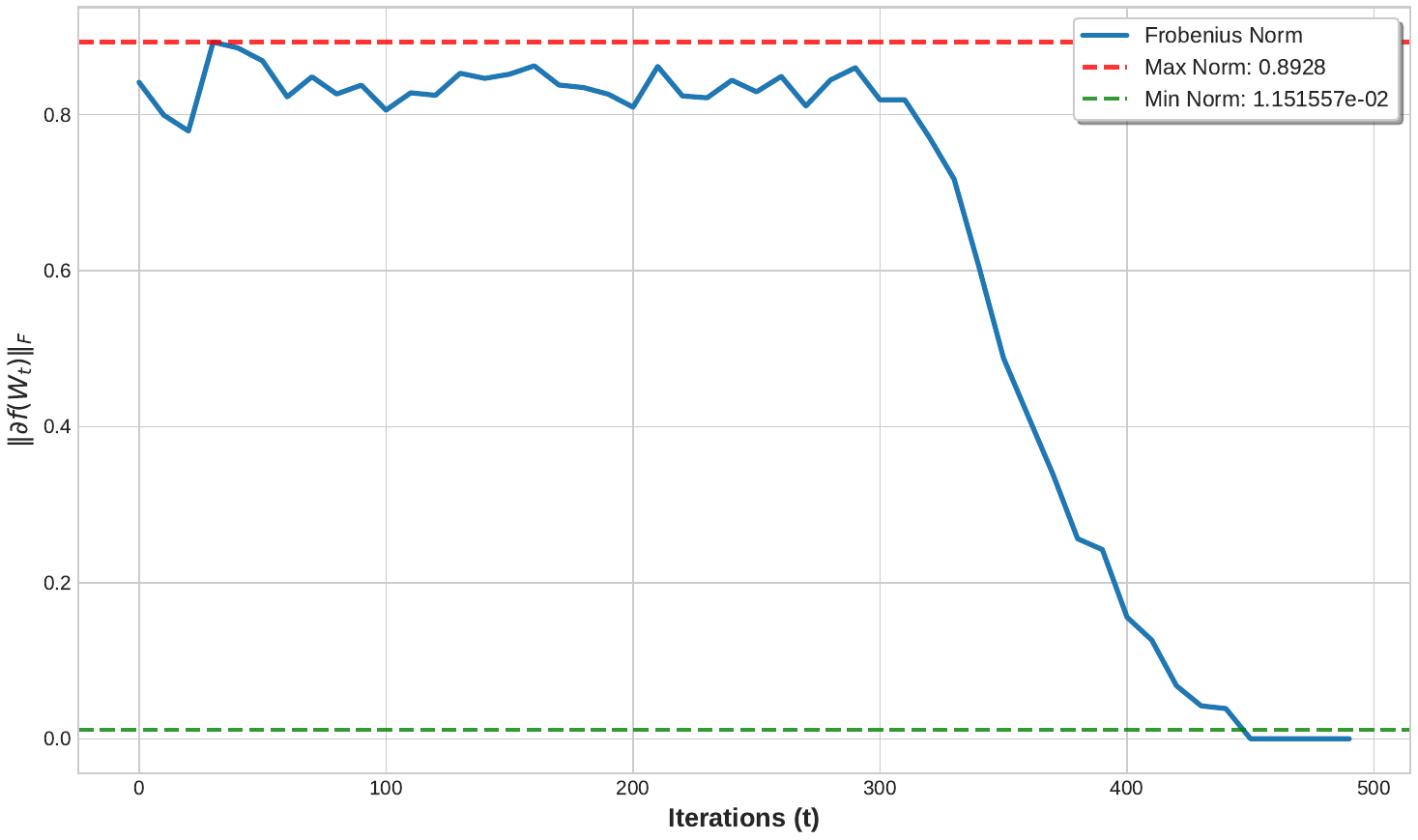}
        \caption{Gradient Norm}
        \label{fig:b2}
    \end{subfigure}
    \hfill
    \begin{subfigure}[b]{0.23\linewidth}
        \centering
        \includegraphics[width=\linewidth,height=\springerfigmaxheight,keepaspectratio]{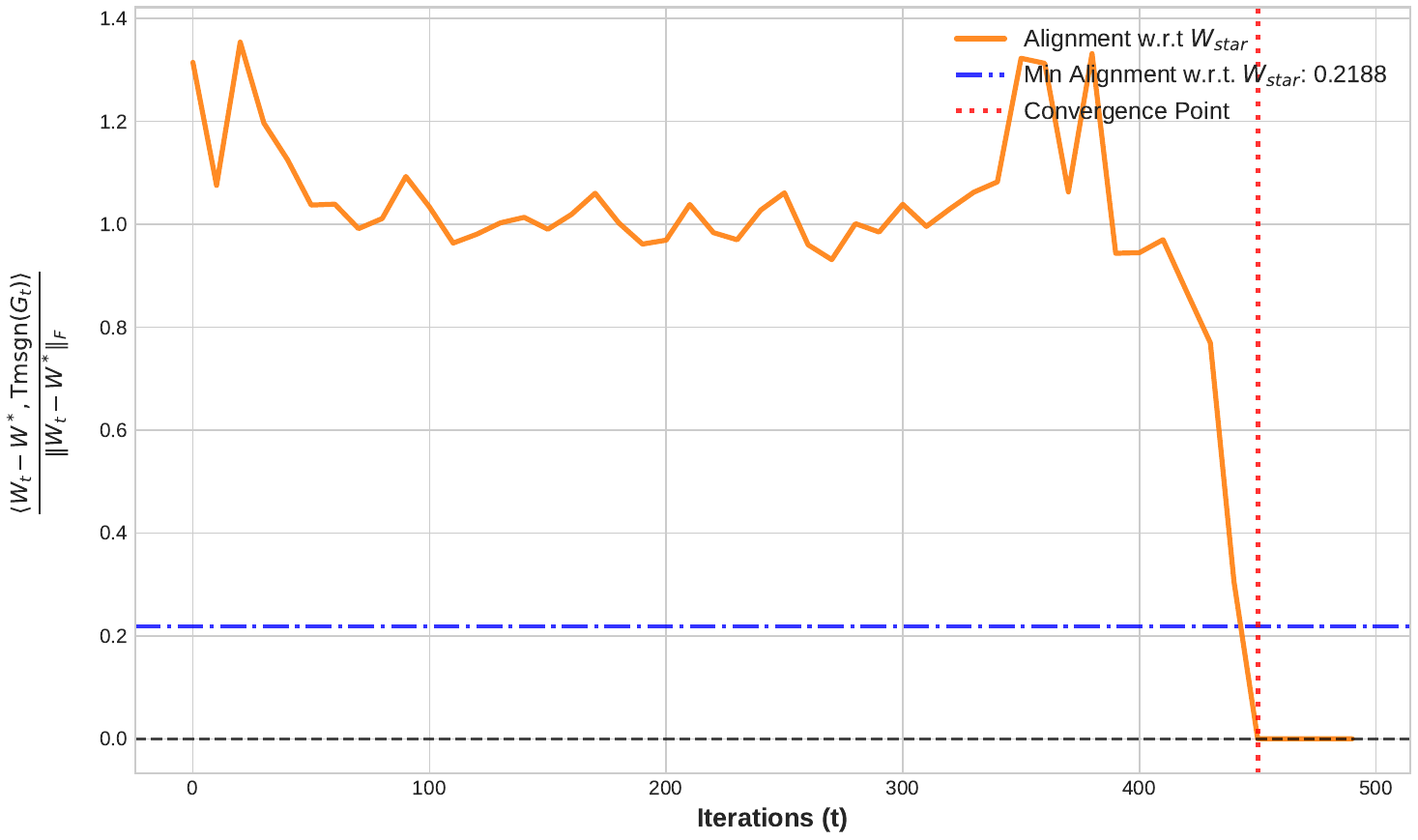}
        \caption{Alignment}
        \label{fig:b3}
    \end{subfigure}
    \hfill
    \begin{subfigure}[b]{0.23\linewidth}
        \centering
        \includegraphics[width=\linewidth,height=\springerfigmaxheight,keepaspectratio]{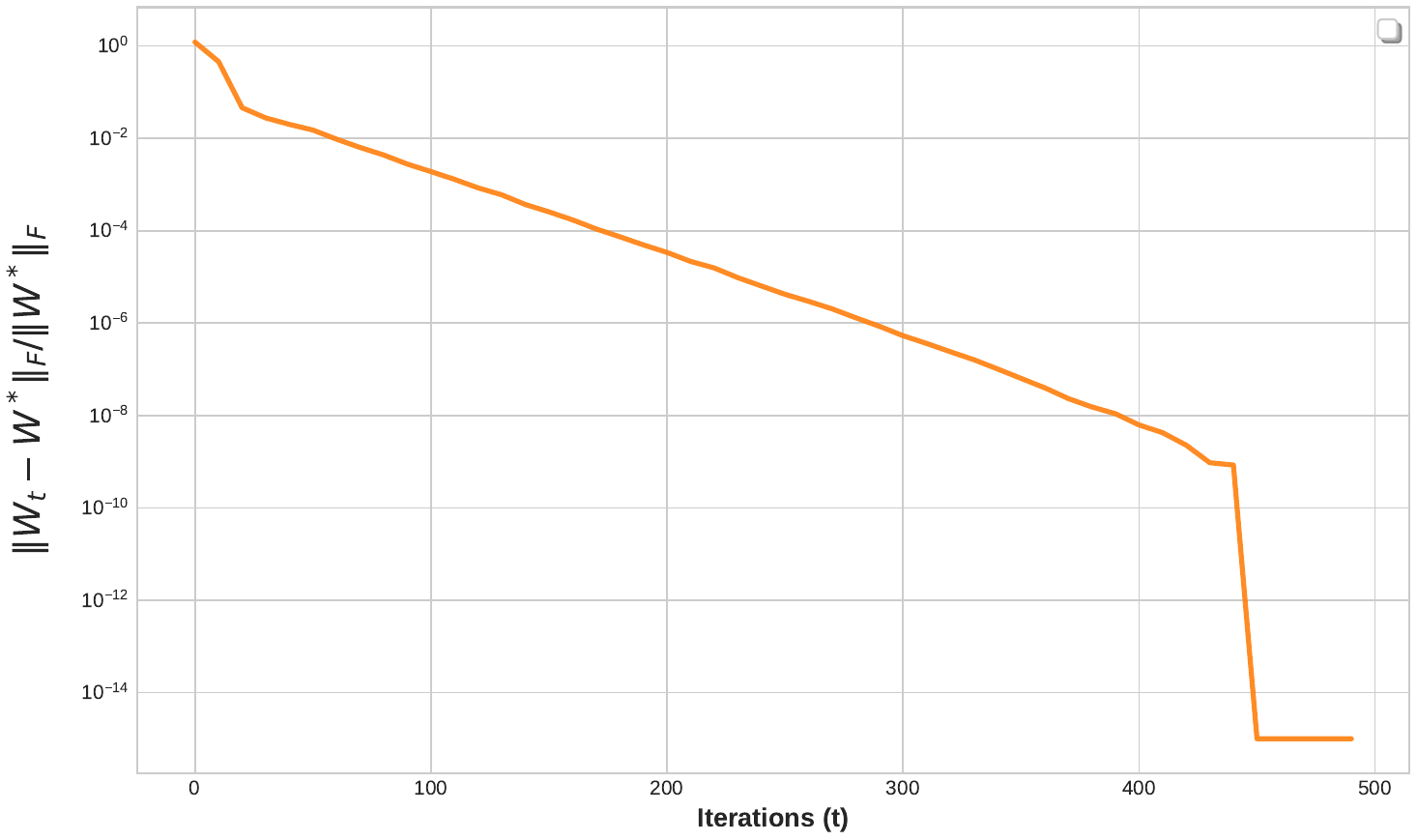}
        \caption{Relative error}
        \label{fig:b4}
    \end{subfigure}

    \vspace{0.05cm}

    \begin{subfigure}[b]{0.23\linewidth}
        \centering
        \includegraphics[width=\linewidth,height=\springerfigmaxheight,keepaspectratio]{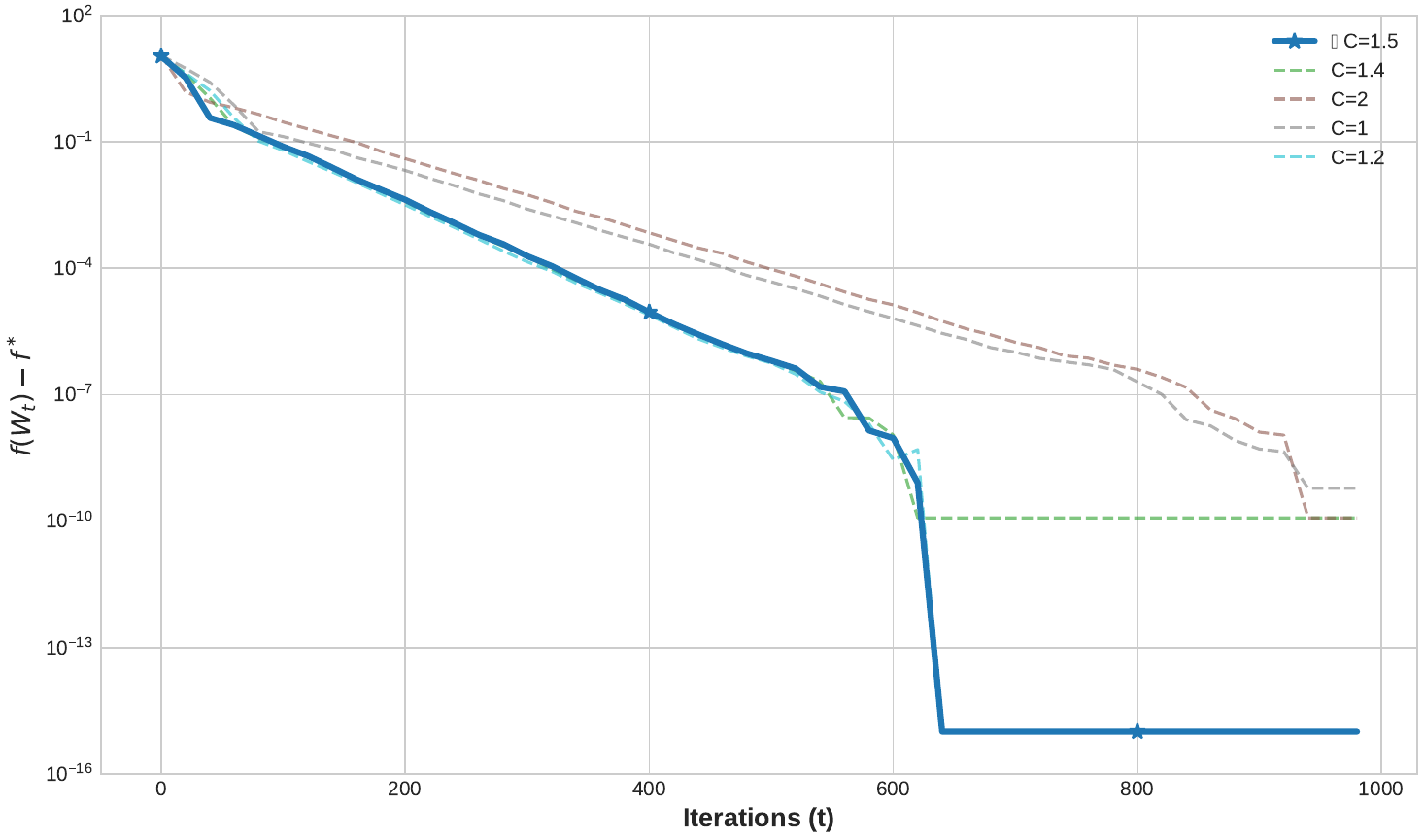}
        \caption{Loss}
        \label{fig:c1}
    \end{subfigure}
    \hfill
    \begin{subfigure}[b]{0.23\linewidth}
        \centering
        \includegraphics[width=\linewidth,height=\springerfigmaxheight,keepaspectratio]{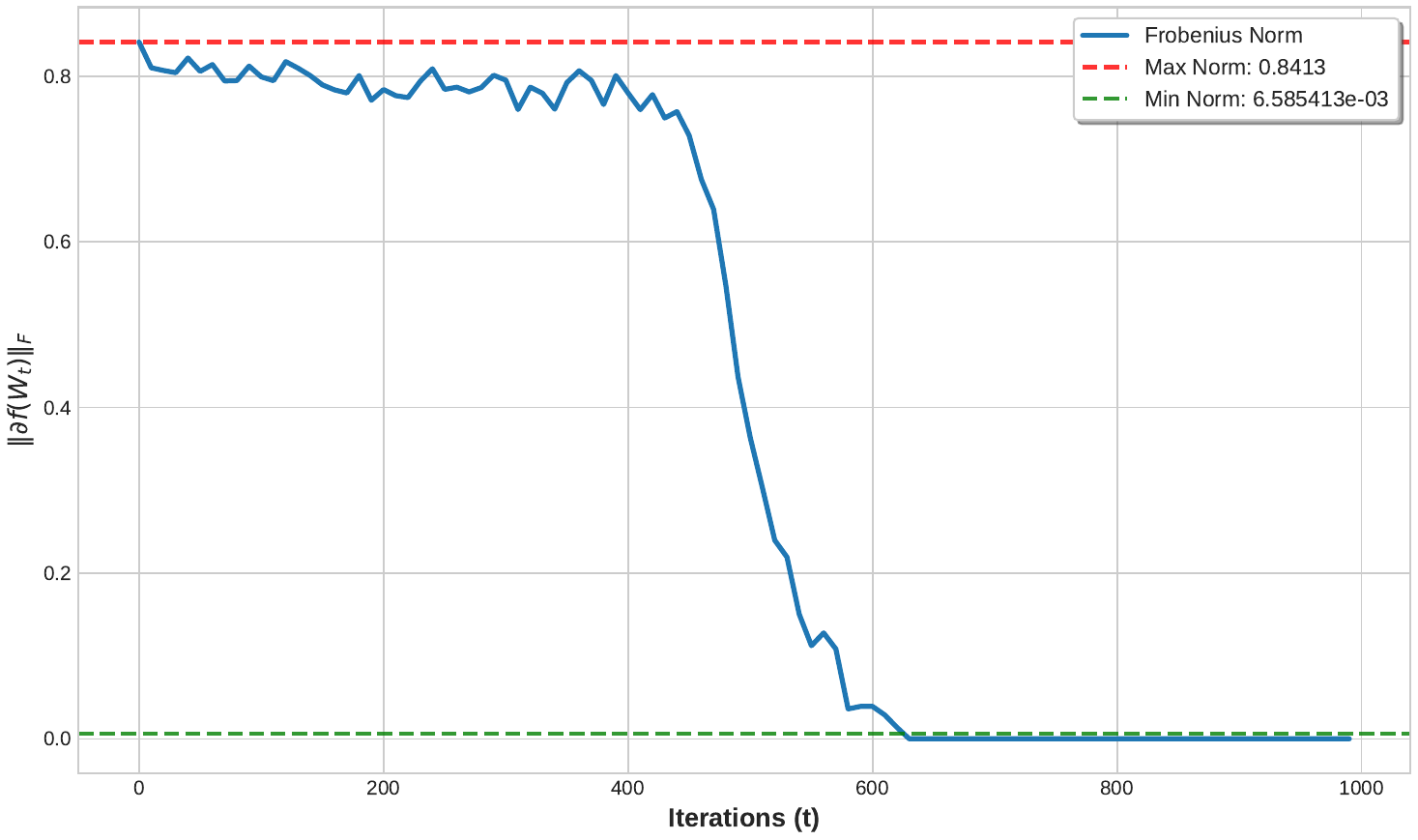}
        \caption{Gradient Norm}
        \label{fig:c2}
    \end{subfigure}
    \hfill
    \begin{subfigure}[b]{0.23\linewidth}
        \centering
        \includegraphics[width=\linewidth,height=\springerfigmaxheight,keepaspectratio]{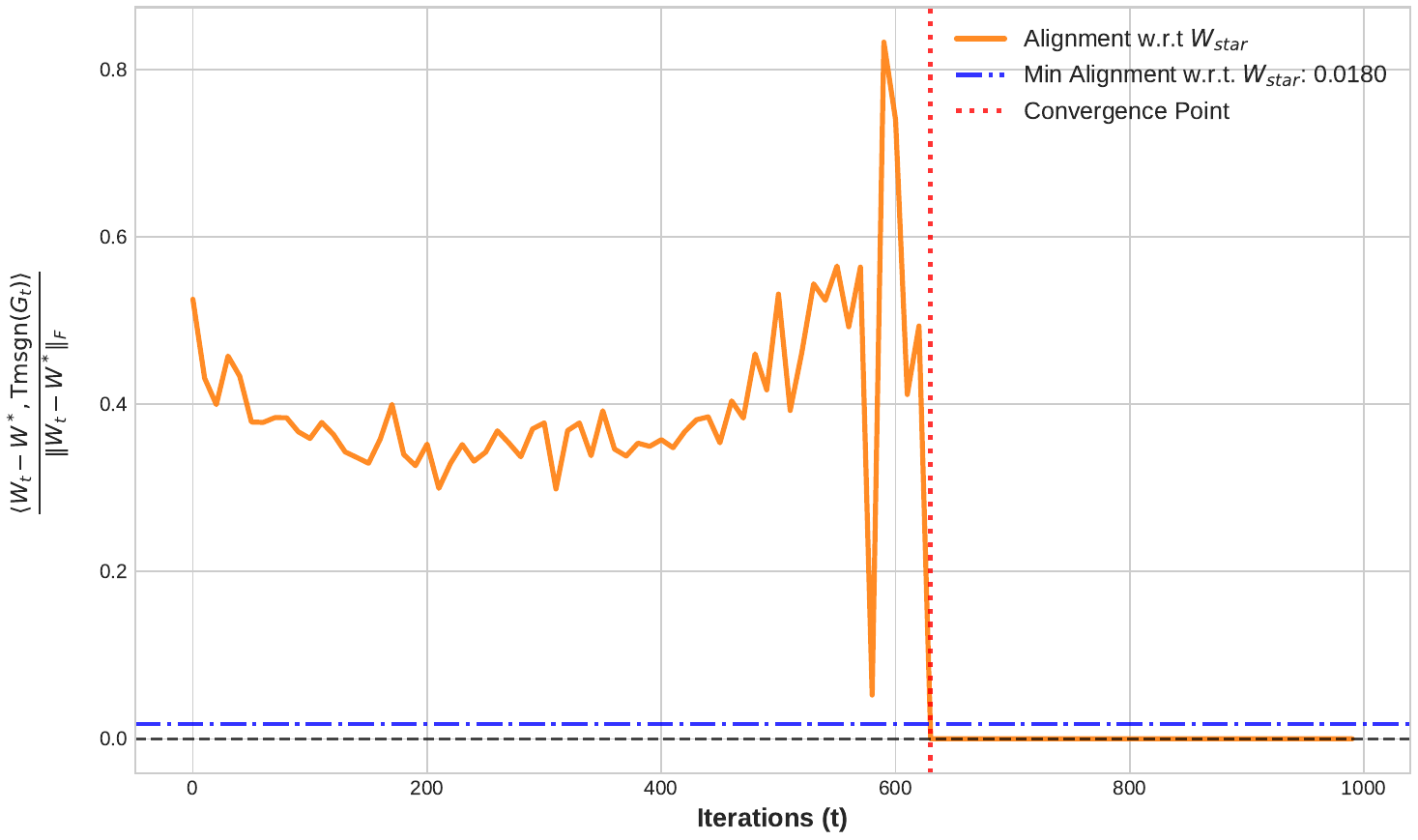}
        \caption{Alignment}
        \label{fig:c3}
    \end{subfigure}
    \hfill
    \begin{subfigure}[b]{0.23\linewidth}
        \centering
        \includegraphics[width=\linewidth,height=\springerfigmaxheight,keepaspectratio]{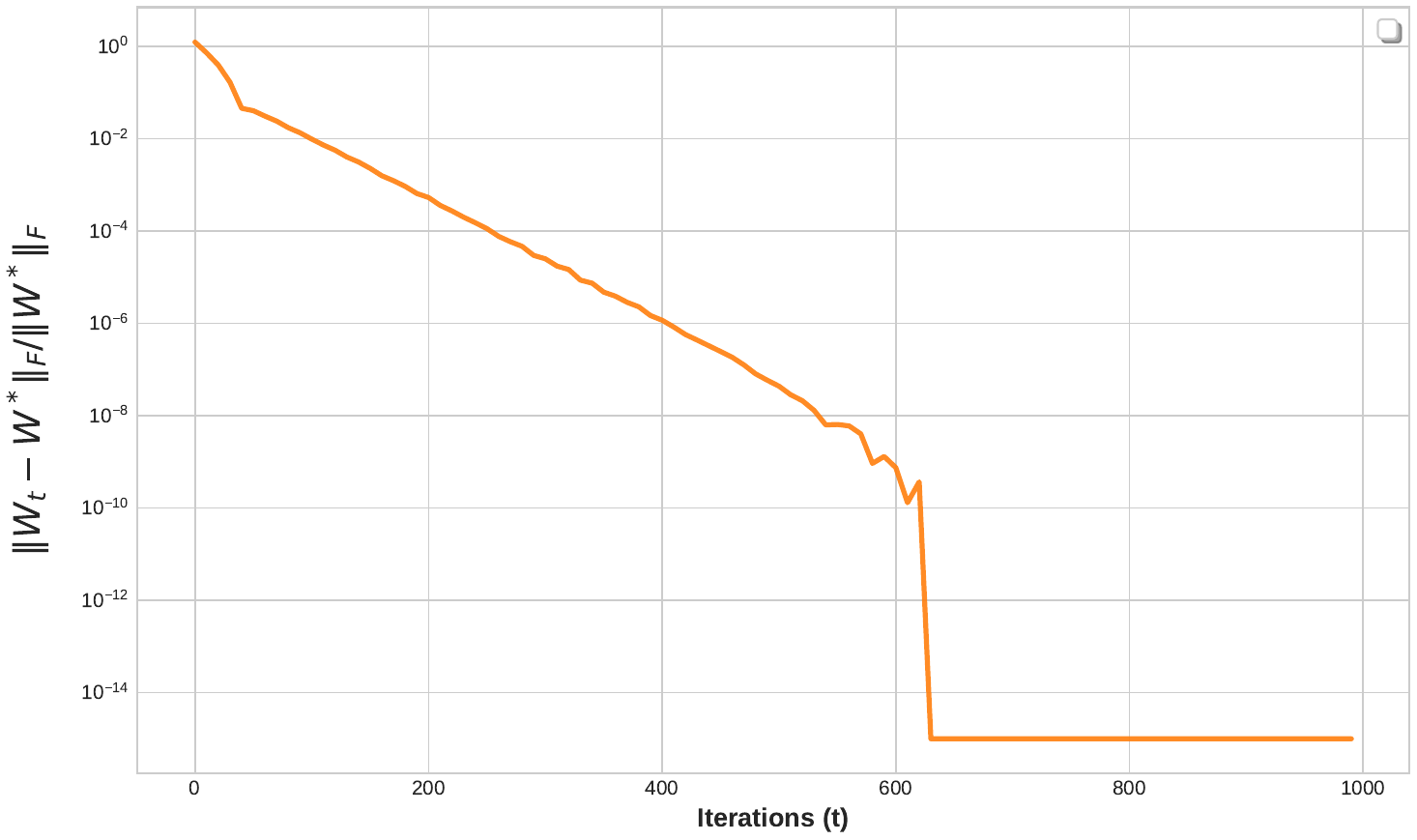}
        \caption{Relative error}
        \label{fig:c4}
    \end{subfigure}
    \end{minipage}

    \caption{Optimization dynamics of Spectral Descent and Truncated Spectral Descent on a linear programming problem \eqref{eq:linear_programming}.
    \textbf{Rows (from top to bottom):} Trajectories generated by standard SD, TSD with $s=3$, and TSD with $s=1$, respectively.
    \textbf{Columns (from left to right):} Evolution of the loss gap $f(W_t) - f^*$, subgradient norm $\|G_t\|_F$, alignment $ \left\langle W_t-W^*, (T)msgn(G_t)\right\rangle/\|W_t-W^*\|_F$, and relative parameter error $\|W_t - W^*\|_F / \|W^*\|_F$.}
    \label{fig:12_grid_comparison}
\end{figure}

\paragraph{Matrix Classification.}
Linear classification is a fundamental problem in machine learning \citep{soudry2018implicit}, where the goal is to learn a weight $W$ that effectively separates data points of different classes. When the data points are represented as matrices (e.g., images), the problem can be naturally formulated as matrix linear classification \citep{luo2015support}. Using the hinge loss, the objective function for this task is expressed as:
\begin{equation*}
    f(W) := \frac{1}{m}\sum_{i=1}^m \textbf{ReLU}(1 - y_i \langle W, X_i \rangle) = \frac{1}{m}\sum_{i=1}^m \max(0, 1 - y_i \langle W, X_i \rangle).
\end{equation*}
It is straightforward to verify that this objective function is a proper polyhedral convex function. Consequently, it inherently satisfies the sharpness property with respect to its optimal solution set $\mathcal{W}^*$.

To empirically evaluate the optimization dynamics of Spectral Descent (SD) for matrix classification, we construct a synthetic dataset comprising $m=2000$ samples.
Each data matrix $X_i \in \mathbb{R}^{10 \times 10}$ is generated with i.i.d. standard Gaussian entries $\mathcal{N}(0, 1)$.
The corresponding ground-truth labels are generated via $\hat{y}_i = \text{sgn}(\langle W^*, X_i \rangle)$, where the optimal weight matrix $W^* \in \mathbb{R}^{10 \times 10}$ is also drawn from $\mathcal{N}(0, 1)$.
To simulate a realistic, non-separable scenario and assess the algorithm's robustness, we randomly flip $10\%$ of the labels to introduce label noise.
It is important to note that for the non-smooth hinge loss objective, the optimal solution $W^*$ cannot be solved explicitly.
Therefore, we employ the CVXPY solver to obtain a high-precision numerical solution $W_{star}$, which serves as the reference $W^*$ for evaluating the parameter convergence of our algorithm.
We initialize $W^{(0)}$ from a standard Gaussian distribution and apply SD with a geometric decay step size schedule.
As depicted in Figure~\ref{fig:image1}, the initial step size (denoted as $C$) is tuned from the set $\{0.6, 0.7, 0.8, 0.9, 1.0\}$, while the decay factor is selected from $\{0.89, 0.91, 0.93, 0.96, 0.97, 0.99\}$.
For the relative error evaluation presented in Figure~\ref{fig:image2}, we specifically report the trajectory corresponding to the optimal parameter configuration that achieved the fastest convergence in Figure~\ref{fig:image1}.

Figure~\ref{fig:MC_convergence} illustrates the convergence behaviors of the objective loss gap $f(W_t) - f^*$ and the relative parameter error $\|W_t - W^*\|_F / \|W^*\|_F$ across 10,000 iterations.
As depicted in Figure~\ref{fig:image1}, the optimality gap exhibits a rapid and remarkably consistent descent across all tested initial step sizes $C$.
The loss drops precipitously in the early iterations and smoothly converges to a high-precision plateau of approximately $10^{-6}$. This empirical trajectory strongly supports our theoretical assertion: the combination of polyhedral convexity and the sharpness property provides a highly favorable landscape that allows spectral updates to efficiently minimize the non-smooth hinge loss.
Furthermore, Figure~\ref{fig:image2} reveals the algorithm's powerful recovery capabilities.
Despite the dataset being linearly inseparable, the relative parameter error strictly and rapidly decays.

  \begin{figure}[H]
    \centering
    \springersetfigwidth
    \begin{minipage}{\springerfigwidth}
    \centering

    \begin{subfigure}[b]{0.48\linewidth}
        \centering
        \includegraphics[width=\linewidth,height=\springerfigmaxheight,keepaspectratio]{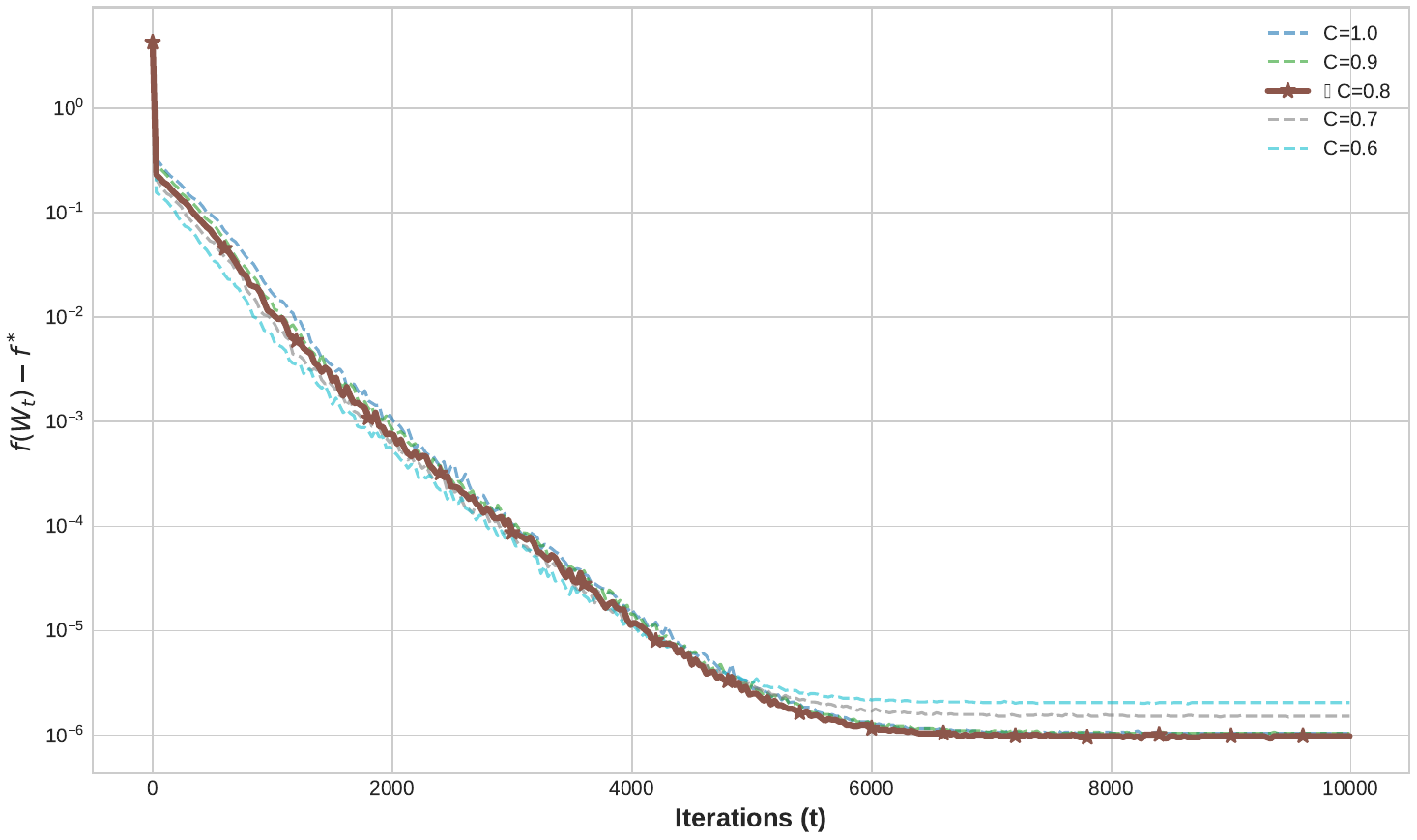}
        \caption{Loss}
        \label{fig:image1}
    \end{subfigure}
    \hfill
    \begin{subfigure}[b]{0.48\linewidth}
        \centering
        \includegraphics[width=\linewidth,height=\springerfigmaxheight,keepaspectratio]{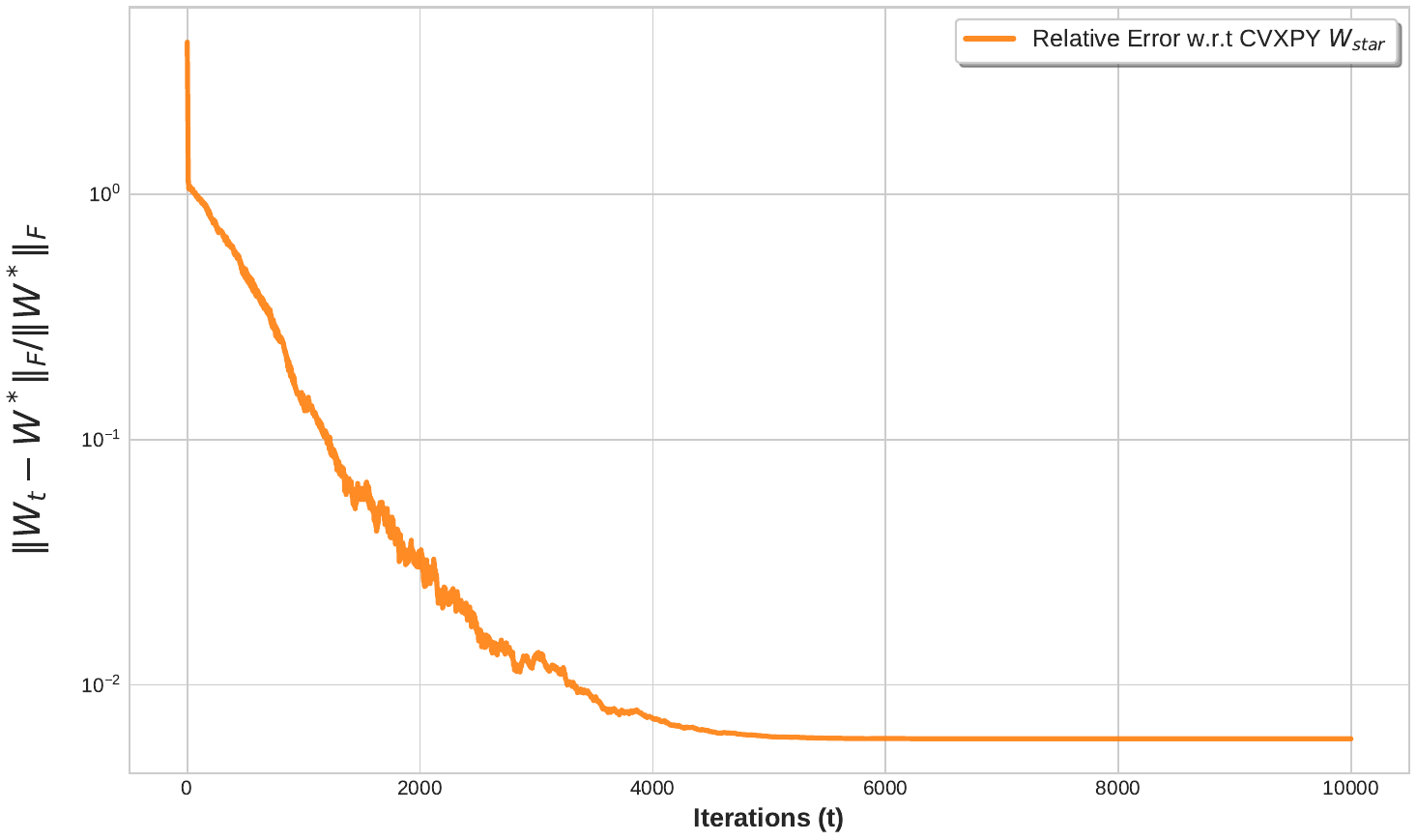}
        \caption{Relative Parameter Error}
        \label{fig:image2}
    \end{subfigure}
    \end{minipage}

    \caption{Convergence of Spectral Descent for matrix classification.}
    \label{fig:MC_convergence}
\end{figure}

\section{Discussion}
\label{sec:discussion}
In this paper, we presented a rigorous theoretical framework for understanding the convergence of Spectral Descent (SD) and its truncated variant (TSD) in non-smooth optimization.
By leveraging the geometric properties of convexity, Lipschitz continuity, and sharpness, we established the global linear convergence of momentum-free Muon-type optimizers.
Furthermore, we introduced decoupled weight decay as a spatial stabilization mechanism, explicitly connecting the regularized updates to the Conditional Subgradient (Frank-Wolfe) method.
This regularization theoretically circumvents stringent condition parameter requirements, guaranteeing a sublinear $\mathcal{O}(1/\sqrt{T})$ convergence rate.
We instantiated our framework on the robust low-rank matrix recovery problem via a least absolute deviation (LAD) formulation, demonstrating that our algorithms can efficiently recover the ground-truth matrix even in the presence of outliers and dense noise.

As the reader may note, our numerical experiments in Section~\ref{section:experiments} suggest that Spectral Descent and its variants exhibit remarkable robustness and efficiency that often exceed our theoretical worst-case bounds.
For instance, in the linear programming experiments, the algorithms successfully and rapidly converge even when the theoretical assumption regarding the condition parameter is not strictly satisfied.
This raises the interesting question: whether tighter, problem-specific alignment conditions can be derived to narrow the gap between deterministic worst-case bounds and empirical performance.

A critical direction for our future research is extending this theoretical framework to highly non-convex and non-smooth landscapes.
Modern machine learning architectures, including transformers and fully connected neural networks with ReLU activations, inherently possess these complex geometries.
While our current experiments empirically validate the superiority of Spectral Descent in training ReLU networks, establishing rigorous convergence guarantees for such settings remains an open challenge.
\backmatter

\section*{Acknowledgements}
This work was supported by NSFC under grant number U21A20426.

\bibliographystyle{plainnat}
\bibliography{refer}

@misc{jordan6muon,
  author={Jordan, Keller and Jin, Yuchen and Boza, Vlado and You, Jiacheng and Cecista, Franz and Newhouse, Laker and Bernstein, Jeremy},
  title={{Muon}: An optimizer for hidden layers in neural networks},
  year={2024},
  howpublished={Keller Jordan's blog. \url{https://kellerjordan.github.io/posts/muon/}},
  note={Accessed 11 May 2026}
}

@article{bernstein2024old,
  title={Old optimizer, new norm: An anthology},
  author={Bernstein, Jeremy and Newhouse, Laker},
  journal={arXiv:2409.20325},
  year={2024}
}

@article{shen2025convergence,
  title={On the convergence analysis of {Muon}},
  author={Shen, Wei and Huang, Ruichuan and Huang, Minhui and Shen, Cong and Zhang, Jiawei},
  journal={arXiv:2505.23737},
  year={2025}
}

@article{yang2026manifold,
  title={Manifold constrained steepest descent},
  author={Yang, Kaiwei and Lai, Lexiao},
  journal={arXiv:2601.21487},
  year={2026}
}

@article{carlson2015stochasticdgm,
  title={Stochastic spectral descent for discrete graphical models},
  author={Carlson, David and Hsieh, Ya-Ping and Collins, Edo and Carin, Lawrence and Cevher, Volkan},
  journal={IEEE J. Sel. Top. Signal Process.},
  volume={10},
  number={2},
  pages={296--311},
  year={2016},
  publisher={IEEE}
}

@article{carlson2015preconditioned,
  title={Preconditioned spectral descent for deep learning},
  author={Carlson, David E and Collins, Edo and Hsieh, Ya-Ping and Carin, Lawrence and Cevher, Volkan},
  journal={Advances in Neural Information Processing Systems},
  volume={28},
  year={2015b}
}

@inproceedings{carlson2015stochasticrbm,
  title={Stochastic spectral descent for restricted {Boltzmann} machines},
  author={Carlson, David and Cevher, Volkan and Carin, Lawrence},
  booktitle={Artificial Intelligence and Statistics},
  pages={111--119},
  year={2015a},
}

@article{li2025note,
  title={A note on the convergence of {Muon} and further},
  author={Li, Jiaxiang and Hong, Mingyi},
  journal={arXiv:2502.02900},
  year={2025},
  url={https://arxiv.org/abs/2502.02900}
}

@article{ma2026preconditioning,
  title={Preconditioning benefits of spectral orthogonalization in {Muon}},
  author={Ma, Jianhao and Huang, Yu and Chi, Yuejie and Chen, Yuxin},
  journal={arXiv:2601.13474},
  year={2026}
}

@article{sfyraki2025lions,
  title={{Lions} and {Muons}: Optimization via stochastic {Frank-Wolfe}},
  author={Sfyraki, Maria-Eleni and Wang, Jun-Kun},
  journal={arXiv:2506.04192},
  year={2025}
}

@article{liu2025muon,
  title={{Muon} is scalable for {LLM} training},
  author={Liu, Jingyuan and Su, Jianlin and Yao, Xingcheng and Jiang, Zhejun and Lai, Guokun and Du, Yulun and Qin, Yidao and Xu, Weixin and Lu, Enzhe and Yan, Junjie and others},
  journal={arXiv:2502.16982},
  year={2025}
}

@article{zhang2025adagrad,
  title={{AdaGrad} meets {Muon}: Adaptive stepsizes for orthogonal updates},
  author={Zhang, Minxin and Liu, Yuxuan and Schaeffer, Hayden},
  journal={arXiv:2509.02981},
  year={2025}
}

@article{ahn2025dion,
  title={{Dion}: Distributed orthonormalized updates},
  author={Ahn, Kwangjun and Xu, Byron and Abreu, Natalie and Fan, Ying and Magakyan, Gagik and Sharma, Pratyusha and Zhan, Zheng and Langford, John},
  journal={arXiv:2504.05295},
  year={2025}
}

@article{wang2025conda,
  title={{Conda}: Column-normalized {Adam} for training large language models faster},
  author={Wang, Junjie and Zhou, Pan and Dong, Yiming and Li, Huan and Li, Jia and Zhou, Xun and Lao, Qicheng and Fang, Cong and Lin, Zhouchen},
  journal={arXiv:2509.24218},
  year={2025}
}

@article{pang2026htmuon,
  title={{HTMuon}: Improving {Muon} via heavy-tailed spectral correction},
  author={Pang, Tianyu and Fang, Yujie and Liu, Zihang and Deng, Shenyang and Hsiung, Lei and Yu, Shuhua and Yang, Yaoqing},
  journal={arXiv:2603.10067},
  year={2026}
}

@misc{bernstein2025manifolds,
  author = {Bernstein, Jeremy},
  title = {Modular Manifolds},
  year = {2025},
  howpublished = {Thinking Machines Lab: Connectionism. \url{https://thinkingmachines.ai/blog/modular-manifolds/}},
  note = {Accessed 11 May 2026}
}

@misc{buchanan2025mmuonadmm,
  author = {Buchanan, Sam},
  title = {A faster manifold {Muon} with {ADMM}},
  year = {2025},
  howpublished={Sam D. Buchanan's blog. \url{https://sdbuchanan.com/blog/manifold-muon/}},
  note={Accessed 11 May 2026}
}

@misc{kexuefm-11215,
        title={Steepest descent on manifolds: 2. {Muon} + orthogonality},
        author={Su, Jianlin},
        year={2025},
        month={August},
        howpublished={\url{https://spaces.ac.cn/archives/11215}},
        note={Accessed 11 May 2026},
}

@misc{kexuefm-11221,
        title={Steepest descent on manifolds: 3. {Muon} + {Stiefel}},
        author={Su, Jianlin},
        year={2025},
        month={August},
        howpublished={\url{https://spaces.ac.cn/archives/11221}},
        note={Accessed 11 May 2026},
}

@inproceedings{brown2020language,
  title={Language models are few-shot learners},
  author={Brown, Tom and Mann, Benjamin and Ryder, Nick and Subbiah, Melanie and Kaplan, Jared D and Dhariwal, Prafulla and Neelakantan, Arvind and Shyam, Pranav and Sastry, Girish and Askell, Amanda and others},
  booktitle={Advances in Neural Information Processing Systems},
  volume={33},
  pages={1877--1901},
  year={2020}
}

@article{team2023gemini,
  title={{Gemini}: A family of highly capable multimodal models},
  author={{Gemini Team} and Anil, Rohan and Borgeaud, Sebastian and Alayrac, Jean-Baptiste and Yu, Jiahui and Soricut, Radu and Schalkwyk, Johan and Dai, Andrew M and Hauth, Anja and Millican, Katie and others},
  journal={arXiv:2312.11805},
  year={2023}
}

@article{kingma2014adam,
  title={{Adam}: A method for stochastic optimization},
  author={Kingma, Diederik P and Ba, Jimmy},
  journal={arXiv:1412.6980},
  year={2014}
}

@article{loshchilov2017decoupled,
  title={Decoupled weight decay regularization},
  author={Loshchilov, Ilya and Hutter, Frank},
  journal={arXiv:1711.05101},
  year={2017}
}

@article{shah2025practical,
  title={Practical efficiency of {Muon} for pretraining},
  author={Shah, Ishaan and Polloreno, Anthony M and Stratos, Karl and Monk, Philip and Chaluvaraju, Adarsh and Hojel, Andrew and Ma, Andrew and Thomas, Anil and Tanwer, Ashish and Shah, Darsh J and others},
  journal={arXiv:2505.02222},
  year={2025}
}

@article{tveit2025muon,
  title={{Muon} optimizer accelerates grokking},
  author={Tveit, Amund and Remseth, Bj{\o}rn and Skogvold, Arve},
  journal={arXiv:2504.16041},
  year={2025}
}

@article{zeng2025glm,
  title={{GLM-4.5}: Agentic, reasoning, and coding ({ARC}) foundation models},
  author={Zeng, Aohan and Lv, Xin and Zheng, Qinkai and Hou, Zhenyu and Chen, Bin and Xie, Chengxing and Wang, Cunxiang and Yin, Da and Zeng, Hao and Zhang, Jiajie and others},
  journal={arXiv:2508.06471},
  year={2025}
}

@misc{deepseekai2026deepseekv4,
  author={{DeepSeek-AI}},
  title={{DeepSeek-V4}: Towards Highly Efficient Million-Token Context Intelligence},
  howpublished={Technical report. \url{https://huggingface.co/deepseek-ai/DeepSeek-V4-Pro/resolve/main/DeepSeek_V4.pdf}},
  year={2026},
  note={Accessed 15 May 2026}
}

@article{si2025adamuon,
  title={{AdaMuon}: Adaptive {Muon} optimizer},
  author={Si, Chongjie and Zhang, Debing and Shen, Wei},
  journal={arXiv:2507.11005},
  year={2025}
}

@article{lau2025polargrad,
  title={{PolarGrad}: A class of matrix-gradient optimizers from a unifying preconditioning perspective},
  author={Lau, Tim Tsz-Kit and Long, Qi and Su, Weijie},
  journal={arXiv:2505.21799},
  year={2025}
}

@article{liu2025cosmos,
  title={{Cosmos}: A hybrid adaptive optimizer for memory-efficient training of {LLMs}},
  author={Liu, Liming and Xu, Zhenghao and Zhang, Zixuan and Kang, Hao and Li, Zichong and Liang, Chen and Chen, Weizhu and Zhao, Tuo},
  journal={arXiv:2502.17410},
  year={2025}
}

@article{an2025asgo,
  title={{ASGO}: Adaptive structured gradient optimization},
  author={An, Kang and Liu, Yuxing and Pan, Rui and Ren, Yi and Ma, Shiqian and Goldfarb, Donald and Zhang, Tong},
  journal={arXiv:2503.20762},
  year={2025}
}

@article{amsel2025polar,
  title={The polar express: Optimal matrix sign methods and their application to the {Muon} algorithm},
  author={Amsel, Noah and Persson, David and Musco, Christopher and Gower, Robert M},
  journal={arXiv:2505.16932},
  year={2025}
}

@article{pethick2025training,
  title={Training deep learning models with norm-constrained {LMOs}},
  author={Pethick, Thomas and Xie, Wanyun and Antonakopoulos, Kimon and Zhu, Zhenyu and Silveti-Falls, Antonio and Cevher, Volkan},
  journal={arXiv:2502.07529},
  year={2025}
}

@article{kravatskiy2025ky,
  title={The {Ky Fan} norms and beyond: Dual norms and combinations for matrix optimization},
  author={Kravatskiy, Alexey and Kozyrev, Ivan and Kozlov, Nikolai and Vinogradov, Alexander and Merkulov, Daniil and Oseledets, Ivan},
  journal={arXiv:2512.09678},
  year={2025}
}

@inproceedings{gupta2018shampoo,
  title={{Shampoo}: Preconditioned stochastic tensor optimization},
  author={Gupta, Vineet and Koren, Tomer and Singer, Yoram},
  booktitle={International Conference on Machine Learning},
  pages={1842--1850},
  year={2018},
}

@article{braun2026spectral,
  title={Spectral Gradient Descent Mitigates Anisotropy-Driven Misalignment: A Case Study in Phase Retrieval},
  author={Braun, Guillaume and Bao, Han and Huang, Wei and Imaizumi, Masaaki},
  journal={arXiv:2601.22652},
  year={2026}
}

@article{fan2026implicit,
  title={Implicit bias of spectral descent and {Muon} on multiclass separable data},
  author={Fan, Chen and Schmidt, Mark and Thrampoulidis, Christos},
  journal={Advances in Neural Information Processing Systems},
  volume={38},
  pages={39622--39669},
  year={2026}
}

@article{gronich2026implicit,
  title={The implicit bias of {Adam} and {Muon} on smooth homogeneous neural networks},
  author={Gronich, Eitan and Vardi, Gal},
  journal={arXiv:2602.16340},
  year={2026}
}

@article{li2026implicit,
  title={The implicit bias of steepest descent with mini-batch stochastic gradient},
  author={Li, Jichu and Tang, Xuan and Zou, Difan},
  journal={arXiv:2602.11557},
  year={2026}
}

@inproceedings{chen2024lion,
  title={{Lion} secretly solves a constrained optimization: As {Lyapunov} predicts},
  author={Chen, Lizhang and Liu, Bo and Liang, Kaizhao and Liu, Qiang},
  booktitle={International Conference on Learning Representations},
  year={2024},
}

@article{chen2025muon,
  title={{Muon} optimizes under spectral norm constraints},
  author={Chen, Lizhang and Li, Jonathan and Liu, Qiang},
  journal={arXiv:2506.15054},
  year={2025}
}

@article{burke1993weak,
  title={Weak sharp minima in mathematical programming},
  author={Burke, James V and Ferris, Michael C},
  journal={SIAM J. Control Optim.},
  volume={31},
  number={5},
  pages={1340--1359},
  year={1993},
  publisher={SIAM}
}

@article{luo1993error,
  title={Error bounds and convergence analysis of feasible descent methods: a general approach},
  author={Luo, Zhi-Quan and Tseng, Paul},
  journal={Ann. Oper. Res.},
  volume={46},
  number={1},
  pages={157--178},
  year={1993},
  publisher={Springer}
}

@article{studniarski1999weak,
  title={Weak sharp minima: characterizations and sufficient conditions},
  author={Studniarski, Marcin and Ward, Doug E},
  journal={SIAM J. Control Optim.},
  volume={38},
  number={1},
  pages={219--236},
  year={1999},
  publisher={SIAM}
}

@article{burke2002weak,
  title={Weak sharp minima revisited Part I: basic theory},
  author={Burke, James and Deng, Sien},
  journal={Control Cybern.},
  volume={31},
  number={3},
  pages={439--469},
  year={2002},
  publisher={Polska Akademia Nauk. Instytut Bada{\'n} Systemowych PAN}
}

@article{xu2024convergence,
  title={Convergence of projected subgradient method with sparse or low-rank constraints},
  author={Xu, Hang and Li, Song and Lin, Junhong},
  journal={Adv. Comput. Math.},
  volume={50},
  number={4},
  pages={62},
  year={2024},
  publisher={Springer}
}

@article{davis2018subgradient,
  title={Subgradient methods for sharp weakly convex functions},
  author={Davis, Damek and Drusvyatskiy, Dmitriy and MacPhee, Kellie J and Paquette, Courtney},
  journal={J. Optim. Theory Appl.},
  volume={179},
  number={3},
  pages={962--982},
  year={2018},
  publisher={Springer}
}

@article{davis2020nonsmooth,
  title={The nonsmooth landscape of phase retrieval},
  author={Davis, Damek and Drusvyatskiy, Dmitriy and Paquette, Courtney},
  journal={IMA J. Numer. Anal.},
  volume={40},
  number={4},
  pages={2652--2695},
  year={2020},
  publisher={Oxford University Press}
}

@article{ding2023sharpness,
  title={Sharpness and well-conditioning of nonsmooth convex formulations in statistical signal recovery},
  author={Ding, Lijun and Wang, Alex L},
  journal={arXiv:2307.06873},
  year={2023}
}

@article{charisopoulos2021low,
  title={Low-rank matrix recovery with composite optimization: good conditioning and rapid convergence},
  author={Charisopoulos, Vasileios and Chen, Yudong and Davis, Damek and D{\'\i}az, Mateo and Ding, Lijun and Drusvyatskiy, Dmitriy},
  journal={Found. Comput. Math.},
  volume={21},
  number={6},
  pages={1505--1593},
  year={2021},
  publisher={Springer}
}

@article{duchi2019solving,
  title={Solving (most) of a set of quadratic equalities: Composite optimization for robust phase retrieval},
  author={Duchi, John C and Ruan, Feng},
  journal={Information and Inference: A Journal of the IMA},
  volume={8},
  number={3},
  pages={471--529},
  year={2019},
  publisher={Oxford University Press}
}

@article{zheng2024adaptive,
  title={Adaptive algorithms for robust phase retrieval},
  author={Zheng, Zhong and Aybat, Necdet Serhat and Ma, Shiqian and Xue, Lingzhou},
  journal={arXiv:2409.19162},
  year={2024}
}

@article{ma2025second,
  title={On second-order weak sharp minima of general nonconvex set-constrained optimization problems},
  author={Ma, Xiaoxiao and Ouyang, Wei and Ye, Jane J and Zhang, Binbin},
  journal={J. Optim. Theory Appl.},
  volume={207},
  number={2},
  pages={21},
  year={2025},
  publisher={Springer}
}

@article{bauer1958minimalstellen,
  title={Minimalstellen von funktionen und extremalpunkte},
  author={Bauer, Heinz},
  journal={Arch. Math.},
  volume={9},
  number={4},
  pages={389--393},
  year={1958},
  publisher={Springer}
}

@article{frank1956algorithm,
  title={An algorithm for quadratic programming},
  author={Frank, Marguerite and Wolfe, Philip},
  journal={Nav. Res. Logist. Q.},
  volume={3},
  number={1-2},
  pages={95--110},
  year={1956}
}

@inproceedings{jaggi2013revisiting,
  title={Revisiting {Frank-Wolfe}: Projection-free sparse convex optimization},
  author={Jaggi, Martin},
  booktitle={International Conference on Machine Learning},
  pages={427--435},
  year={2013},
}

@article{ravi2019deterministic,
  title={A deterministic nonsmooth {Frank-Wolfe} algorithm with coreset guarantees},
  author={Ravi, Sathya N and Collins, Maxwell D and Singh, Vikas},
  journal={INFORMS J. Optim.},
  volume={1},
  number={2},
  pages={120--142},
  year={2019},
  publisher={INFORMS}
}

@book{clarke1990optimization,
  title={Optimization and nonsmooth analysis},
  author={Clarke, Frank H},
  year={1990},
  publisher={SIAM},
  address={Philadelphia, PA}
}

@article{recht2010guaranteed,
  title={Guaranteed minimum-rank solutions of linear matrix equations via nuclear norm minimization},
  author={Recht, Benjamin and Fazel, Maryam and Parrilo, Pablo A},
  journal={SIAM Rev.},
  volume={52},
  number={3},
  pages={471--501},
  year={2010},
  publisher={SIAM}
}

@article{dielman2005least,
  title={Least absolute value regression: recent contributions},
  author={Dielman, Terry E},
  journal={J. Stat. Comput. Simul.},
  volume={75},
  number={4},
  pages={263--286},
  year={2005},
  publisher={Taylor \& Francis}
}

@article{li2020nonconvex,
  title={Nonconvex robust low-rank matrix recovery},
  author={Li, Xiao and Zhu, Zhihui and So, Anthony Man-Cho and Vidal, Ren{\'e}},
  journal={SIAM J. Optim.},
  volume={30},
  number={1},
  pages={660--686},
  year={2020},
  publisher={SIAM}
}

@inproceedings{srebro2004maximum,
  title={Maximum-margin matrix factorization},
  author={Srebro, Nathan and Rennie, Jason and Jaakkola, Tommi},
  booktitle={Advances in Neural Information Processing Systems},
  volume={17},
  year={2004}
}

@inproceedings{li2018algorithmic,
  title={Algorithmic regularization in over-parameterized matrix sensing and neural networks with quadratic activations},
  author={Li, Yuanzhi and Ma, Tengyu and Zhang, Hongyang},
  booktitle={Conference on Learning Theory},
  pages={2--47},
  year={2018},
}

@article{davenport2016overview,
  title={An overview of low-rank matrix recovery from incomplete observations},
  author={Davenport, Mark A and Romberg, Justin},
  journal={IEEE J. Sel. Top. Signal Process.},
  volume={10},
  number={4},
  pages={608--622},
  year={2016},
  publisher={IEEE}
}

@article{huang2025adversarial,
  title={Adversarial phase retrieval via nonlinear least absolute deviation},
  author={Huang, Gao and Li, Song and Xu, Hang},
  journal={IEEE Trans. Inf. Theory},
  year={2025},
  publisher={IEEE}
}

@article{watson1992characterization,
  title={Characterization of the subdifferential of some matrix norms},
  author={Watson, G Alistair},
  journal={Linear Algebra Appl.},
  volume={170},
  number={1},
  pages={33--45},
  year={1992},
}

@inproceedings{luo2015support,
  title={Support matrix machines},
  author={Luo, Luo and Xie, Yubo and Zhang, Zhihua and Li, Wu-Jun},
  booktitle={International Conference on Machine Learning},
  pages={938--947},
  year={2015},
}

@article{soudry2018implicit,
  title={The implicit bias of gradient descent on separable data},
  author={Soudry, Daniel and Hoffer, Elad and Nacson, Mor Shpigel and Gunasekar, Suriya and Srebro, Nathan},
  journal={J. Mach. Learn. Res.},
  volume={19},
  number={70},
  pages={1--57},
  year={2018}
}

@inproceedings{roulet2017sharpness,
  title={Sharpness, restart and acceleration},
  author={Roulet, Vincent and d'Aspremont, Alexandre},
  booktitle={Advances in Neural Information Processing Systems},
  volume={30},
  year={2017}
}

\newpage
\appendix
\etocdepthtag.toc{mtappendix}
\etocsettagdepth{mtmain}{none}
\etocsettagdepth{mtappendix}{subsection} 

\tableofcontents 

\section{Technical Lemmas for Subsection \ref{ssec:spectral_descent}}
\label{app:lower_bound_descent_term}
\subsection{Lower Bound for Spectral Descent (Theorem \ref{theorem:spectral_descent_convergence})}
\begin{lemma}[A Uniform Lower Bound on the Descent Term]
\label{lemma:lower_bound_descent_term_SD}
Let $\mathcal{R}>0$ and $L>\mu>0$, and define $\kappa := \mu/L \in (0,1)$.
Consider the feasible set
\begin{equation}
    \mathcal{F} := \left\{ (\mathbf{x}, \boldsymbol{\sigma}) \in \mathbb{R}^n \times \mathbb{R}^n \;\middle|\;
    \begin{aligned}
        & \|\mathbf{x}\|_2 \le \mathcal{R}, \quad \|\boldsymbol{\sigma}\|_2 \le L, \\
        & \boldsymbol{\sigma}^\top \mathbf{x} \ge \mu\mathcal{R}, \quad \sigma_1 \ge \dots \ge \sigma_n \ge 0
    \end{aligned}
    \right\}
\end{equation}
and the objective function $f(\mathbf{x}) := \mathbf{1}^\top \mathbf{x} = \sum_{i=1}^n x_i$.
Then, for any $(\mathbf{x}, \boldsymbol{\sigma}) \in \mathcal{F}$, we have
\begin{equation}
    f(\mathbf{x}) \geq \left( \kappa - \sqrt{n-1}\sqrt{1-\kappa^2} \right)\mathcal{R}.
\end{equation}
Furthermore, this lower bound is tight (i.e., strictly achievable).
\end{lemma}

\begin{proof}
 Let $(\mathbf{x}, \boldsymbol{\sigma}) \in \mathcal{F}$ be any feasible pair.
 For any feasible $\mathbf{x}$, define the index set of positive components $\mathcal I := \{ i\in[n] : x_i>0\}$ and its complement $\mathcal I^c := [n]\setminus \mathcal I$.
Then the objective function $f(\mathbf{x})$ can be decomposed as:
    \begin{equation}
        \label{eq:obj_decomp}
        f(\mathbf{x})= \sum_{i \in \mathcal{I}} x_i - \sum_{i \in \mathcal{I}^c} |x_i| = \|\mathbf{x}_{\mathcal{I}}\|_1 - \|\mathbf{x}_{\mathcal{I}^c}\|_1,
    \end{equation}
    where $\mathbf{x}_{\mathcal{I}}$ and $\mathbf{x}_{\mathcal{I}^c}$ denote the sub-vectors restricted to the respective index sets.

We first show that $\mathcal I$ is nonempty.
If $\mathcal I=\emptyset$, then $x_i\leq 0$ for all $i$.
Since $\sigma_i\ge 0$ for all $i$, this implies
$\boldsymbol{\sigma}^\top \mathbf{x}=\sum_{i=1}^n \sigma_i x_i\le 0$,
which contradicts the constraint $\boldsymbol{\sigma}^\top \mathbf{x}\ge \mu\mathcal R>0$.
Hence $\mathcal I\neq\emptyset$ and $|\mathcal I^c|\le n-1$.

By Cauchy--Schwarz inequality, we have
$
\|\mathbf{x}_{\mathcal{I}^c}\|_1=\sum_{i\in\mathcal I^c}|x_i|
\leq\sqrt{|\mathcal I^c|}\,\|\mathbf{x}_{\mathcal{I}^c}\|_2,
$
and
$
\|\mathbf{x}_{\mathcal{I}}\|_1\geq\|\mathbf{x}_{\mathcal{I}}\|_2,
$
where the second inequality follows from the fact that $\|\mathbf{v}\|_1\ge \|\mathbf{v}\|_2$, for any vector $\mathbf{v}\in\R^k$.
Combining the two bounds yields
\begin{equation}
\label{eq:f_lower_bound_refined}
f(\mathbf{x})=\|\mathbf{x}_{\mathcal{I}}\|_1-\|\mathbf{x}_{\mathcal{I}^c}\|_1
\geq \|\mathbf{x}_{\mathcal{I}}\|_2-\sqrt{|\mathcal I^c|}\,\|\mathbf{x}_{\mathcal{I}^c}\|_2.
\end{equation}

Since $\sigma_i\ge 0$ and $x_i\leq 0$ for $i\in\mathcal I^c$, we have
$\boldsymbol{\sigma}^\top \mathbf{x}
=\sum_{i\in\mathcal I}\sigma_i x_i+\sum_{i\in\mathcal I^c}\sigma_i x_i
\leq \sum_{i\in\mathcal I}\sigma_i x_i$.
Therefore, using Cauchy--Schwarz inequality and $\|\boldsymbol{\sigma}_{\mathcal I}\|_2\leq \|\boldsymbol{\sigma}\|_2\le L$, we have
$\mu\mathcal R
\leq \sum_{i\in\mathcal I}\sigma_i x_i
\leq \|\boldsymbol{\sigma}_{\mathcal I}\|_2\,\|\mathbf{x}_{\mathcal I}\|_2
\leq L\,\|\mathbf{x}_{\mathcal I}\|_2$,
which implies $\|\mathbf{x}_{\mathcal{I}}\|_2\geq \kappa\mathcal R$ with $\kappa:=\mu/L$.
Moreover, we obtain $\|\mathbf{x}_{\mathcal I^c}\|_2^2
= \|\mathbf{x}\|_2^2-\|\mathbf{x}_{\mathcal I}\|_2^2
\leq \mathcal R^2-\kappa^2\mathcal R^2
=(1-\kappa^2)\mathcal R^2$.
Substituting these bounds into \eqref{eq:f_lower_bound_refined} yields, for any feasible $\mathbf{x}$,
\begin{equation}
f(\mathbf{x})
\geq \kappa\mathcal R-\sqrt{|\mathcal I^c|}\sqrt{1-\kappa^2}\,\mathcal R
\geq \Bigl(\kappa-\sqrt{n-1}\sqrt{1-\kappa^2}\Bigr)\mathcal R.
\end{equation}

Let
\begin{equation}
\label{eq:tightness_construction}
\left\{
\begin{aligned}
    \boldsymbol{\sigma}^* &= (L, 0, \dots, 0)^\top, \\
    x_j^* &= \begin{cases}
        \kappa \mathcal{R}, & j=1, \\
        -\sqrt{\frac{1-\kappa^2}{n-1}} \mathcal{R}, & j=2, \dots, n.
    \end{cases}
\end{aligned}
\right.
\end{equation}
It is straightforward to verify that $(\mathbf{x}^*, \boldsymbol{\sigma}^*) \in \mathcal{F}$ and $f(\mathbf{x}^*) = (\kappa - \sqrt{n-1}\sqrt{1-\kappa^2})\mathcal{R}$.
This confirms the tightness of the lower bound, completing the proof.
\end{proof}

\subsection{Lower Bound for Truncated Spectral Descent (Theorem \ref{theorem:truncated_spectral_descent_convergence})}
\label{app:lower_bound_descent_term_TSD}
The analysis for the truncated case is more involved than standard SD because the objective function depends only on the first $s$ components, while the constraints involve the full vectors.
We establish the tight uniform lower bound as follows.
\begin{lemma}[A Uniform Lower Bound on the Truncated Descent Term]
\label{lemma:lower_bound_descent_term_Truncated}
Let $\mathcal{R}>0$, $L>\mu>0$, and integer $1 \leq s \leq n$. Define $\kappa := \mu/L \in (0,1)$.
Consider the feasible set
\begin{equation}
	\label{constraint_set_Truncated}
    \mathcal{F}_s := \left\{ (\mathbf{x}, \boldsymbol{\sigma}) \in \mathbb{R}^n \times \mathbb{R}^n \;\middle|\;
    \begin{aligned}
        & \|\mathbf{x}\|_2 \le \mathcal{R}, \quad \|\boldsymbol{\sigma}\|_2 \le L, \\
        & \boldsymbol{\sigma}^\top \mathbf{x} \geq \mu\mathcal{R}, \\
        & \sigma_1 \geq \sigma_2 \geq \dots \geq \sigma_n \geq 0
    \end{aligned}
    \right\}
\end{equation}
and the objective function $f_s(\mathbf{x}) := \sum_{i=1}^s x_i$.
Then, for any $(\mathbf{x}, \boldsymbol{\sigma}) \in \mathcal{F}_s$, we have
\begin{equation}
    f_s(\mathbf{x}) \geq \left( \kappa \alpha_s - \sqrt{s - \alpha_s^2} \sqrt{1-\kappa^2} \right)\mathcal{R},
\end{equation}
where $\alpha_s$ is defined as $\alpha_s := \min\left\{ 1, \frac{s}{\sqrt{n}} \right\}$.
Furthermore, this lower bound is tight.
\end{lemma}
\begin{proof}
  By the fundamental property of minimization over nested domains, the uniform lower bound of $f_s(\mathbf{x})$ over the joint set $\mathcal{F}_s$ is equivalent to a two-stage optimization problem:
    \begin{equation}
        \label{eq:two_stage_problem}
        f^* = \min_{\boldsymbol{\sigma} \in \Sigma} \left( \min_{\mathbf{x} \in \Xi(\boldsymbol{\sigma})} \mathbf{c}^\top \mathbf{x} \right),
    \end{equation}
    where $\mathbf{c} := (\underbrace{1,\ldots,1}_{s}, 0,\ldots,0)^\top$, $\Sigma := \{ \boldsymbol{\sigma} \mid \|\boldsymbol{\sigma}\|_2 \le L, \sigma_1 \ge \dots \ge \sigma_n \ge 0 \}$, and $\Xi(\boldsymbol{\sigma}) := \{ \mathbf{x} \mid \|\mathbf{x}\|_2 \le \mathcal{R}, \boldsymbol{\sigma}^\top \mathbf{x} \ge \mu\mathcal{R} \}$.
    We solve this sequentially in two steps.

    \textit{Step 1: Inner Minimization.}
    For a fixed $\boldsymbol{\sigma} \in \Sigma$ with $\|\boldsymbol{\sigma}\|_2 > \mu$, Slater's condition holds.
    The KKT stationarity condition $\nabla_{\mathbf{x}} \mathcal{L} = \mathbf{0}$ yields the optimal form $\mathbf{x}^* = \frac{1}{\lambda}(\nu \boldsymbol{\sigma} - \mathbf{c})$.
    Assuming both constraints are active, substituting $\mathbf{x}^*$ into the linear constraint $\boldsymbol{\sigma}^\top \mathbf{x}^* = \mu\mathcal{R}$ determines the multiplier as $\lambda = (\nu \|\boldsymbol{\sigma}\|_2^2 - \mathbf{c}^\top \boldsymbol{\sigma}) / (\mu\mathcal{R})$.
    Plugging this relation into the norm constraint $\|\mathbf{x}^*\|_2^2 = \mathcal{R}^2$ eliminates $\lambda$ and reduces to a single quadratic equation in $\nu$:
    \begin{equation}
        \nu^2 \|\boldsymbol{\sigma}\|_2^2 - 2\nu \mathbf{c}^\top \boldsymbol{\sigma} + \frac{(\mathbf{c}^\top \boldsymbol{\sigma})^2 - s\mu^2}{\|\boldsymbol{\sigma}\|_2^2 - \mu^2} = 0,
    \end{equation}
    where we use the fact that $\|\mathbf{c}\|_2^2 = s$.
    We select the larger root to maintain dual feasibility ($\lambda, \nu \ge 0$), which directly yields the following explicit inner minimum:
    \begin{equation} \label{eq:inner_optimal_value}
        f_{s,\boldsymbol{\sigma}}^* = \frac{\mathcal{R}}{\|\boldsymbol{\sigma}\|_2^2} \left( \mu \mathbf{c}^\top \boldsymbol{\sigma} - \sqrt{s\|\boldsymbol{\sigma}\|_2^2 - (\mathbf{c}^\top \boldsymbol{\sigma})^2} \sqrt{\|\boldsymbol{\sigma}\|_2^2 - \mu^2} \right).
    \end{equation}

    Furthermore, for the boundary case where $\|\boldsymbol{\sigma}\|_2 = \mu$, the Cauchy-Schwarz inequality combined with the constraints forces $\mathbf{x}$ to be collinear with $\boldsymbol{\sigma}$, reducing the feasible set to a single-point set $\mathbf{x}^* = \frac{\mathcal{R}}{\mu} \boldsymbol{\sigma}$ with the objective value $\frac{\mathcal{R}}{\mu} \mathbf{c}^\top \boldsymbol{\sigma}$.
    Notably, the limit of Eq. \eqref{eq:inner_optimal_value} as $\|\boldsymbol{\sigma}\|_2 \downarrow \mu$ coincides exactly with this boundary value, establishing that the explicit formula holds for all $\boldsymbol{\sigma} \in \Sigma$.

    \textit{Step 2: Outer Minimization.}
    We employ a polar decomposition $\boldsymbol{\sigma} = r \mathbf{d}$, where $r \in [\mu, L]$ and $\|\mathbf{d}\|_2 = 1$. Let $\gamma := \mathbf{c}^\top \mathbf{d}$. The inner optimal value \eqref{eq:inner_optimal_value} becomes $f(r, \gamma) = \frac{\mathcal{R}}{r} ( \mu \gamma - \sqrt{s - \gamma^2} \sqrt{r^2 - \mu^2} )$.
    Since $\frac{\partial f}{\partial r} < 0$ and $\frac{\partial f}{\partial \gamma} > 0$, minimizing $f(r, \gamma)$ requires maximizing $r$ and minimizing $\gamma$. We trivially obtain $r^* = L$.

    Minimizing $\gamma$ equates to solving $\gamma^* = \min_{\|\mathbf{d}\|_2=1, \mathbf{d} \in \mathcal{K}} \mathbf{c}^\top \mathbf{d}$, where $\mathcal{K}$ is the cone of non-increasing non-negative vectors. Applying the variable change $\mathbf{y} = \mathbf{d} / (\mathbf{c}^\top \mathbf{d})$, this is mathematically equivalent to the reciprocal of $\rho^* = \max_{\mathbf{y} \in \mathcal{S}} \|\mathbf{y}\|_2$, where $\mathcal{S} = \{ \mathbf{y} \in \mathcal{K} : \mathbf{c}^\top \mathbf{y} = 1 \}$.

    By Bauer's Maximum Principle \citep{bauer1958minimalstellen}, the maximum of the convex function $\|\mathbf{y}\|_2$ over the compact set $\mathcal{S}$ is attained at one of its extreme points.
    The extreme rays of the cone $\mathcal{K}$ are generated by $\mathbf{v}_k = \sum_{j=1}^k \mathbf{e}_j$ for $k=1,\dots,n$.
    Intersecting these rays with the hyperplane $\mathbf{c}^\top \mathbf{y} = 1$ yields the extreme points of $\mathcal{S}$, given by $\mathbf{u}_k = \mathbf{v}_k / \min(k,s)$.
    Computing their norms gives $\|\mathbf{u}_k\|_2 = \sqrt{k}/\min(k,s)$.
    Notice that this norm is strictly decreasing for $1 \leq k \leq s$ (attaining a local maximum of $1$ at $k=1$) and strictly increasing for $s < k \le n$ (attaining a local maximum of $\sqrt{n}/s$ at $k=n$).
    Comparing these two boundary cases yields the global maximum $\rho^* = \max \{ 1, \frac{\sqrt{n}}{s} \}$, which implies $\gamma^* = 1/\rho^* = \min \{ 1, \frac{s}{\sqrt{n}} \} = \alpha_s$.

    Substituting $r^* = L$ and $\gamma^* = \alpha_s$ back into $f(r, \gamma)$ directly yields the desired lower bound, completing the proof.
\end{proof}

\section{Proofs for Section \ref{ssec:regularized_spectral_descent_with_weight_decay}}
\label{Proofs for RSD-WD}
\subsection{Proof of Lemma \ref{lemma:equivalence_spectral_descent_wd_frank_wolfe}}
\label{proof:equivalence_spectral_descent_wd_frank_wolfe}

\begin{proof}
We now prove the RSD-WD case.
For simplicity, we write
$$X=X^{(t)},\qquad
    \mathcal T=\mathcal T_t,\qquad
    \tilde{\mathbf G}=\tilde{\mathbf G}^{(t)},\qquad
    D=D^{(t)},\qquad
    g=g^{(t)}.$$
Since $D\in\partial\|\tilde{\mathbf G}\|_*$, we have $\|D\|_2\le 1$ and $\langle D,\tilde{\mathbf G}\rangle=\|\tilde{\mathbf G}\|_*$.
Therefore $g=-\frac1\lambda D\in\mathcal X_\lambda$.

First, we show that $g$ minimizes the linear function generated by $\tilde{\mathbf G}$. 
For any $Y\in\mathcal X_\lambda$, by the duality between the spectral norm and the nuclear norm, $\langle Y,\tilde{\mathbf G}\rangle\ge-\|Y\|_2\|\tilde{\mathbf G}\|_*\ge-\frac1\lambda\|\tilde{\mathbf G}\|_*$.
On the other hand, $\langle g,\tilde{\mathbf G}\rangle=-\frac1\lambda\langle D,\tilde{\mathbf G}\rangle=-\frac1\lambda\|\tilde{\mathbf G}\|_*$.
Hence $$g\in\arg\min_{Y\in\mathcal X_\lambda}\langle Y-X,\tilde{\mathbf G}\rangle.$$

Second, the KKT condition $0\in D+\lambda X+N_{\mathcal T}(\tilde{\mathbf G})$ implies $-(D+\lambda X)\in N_{\mathcal T}(\tilde{\mathbf G})$.
By the definition of the normal cone, for every \(\mathbf G\in\mathcal T\), $\left\langle-(D+\lambda X),\mathbf G-\tilde{\mathbf G}\right\rangle\le 0$.
Since $g-X=-\frac1\lambda(D+\lambda X)$,
we obtain $\left\langle g-X,\mathbf G-\tilde{\mathbf G}\right\rangle\leq 0$, $\forall \mathbf G\in\mathcal T$.
Therefore
$$\tilde{\mathbf G} \in\arg\max_{\mathbf G\in\mathcal T}\left\langle g-X,\mathbf G\right\rangle.$$
Combining the two optimality relations gives the saddle inequality
$$\langle g-X,\mathbf G\rangle
    \le
    \langle g-X,\tilde{\mathbf G}\rangle
    \le
    \langle Y-X,\tilde{\mathbf G}\rangle,
    \qquad
    \forall \mathbf G\in\mathcal T,\ \forall Y\in\mathcal X_\lambda.
$$
Thus
$$ g
    \in
    \arg\min_{Y\in\mathcal X_\lambda}
    \max_{\mathbf G\in\mathcal T}
    \left\langle
    Y-X,\mathbf G
    \right\rangle .
$$
Finally, if \(\gamma_t=\lambda\eta_t\), then
$$(1-\gamma_t)X+\gamma_t g
    =
    (1-\lambda\eta_t)X
    +
    \lambda\eta_t\left(-\frac1\lambda D\right)
    =
    X-\eta_t(D+\lambda X).
$$
This is exactly the RSD-WD update.
The RTSD-WD case follows from the same argument, with the nuclear norm replaced by the Ky Fan \(s\)-norm and with \(\mathcal X_\lambda\) replaced by \(\mathcal X_{\lambda,s}\).
We omit the details for brevity.
\end{proof}

\subsection{Proof of Proposition \ref{prop:boundedness_curvature_matrix}}
\label{proof:boundedness_curvature_matrix}
\begin{proof}[Proof of Proposition \ref{prop:boundedness_curvature_matrix}]
	$D_{\mathcal{X}_\lambda} = \sup_{X,Y \in \mathcal{X}_\lambda} \|X-Y\|_F = 2\sqrt{n}/\lambda < \infty$ denote the diameter of the domain.
	For any $X,S\in\mathcal{X}_\lambda$ and $\gamma\in(0,1]$, let $Y=(1-\gamma)X+\gamma S$.
	Let the objective term be denoted by
	$Q(X, Y, \mathbf{G}) := \frac{2}{\gamma^2} (f(Y) - f(X) - \langle Y-X, \mathbf{G}\rangle)$.
	We analyze the upper bound of the curvature term by considering two regimes based on the step size relative to $\epsilon$.

	\textit{Case 1: Small Step Regime} ($\|Y - X\|_F \le \epsilon$).
	When the step size is sufficiently small such that $Y$ falls within the $\epsilon$-neighborhood of $X$, we invoke Lebourg's Mean Value Theorem \citep{clarke1990optimization} for convex Lipschitz functions.

	By the Mean Value Theorem, there exists some point $Z$ on the open segment $(X, Y)$ such that
	\begin{equation}
		f(Y) - f(X) = \langle Y - X, \G_Z \rangle.
	\end{equation}
	Here, $\G_Z \in \partial f(Z)$ is a subgradient at $Z$.
	Since $Z$ lies on the segment $[X, Y]$ and $\|Y - X\|_F \le \epsilon$, it follows geometrically that $\|Z - X\|_F \le \|Y - X\|_F \le \epsilon$.
	Thus, $Z \in N(X, \epsilon)$, which implies that $\G_Z \in T(X, \epsilon)$.
	We can select $\mathbf{G} = \G_Z$.
	Substituting this into the objective term yields:
	\begin{equation}
		f(Y) - f(X) - \langle Y - X, \G \rangle = f(Y) - f(X) - \langle Y - X, \G_Z \rangle = 0.
	\end{equation}
	Consequently, in the small step regime, the curvature term vanishes (is bounded by 0).

	\textit{Case 2: Large Step Regime} ($\|Y - X\|_F > \epsilon$).
	When the step size is large, specifically $\|Y - X\|_F > \epsilon$, the term is bounded by a constant dependent on the domain diameter and Lipschitz constant.

	The condition $\|Y - X\|_F > \epsilon$ implies a lower bound on the step size $\gamma$.
	Since $\|Y - X\|_F = \gamma \|S - X\|_F \le \gamma D_{\mathcal{X}_\lambda}$, we have:
	\begin{equation}
		\gamma \ge \frac{\epsilon}{D_{\mathcal{X}_\lambda}} \implies \frac{1}{\gamma} \le \frac{D_{\mathcal{X}_\lambda}}{\epsilon}.
	\end{equation}
	Using the $L$-Lipschitz property of $f$, we have $\|\mathbf{G}\|_F \le L$.
	We apply the Cauchy-Schwarz inequality and the Lipschitz condition to bound the numerator:
	$$f(Y) - f(X) \le L \|Y - X\|_F,$$
	$$-\langle Y - X, \mathbf{G} \rangle \le \|\mathbf{G}\|_F \|Y - X\|_F \le L \|Y - X\|_F.$$
	Thus, the numerator is bounded by $2L \|Y - X\|_F = 2L \gamma \|S - X\|_F$.
	Substituting these bounds into $C_f$:
	\begin{equation}
		\begin{aligned}
			\frac{2}{\gamma^2} (f(Y) - f(X) - \langle Y - X, \mathbf{G} \rangle) &\le \frac{2}{\gamma^2} (2L \gamma \|S - X\|_F) \\
			&\le \frac{4L D_{\mathcal{X}_\lambda}}{\gamma} \\
			&\le \frac{4L D_{\mathcal{X}_\lambda}^2}{\epsilon}.
		\end{aligned}
	\end{equation}
\end{proof}

\subsection{Proof of Lemma \ref{lemma:epsilon-subgradient-inequality}}
\label{proof:epsilon-subgradient-inequality}
\begin{lemma}[$\epsilon$-subgradient Inequality]
	\label{lemma:epsilon-subgradient-inequality}
	Let $f:\R^{n_1\times n_2}\to\R$ satisfy Assumption \ref{assumption:Convexity} and \ref{assumption:Lipschitz Continuity} with Lipschitz constant $L>0$.
	Then for any $\G\in T(X,\frac{\epsilon}{2L})$, we have
	\begin{equation}
		f(Y) \geq f(X) + \langle Y-X, \G\rangle -\epsilon\quad \forall Y\in\R^{n_1\times n_2}
	\end{equation}
\end{lemma}
\begin{proof}
	First, we establish the inequality for elements in the set $\bigcup_{u \in N(X, \frac{\epsilon}{2L})} \partial f(u)$.
    Since $f$ is convex, for any $Y\in\R^{n_1\times n_2}$ and $Z\in N(X,\frac{\epsilon}{2L})$, we have
	\begin{equation}
		\begin{aligned}
			&f(Y) \geq f(Z) + \langle Y-Z, \G_{Z}\rangle, \quad \text{where}\;\G_{Z}\in\partial f(Z), \\
			&f(Z) \geq f(X) + \langle Z-X, \G_{X}\rangle, \quad \text{where}\;\G_{X}\in\partial f(X).
		\end{aligned}
	\end{equation}
	Combining the above two inequalities yields
	\begin{equation}
		\begin{aligned}
		f(Y) &\geq f(X) + \langle Y-Z, \G_{Z}\rangle + \langle Z-X, \G_{X}\rangle\\
		&= f(X) + \langle Y-X, \G_{Z}\rangle - \langle Z-X, \G_{Z}-\G_{X}\rangle.			\end{aligned}
	\end{equation}
	By the Lipschitz continuity of $f$ and definition of $N(X,\frac{\epsilon}{2L})$, we have
	\begin{equation}
		\langle Z-X, \G_{Z}-\G_{X}\rangle \leq \|Z-X\|_F \|\G_{Z}-\G_{X}\|_F \leq 2L \frac{\epsilon}{2L} = \epsilon.
	\end{equation}
	Substituting this back, we obtain that for any $\mathbf{g} \in \partial f(Z) \subseteq \bigcup_{u \in N(X, \frac{\epsilon}{2L})} \partial f(u)$,
	\begin{equation}
        \label{eq:epsilon-subgradient Inequality_base}
		f(Y) \geq f(X) + \langle Y-X, \mathbf{g}\rangle - \epsilon.
	\end{equation}
	Finally, we extend this to the convex hull.
    Notice that the inequality \eqref{eq:epsilon-subgradient Inequality_base} is linear with respect to $\mathbf{g}$.
    Since the inequality holds for all $\mathbf{g} \in \bigcup_{u \in N(X, \frac{\epsilon}{2L})} \partial f(u)$, it naturally holds for any convex combination of these subgradients.
    Thus, for any $\mathbf{G} \in T(X, \frac{\epsilon}{2L})$, the claim holds.
\end{proof}

\section{Auxiliary Lemmas for Section \ref{sec:application}}
\label{app:auxiliary_lemmas_robust_low_rank_matrix_recovery}
\subsection{Proof of Lemma \ref{lemma:nuclear_norm_property}}
\label{proof:nuclear_norm_property}
\begin{lemma}
	\label{lemma:nuclear_norm_property}
	If $X^*$ is a rank-$r^*$ matrix, then for any matrix $H\in\R^{n_1\times n_2}$, we have
	\begin{equation}
		\|X^*+H\|_* \geq \|X^*\|_* + \|H_{T_{X^*}^\perp}\|_* - \|H_{T_{X^*}}\|_*.
	\end{equation}

\end{lemma}

\begin{proof}
It is well-known that the nuclear norm is a convex function.
Then for any matrix $X\in\R^{n_1\times n_2}$ and $H\in\R^{n_1\times n_2}$, we have
\begin{equation} \label{eq:subgradient_inequality}
    \|X + H\|_* \geq \|X\|_* + \langle g, H \rangle,\quad \forall g \in \partial \|X\|_*.
\end{equation}

Let $X^* = U\Sigma V^\top$ be the compact SVD of the rank-$r^*$ matrix $X^*$.
The subdifferential of the nuclear norm at $X^*$ is given by (see, e.g., \citep{watson1992characterization, recht2010guaranteed}):
\begin{equation} \label{eq:subdifferential_set}
    \partial \|X^*\|_* = \left\{ UV^\top + W \mid U^\top W = 0, WV = 0, \|W\|_2 \leq 1 \right\}.
\end{equation}
Notice that the matrix $W$ belongs to the orthogonal complement space $T_{X^*}^\perp$.
By the definition of the tangent space $T_{X^*}$ and its orthogonal complement $T_{X^*}^\perp$, any matrix $H \in \R^{n_1 \times n_2}$ can be uniquely decomposed as
\begin{equation}
    H = H_{T_{X^*}} + H_{T_{X^*}^\perp}.
\end{equation}
Let the compact SVD of the projection $H_{T_{X^*}^\perp}$ be $\tilde{U} \tilde{\Sigma} \tilde{V}^\top$.
To obtain the lower bound, we construct a matrix $g^* = UV^\top + W^*$, where $W^* = \tilde{U} \tilde{V}^\top$.
It is clear that $\|W^*\|_2 \leq 1$ and $W^* \in T_{X^*}^\perp$.
Thus, $g^* \in \partial \|X^*\|_*$.

Substituting this specific matrix $g^*$ into the subgradient inequality (\ref{eq:subgradient_inequality}), we obtain
\begin{equation}
    \|X^* + H\|_* \geq \|X^*\|_* + \langle UV^\top + W^*, H_{T_{X^*}} + H_{T_{X^*}^\perp} \rangle.
\end{equation}
It is easy to verify that $UV^\top \in T_{X^*}$ and $W^* \in T_{X^*}^\perp$.
Therefore, we have $\langle UV^\top, H_{T_{X^*}^\perp} \rangle = 0$ and $\langle W^*, H_{T_{X^*}} \rangle = 0$.
This allows us to separate the inner product into two terms:
\begin{equation} \label{eq:separated_inner_product}
    \|X^* + H\|_* \geq \|X^*\|_* + \langle UV^\top, H_{T_{X^*}} \rangle + \langle W^*, H_{T_{X^*}^\perp} \rangle.
\end{equation}
For the first term, applying H\"older's inequality, we have
   $$ \langle UV^\top, H_{T_{X^*}} \rangle \geq -\|UV^\top\|_2 \|H_{T_{X^*}}\|_* = -\|H_{T_{X^*}}\|_*. $$
For the second term, based on our construction of $W^* = \tilde{U}\tilde{V}^\top$, the inner product exactly yields
   $$ \langle W^*, H_{T_{X^*}^\perp} \rangle = \operatorname{tr}(\tilde{V}\tilde{U}^\top \tilde{U} \tilde{\Sigma} \tilde{V}^\top) = \operatorname{tr}(\tilde{\Sigma}) = \|H_{T_{X^*}^\perp}\|_*. $$
Finally, substituting these two bounds back into (\ref{eq:separated_inner_product}) yields
\begin{equation}
    \|X^* + H\|_* \geq \|X^*\|_* - \|H_{T_{X^*}}\|_ * + \|H_{T_{X^*}^\perp}\|_*.
\end{equation}
This completes the proof.
\end{proof}
\subsection{Proof of Proposition \ref{prop:equivalence_neighborhood_subgradients}}
\label{proof:equivalence_neighborhood_subgradients}
   \begin{proposition}
    \label{prop:equivalence_neighborhood_subgradients}
    Consider the structured subdifferential set $T(X, \epsilon) := \text{conv}\left(\bigcup_{Y \in N(X, \epsilon)} \partial f(Y)\right)$, where $N(X, \epsilon) = \{Y \in \R^{n_1 \times n_2} : \|\mathcal{A}(Y) - \mathcal{A}(X)\|_\infty \leq \epsilon\}$ is the unconstrained $L_\infty$-neighborhood around $X$.
    Then, $T(X, \epsilon)$ is a subset of the surrogate set $H(X, \epsilon)$, defined as:
    \begin{equation}
        H(X, \epsilon) := \left\{ \mathcal{A}^* v \;\middle|\; v \in H_1(X, \epsilon) \right\},
    \end{equation}
    where $H_1(X, \epsilon) \subset \mathbb{R}^m$ is defined by:
    \begin{equation}
        H_1(X, \epsilon) := \left\{ v \in \mathbb{R}^m \;\middle|\;
        \begin{aligned}
        v_i &= \frac{\operatorname{sign}(\mathcal{A}(X)-b)_i}{m}, && \text{if } |\mathcal{A}(X)-b|_i > \epsilon \\
        v_i &\in \left[-\frac{1}{m}, \frac{1}{m}\right], && \text{if } |\mathcal{A}(X)-b|_i \leq \epsilon
        \end{aligned}
        \right\}.
    \end{equation}
\end{proposition}

\begin{proof}
    Let $\hat{f}(z) := \frac{1}{m}\|z\|_1$. By the chain rule for subdifferentials, we have $\partial f(Y) = \mathcal{A}^* \partial \hat{f}(\mathcal{A}(Y)-b)$ for any $Y \in \R^{n_1 \times n_2}$.
    Since $\mathcal{A}^*$ is a linear operator, it preserves convexity and commutes with the convex hull operation.
    Thus, we can rewrite $T(X, \epsilon)$ as:
    \begin{equation}
        \begin{aligned}
            T(X, \epsilon) &= \text{conv}\left(\bigcup_{Y \in N(X, \epsilon)} \mathcal{A}^* \partial \hat{f}(\mathcal{A}(Y)-b)\right) \\
            &= \mathcal{A}^* \left[ \text{conv}\left(\bigcup_{Y \in N(X, \epsilon)} \partial \hat{f}(\mathcal{A}(Y)-b)\right) \right].
        \end{aligned}
    \end{equation}
    Define the inner convex set $T_1(X, \epsilon) := \text{conv}\left(\bigcup_{Y \in N(X, \epsilon)} \partial \hat{f}(\mathcal{A}(Y)-b)\right)$.
    To prove $T(X, \epsilon) \subseteq H(X, \epsilon)$, it suffices to show that $T_1(X, \epsilon) \subseteq H_1(X, \epsilon)$.

    Let $Y \in N(X, \epsilon)$ be an arbitrary element in the neighborhood, which strictly implies $\|\mathcal{A}(Y) - \mathcal{A}(X)\|_\infty \leq \epsilon$.
    Let $v \in \partial \hat{f}(\mathcal{A}(Y)-b)$ be any valid subgradient vector at $Y$. We analyze $v$ coordinate by coordinate for $i \in [m]$:
    \begin{itemize}
        \item If $|\mathcal{A}(X)-b|_i > \epsilon$, then we have $v_i = \frac{1}{m}\sign(\mathcal{A}(X)_i - b_i)$.
        Otherwise, $\mathcal{A}(Y)_i - \mathcal{A}(X)_i$ must be greater than $\epsilon$, which contradicts the definition of $N(X,\epsilon)$.
        \item If $|\mathcal{A}(X)-b|_i \leq \epsilon$, then $v_i$ can take any value in the interval $[-\frac{1}{m}, \frac{1}{m}]$.
    \end{itemize}

    Therefore, for any $Y \in N(X, \epsilon)$, every element $v \in \partial \hat{f}(\mathcal{A}(Y)-b)$ satisfies the defining conditions of $H_1(X, \epsilon)$.
    This implies that $\bigcup_{Y \in N(X, \epsilon)} \partial \hat{f}(\mathcal{A}(Y)-b) \subseteq H_1(X, \epsilon)$.

    Observe that $H_1(X, \epsilon)$ is a convex set (a polytope).
    Since the union of all subgradients is contained within the convex set $H_1(X, \epsilon)$, their convex hull must also be contained within it.
    Thus, $T_1(X, \epsilon) \subseteq H_1(X, \epsilon)$, which completes the proof.
\end{proof}
\subsection{Proof of Lemma \ref{lemma:surrogate-subgradient-inequality}}
\label{proof:surrogate-subgradient-inequality}
\begin{lemma}[Approximate Subgradient Inequality]
    \label{lemma:surrogate-subgradient-inequality}
    Let $f(X) = \frac{1}{m}\|\mathcal{A}(X)-b\|_1$ be the LAD objective.
    Then for any surrogate subgradient $\G \in H(X, \epsilon)$, the following $2\epsilon$-subgradient inequality rigorously holds:
    \begin{equation}
        f(Y) \geq f(X) + \langle Y-X, \G\rangle - 2\epsilon \quad \forall Y \in \R^{n_1 \times n_2}.
    \end{equation}
    Furthermore, this inequality holds for any $\G \in T(X, \epsilon)$ since $T(X, \epsilon) \subseteq H(X, \epsilon)$.
\end{lemma}
\begin{proof}
    Let $\G \in H(X, \epsilon)$ be an arbitrary element.
    By the definition of $H$, there exists a dual vector $v \in H_1(X, \epsilon)$ such that $\G = \mathcal{A}^* v$, and it satisfies $\|v\|_\infty \leq \frac{1}{m}$.
    For any $Y \in \R^{n_1 \times n_2}$, we evaluate the inner product:
    \begin{equation}
        \label{eq:decouple_inner_product}
        \langle Y-X, \G \rangle = \langle Y-X, \mathcal{A}^* v \rangle = \langle \mathcal{A}(Y)-\mathcal{A}(X), v \rangle = \langle (\mathcal{A}(Y)-b) - (\mathcal{A}(X)-b), v \rangle.
    \end{equation}
    Let the residual vectors be $z(Y) = \mathcal{A}(Y)-b$ and $z(X) = \mathcal{A}(X)-b$.
    We decouple the inner product:
    \begin{equation}
        \langle Y-X, \G \rangle = \langle z(Y), v \rangle - \langle z(X), v \rangle.
    \end{equation}
    For the first term, applying H\"older's inequality and the global bound $\|v\|_\infty \leq \frac{1}{m}$, we have
    \begin{equation}
        \label{eq:holder_zY}
        \langle z(Y), v \rangle \leq \|z(Y)\|_1 \|v\|_\infty \leq \frac{1}{m} \|z(Y)\|_1 = f(Y).
    \end{equation}
    For the second term, we decompose the inner product $\langle z(X), v \rangle$ based on the index partition defined by $H_1(X, \epsilon)$.
    Let the clear set be $I_{clear} = \{i : |z_i(X)| > \epsilon\}$ and the active set be $I_{active} = \{i : |z_i(X)| \leq \epsilon\}$. Then,
    \begin{equation}
        \begin{aligned}
            \langle z(X), v \rangle &= \sum_{i \in I_{clear}} z_i(X) \frac{\operatorname{sign}(z_i(X))}{m} + \sum_{i \in I_{active}} z_i(X) v_i\\
             &= \frac{1}{m} \sum_{i \in I_{clear}} |z_i(X)| + \sum_{i \in I_{active}} z_i(X) v_i.
        \end{aligned}
    \end{equation}
    Consider the residual difference $f(X) - \langle z(X), v \rangle$, we have
    \begin{equation}
            f(X) - \langle z(X), v \rangle= \sum_{i \in I_{active}} \left( \frac{|z_i(X)|}{m} - z_i(X) v_i \right).
    \end{equation}
    For $i \in I_{active}$, since $v_i \in [-\frac{1}{m}, \frac{1}{m}]$, it holds that $-z_i(X) v_i \leq \frac{|z_i(X)|}{m}$.
    Furthermore, it is clear that $|z_i(X)| \leq \epsilon$.
    Therefore,
    \begin{equation}
        \label{eq:bound_zX_coordinate}
        \frac{|z_i(X)|}{m} - z_i(X) v_i \leq \frac{2|z_i(X)|}{m} \leq \frac{2\epsilon}{m}.
    \end{equation}
    Summing this deviation \eqref{eq:bound_zX_coordinate} over the entire active set, we obtain
    \begin{equation}
        \label{eq:bound_zX}
        f(X) - \langle z(X), v \rangle \leq \sum_{i \in I_{active}} \frac{2\epsilon}{m} = \frac{|I_{active}|}{m} (2\epsilon) \leq 2\epsilon.
    \end{equation}
    Substituting \eqref{eq:holder_zY} and \eqref{eq:bound_zX} back into the inner product equation \eqref{eq:decouple_inner_product} yields
    \begin{equation}
        \langle Y-X, \G \rangle \leq f(Y) - \langle z(X), v \rangle \leq f(Y) - f(X) + 2\epsilon.
    \end{equation}
    Rearranging the terms provides the desired inequality.
\end{proof}
\subsection{Proof of Proposition \ref{prop:boundedness_curvature_matrix_LAD}}
\label{proof:boundedness_curvature_matrix_LAD}
\begin{proposition}
\label{prop:boundedness_curvature_matrix_LAD}
Let $f(X) = \frac{1}{m}\|\mathcal{A}(X)-b\|_1$ be the LAD objective defined on the constraint set $\mathcal{X} := \left\{X \in \R^{n_1 \times n_2} : \|X\|_* \leq a \right\}$.
Assume that $f$ satisfies Lipschitz continuity over $\mathcal{X}$ with Lipschitz constant $L_* > 0$ relative to the nuclear norm.
Then for any $\epsilon > 0$, 
\begin{equation}
    \mathcal{C}^T_f(\epsilon) \leq \frac{16 m L_*^2 a^2}{\epsilon}.
\end{equation}
Moreover, since $T(X,\epsilon)\subseteq H(X,\epsilon)$, the surrogate curvature also satisfies
\begin{equation}
    \mathcal{C}^H_f(\epsilon) \leq  \mathcal{C}^T_f(\epsilon) \leq \frac{16 m L_*^2 a^2}{\epsilon}.
\end{equation}
\end{proposition}
\begin{proof}
    By Proposition \ref{prop:equivalence_neighborhood_subgradients}, the convex hull of the union of subdifferentials is a subset of the surrogate set, i.e., $T(X, \epsilon) \subseteq H(X, \epsilon)$.
    Thus, the curvature evaluated over the surrogate set $H(X, \epsilon)$ is strictly upper bounded by the curvature evaluated over $T(X, \epsilon)$:
    \begin{equation}
        \mathcal{C}^H_f(\epsilon) \leq \mathcal{C}^T_f(\epsilon) = \sup_{\substack{X,S\in\mathcal{X}, \gamma\in(0,1] \\ Y=(1-\gamma)X+\gamma S}} \min_{\G\in T(X,\epsilon)} \frac{2}{\gamma^2} \Big( f(Y) - f(X) - \langle Y-X, \G\rangle \Big).
    \end{equation}
    It remains to bound $\mathcal C_f^T(\epsilon)$.
    Fix $X,S\in\mathcal X$ and $\gamma\in(0,1]$, and set $Y=(1-\gamma)X+\gamma S$.

If $\|\mathcal A(Y)-\mathcal A(X)\|_\infty\le \epsilon$, then every point on the segment $[X,Y]$ lies in $N(X,\epsilon)$. 
By Mean Value Theorem, there exists $Z\in[X,Y]$ and $G_Z\in\partial f(Z)\subseteq T(X,\epsilon)$ such that
$$
    f(Y)-f(X)=\langle Y-X,G_Z\rangle.
$$
Thus the curvature numerator is zero.

Otherwise, if $\|\mathcal A(Y)-\mathcal A(X)\|_\infty>\epsilon$, then
$$
    \epsilon
    <
    \|\mathcal A(Y)-\mathcal A(X)\|_\infty
    =
    \gamma\|\mathcal A(S-X)\|_\infty
    \le
    \gamma\|\mathcal A(S-X)\|_1
    \le
    2mL_*a\,\gamma.
$$
Hence
$$
    \frac1\gamma\le \frac{2mL_*a}{\epsilon}.
$$
For any $G\in T(X,\epsilon)$, the nuclear-norm Lipschitz bound gives
$$
    f(Y)-f(X)-\langle Y-X,G\rangle
    \le
    2L_*\|Y-X\|_*
    \le
    4\gamma L_*a.
$$
Therefore,
$$
    \frac{2}{\gamma^2}
    \left(f(Y)-f(X)-\langle Y-X,G\rangle\right)
    \le
    \frac{8L_*a}{\gamma}
    \le
    \frac{16mL_*^2a^2}{\epsilon}.
$$
Taking the supremum over $X,S,\gamma$ proves the claimed bound.
\end{proof}

\end{document}